\documentclass[11pt,letterpaper,english]{article}
\usepackage[small,compact]{titlesec}
\usepackage[toc,page]{appendix}
\usepackage{microtype}
\usepackage{float}
\usepackage{amsmath}
\usepackage{amsthm}
\usepackage{mathrsfs}
\usepackage{graphicx}
\usepackage{caption}
\usepackage{subcaption}
\usepackage{setspace}
\usepackage{threeparttable}
\newtheorem{thm}{Theorem}

\usepackage{amssymb}
\usepackage{newtxtext} % Times Roman for text
\usepackage{newtxmath} % Times Roman for math
\usepackage{multirow}
\usepackage{rotating}
\usepackage{booktabs}  % For nicer rules in tables
\usepackage{makecell}
\usepackage{array}
\usepackage{color}
\usepackage[hidelinks]{hyperref}
\usepackage{fullpage}
\newtheorem{defn}{Definition}

\usepackage{algorithm}
\usepackage{algpseudocode}
\usepackage{tikz}

\usepackage{xcolor}
\usepackage[margin=2.5cm]{geometry}
\usepackage[sectionbib,round]{natbib}
\usepackage[normalem]{ulem}

\usepackage{bibunits}
\defaultbibliographystyle{abbrvnat} 
\defaultbibliography{references} 
\usepackage{listings}
\bibpunct[, ]{(}{)}{;}{a}{}{,}

\newcommand{\squishlist}{
   \begin{list}{$\bullet$}
    { \setlength{\itemsep}{0pt} \setlength{\parsep}{1pt}
      \setlength{\topsep}{1pt} \setlength{\partopsep}{1pt}
      \setlength{\leftmargin}{1.5em} \setlength{\labelwidth}{1em}
      \setlength{\labelsep}{0.5em} } }

\newcommand{\squishlisttwo}{
   \begin{list}{$\bullet$}
    { \setlength{\itemsep}{0pt} \setlength{\parsep}{0pt}
      \setlength{\topsep}{0pt} \setlength{\partopsep}{0pt}
      \setlength{\leftmargin}{1em} \setlength{\labelwidth}{1.5em}
      \setlength{\labelsep}{0.5em} } }

\newcommand{\squishend}{
    \end{list}  }

\begin{document}

\title{Adaptive Ad Load Design for Sponsored Search Markets: Evidence, Theory, and Deployment}
\author{
      Mohammad Rashid \thanks{We would like to thank the participants of the PhD workshop at the University of Washington and the 2026 UW-UBC conference for their comments. We are also grateful to Simha Mummalaneni, Omid Rafieian, Ken Wilbur, Zikun Ye, and Dennis Zhang for their detailed feedback, which has significantly improved the paper. Please address all correspondence to: rashid98@uw.edu, hemay@uw.edu. }  \\
        University of Washington\\
        \and
        Hema Yoganarasimhan\footnotemark[1]\\
        University of Washington\\
 }
\maketitle
\begin{abstract}

Ad-load design is a central supply-side decision in sponsored search: more sponsored slots can raise revenue, but may crowd out organic results and degrade user outcomes. We study this trade-off using a large-scale randomized field experiment on an Android app store, where over five million users are exposed to one through six sponsored slots. Increasing ad load raises revenue by up to 43\%, but reduces total search conversions by up to 5\% and daily engagement by up to 2.2\%. These average effects mask substantial heterogeneity: additional slots generate large revenue gains for high-ad-conversion queries, but little or negative marginal revenue for low-conversion queries. The trade-off also shifts within query as advertiser composition changes, such as brand-advertiser presence. Motivated by these findings, we design and deploy a novel adaptive algorithm -- {\it e}xploration-augmented Locally Adaptive Ad Load (e-LAAL). e-LAAL combines LAAL, a model-free query-level decision rule that updates ad-load recommendations using recent outcomes, with static exploration arms that maintain support and provide fixed-policy counterfactual benchmarks. We provide a finite-time dynamic-regret guarantee for the e-LAAL architecture. In a platform-level production deployment serving 22.3 million users and 77.6 million searches, e-LAAL improves the empirical revenue--conversion trade-off relative to deployed static benchmarks and outperforms uniform and historical query-dependent static benchmarks.

\end{abstract}
\noindent \textbf{Keywords:} Digital advertising, Adaptive algorithms, Sponsored search, Ad load

\thispagestyle{empty}

\begin{bibunit}
\newpage
\setcounter{page}{1}

\section{Introduction}

\subsection{Sponsored Search Advertising}

Search advertising is one of the central monetization channels in the digital economy. In 2025, search advertising revenue in the United States reached \$114.2 billion, making it the largest format of online advertising \citep{iab_pwc_2026_internet_ad_revenue_2025}. The success of sponsored advertising stems from a unique structural advantage: the precise alignment between user intent and advertiser offerings. Unlike display or social media advertising, which often rely on interrupting a user's consumption of content, search ads are served at the moment a user reveals a specific need or interest through a query. This intent-driven architecture reduces the friction between the user's objective and the advertisement, making search ads more relevant, less intrusive, and often more effective for direct-response objectives than other digital formats. This dynamic is particularly critical in mobile app marketplaces, where the sheer volume of available products creates a significant discovery bottleneck. With millions of applications available, search has become the primary navigation tool for users; Apple reports that nearly 65\% of all App Store downloads are initiated directly after a search query \citep{apple_ads_app_store}. The search results page, therefore, becomes a natural advertising surface. 

Within this environment, the platform faces an important design problem. While the auction format determines \emph{which} ads are shown and at what price, a more fundamental supply-side decision remains under the platform's direct control: \emph{ad load design}, or determining the number of sponsored slots to display. Increasing ad load expands paid inventory and can increase short-run platform revenue, but sponsored results may displace organic results that are more relevant to the user. Excessive ad exposure can therefore degrade search quality, reduce conversions, and weaken engagement \citep{kohavi2012trustworthy, Carrion2023VirtualBids}. Ad load design is therefore a strategic platform-level allocation problem that governs the balance between monetization and user experience. Yet despite the importance of this supply-side choice, there are no established paradigms for how platforms should set ad load in sponsored search settings.

\subsection{Research Agenda and Challenges}

In this paper, we examine the platform's ad load design problem through three interrelated research questions. First, what is the causal effect of changing the number of sponsored slots on platform outcomes? We quantify the marginal impact of additional ads on advertising revenue and conversion,\footnote{Our empirical setting is an app-store platform (introduced in \S\ref{ssec:our_approach}), where the natural unit of conversion is an app install. Throughout this paper, we define conversion as the number of apps installed through search, originating from both sponsored ads and organic results. We focus on total search conversions as a main outcome because app installs are the platform's primary matching outcome and provide immediate search-level feedback on whether users find relevant apps.} which provides the baseline trade-off underlying ad-load decisions. Second, how do ad-load effects vary across queries and market conditions? We examine heterogeneity across queries, based on differences in baseline ad conversion potential, and within queries over time, based on shifts in advertiser composition such as brand participation. Third, given such heterogeneity, how can platforms design an ad-load policy that tailors sponsored inventory to context and adapts as market conditions change?

Addressing these questions comes with three challenges. First, identifying the revenue--conversion trade-off requires exogenous user-level variation in ad load, so outcome differences can be attributed to the number of sponsored slots rather than selection or demand shocks. Second, heterogeneity analysis requires moderators that are predetermined or independently measured; otherwise, query- or market-level variables may be endogenous to the realized ad load. Third, policy design requires persistent support for multiple ad-load levels within each context over time. Without such support, the platform cannot keep updating counterfactual performance as advertiser entry and exit, budget depletion, bid changes, seasonality, and user intent shift. Thus, the policy problem is not only to learn which ad load works on average, but to build a theoretically grounded and deployable adaptive system that continues learning as the market evolves.

\subsection{Our Approach and Findings}
\label{ssec:our_approach}

We study ad-load design on a large Android app-store platform in an Asian country. The platform is a major mobile distribution channel, with over 70\% market share, roughly 40 million monthly users, 5 million daily active users, and nearly half a million apps. Search is a central discovery mechanism: users install about 2.5 million apps per day from search results pages, of which roughly 50,000 installs are attributed to sponsored ads. When a user submits a query, the platform runs a real-time Generalized Second-Price auction among eligible advertisers and displays sponsored results above organic results. The status quo before our intervention was a \emph{single-ad layout}, in which one sponsored result appeared at the top of the page, followed by organic results. This setting is ideal for studying ad-load design because increasing sponsored slots directly expands monetizable inventory while displacing organic content.

%First, we run a large-scale randomized field experiment to quantify the causal effects of static ad-load rules. Second, we use the experimental variation to document heterogeneity in the revenue--conversion trade-off across queries and over time. Third, motivated by this heterogeneity, we formulate ad-load design as a sequential decision problem and identify the requirements that a deployable ad-load policy must satisfy in this environment. Fourth, we design and deploy exploration-augmented Locally Adaptive Ad Load (e-LAAL), a model-free query-level policy that uses recent outcomes to update ad-load decisions while maintaining parallel static exploration arms. Fifth, we benchmark e-LAAL against uniform and query-dependent static policies to assess whether adaptive learning improves the revenue--conversion frontier beyond what fixed rules can achieve. We now discuss these steps in detail below.
We address the paper's research questions in five steps. First, we run a 66-day user-level randomized field experiment that exogenously varies the number of sponsored slots shown on the search results page. The experiment covers 10\% of the platform's user base and includes over 5 million users, 26 million searches, 377K sponsored-ad installs, and 15 million organic installs. Users are assigned to six persistent regimes: a control group, covering 5\% of users, receives the status quo one-ad layout, while five treatment groups, $T_2$ through $T_6$, each covering 1\% of users, receive fixed layouts with two through six sponsored slots. This design identifies the causal effects of moving from the status quo to increasingly aggressive ad-load regimes at a scale large enough to detect both monetization effects and user-side behavioral responses.

The experiment reveals a clear aggregate trade-off. Moving from one ad to six ads increases revenue per user by 43\%, but reduces total conversions per user by up to 5\%. Higher ad loads also increase searches per user by up to 3\%, suggesting that users exposed to heavier ad loads may need to search more to find relevant apps, while reducing daily engagement by up to 2.2\% and generating a negative engagement trend among users exposed to at least four ads per search. Taken together, these results show that additional sponsored slots raise short-run monetization but degrade user-side outcomes, making ad-load design a multi-objective optimization problem rather than a simple revenue-maximization problem.

Second, we use the experimental variation to show that the revenue--conversion trade-off is heterogeneous across queries and over time. To study cross-query heterogeneity, we focus on the top 10,000 queries by advertising conversion count and partition them into three equally sized segments (Low, Medium, High) based on advertising conversion rate in the control condition, where the sponsored ad always appears in the first position. Across all three segments, higher ad loads reduce conversions, but revenue effects differ sharply: in the Low segment, additional slots do not generate incremental revenue, whereas in the High segment, extra slots generate substantial incremental revenue. Thus, the value of an additional sponsored slot depends critically on query-level advertiser quality: a uniform ``always show $i$ ads'' rule misallocates inventory by showing too many ads on queries where ads convert poorly and too few ads where sponsored demand is valuable. We then examine time variation using brand queries, i.e., branded phrases such as ``WhatsApp,'' and find that periods with brand-advertiser presence are associated with higher revenue and conversions, with the largest gains under high ad loads. This within-query variation shows that the same query can appear on a different revenue--conversion frontier as advertiser composition changes. These results imply that effective ad-load design must be both query-sensitive and time-adaptive. 
%\rev{A natural question is whether ad load should also be personalized across \emph{users}. We focus on query- and time-level adaptation for two reasons: the platform avoids user-level ad-load personalization due to privacy and fairness considerations, and the dominant sources of heterogeneity in our setting operate at the query and time level (see \S\ref{sec:prob_setup}).}

%At a high level, this problem maps naturally to a contextual bandit: the query serves as the context, the ad load is the action, and revenue and conversions constitute the reward. However,
Third, these empirical patterns motivate an adaptive ad-load allocation policy but also highlight why standard parametric contextual bandits are not a good fit for this setting. Conventional contextual bandits require a stable reward model over observed contexts, but many reward-relevant states in sponsored search are partially observed, including advertiser entry and exit, budget depletion, bid changes, app availability, brand participation, and shifts in user demand. The mapping from query, ad load, and market state to revenue and conversions is also complex and non-stationary, and ad-load decisions may affect future search behavior, so context arrivals are not fully policy-invariant. Moreover, per-search exploration schemes such as standard $\epsilon$-greedy or Thompson sampling do not preserve persistent user-level fixed-policy cohorts, which are needed to evaluate counterfactual policies such as ``always show three ads'' when ad load can affect subsequent usage.

Fourth, we therefore design Locally Adaptive Ad Load (LAAL) and its exploration-augmented implementation, e-LAAL. LAAL is a model-free, query-level, sliding-window policy. For each query--ad-load pair, it estimates recent mean revenue and total conversions, maps these estimates into a scalarized reward, and chooses the ad load with the highest estimated value. This design makes the policy both query-adaptive and time-adaptive: different queries can receive different ad loads, and the same query can receive different ad loads as recent market conditions evolve. The policy does not require a parametric reward model or full observability of latent market states; it can react to changes reflected in recent outcomes, even if the source of the change is not explicitly labeled.

The deployed e-LAAL architecture adds a crucial exploration layer. Rather than randomizing exploration within a single adaptive arm, e-LAAL runs LAAL on 90\% of users while keeping static arms $C,T_2,\ldots,T_6$ active on the remaining 10\%. These static arms provide persistent support for every ad-load level and contemporaneous fixed-policy counterfactuals. We also provide a finite-time high-probability dynamic-regret guarantee for e-LAAL. The proof adapts non-stationary bandit arguments to our deployment architecture by accounting for the fact that exploration is assigned through persistent user-level cohorts, while realized query-level support for each ad-load arm may deviate from the platform's designed traffic shares. 

As part of this design, we evaluate e-LAAL in a 22-day production deployment. The deployment serves 22.33 million users and 77.60 million searches, with the adaptive LAAL arm alone covering 20.10 million users and 69.85 million searches. LAAL recommends 3.1 ads per search on average and produces a bimodal recommendation distribution, often choosing either low ad loads when conversion losses dominate or high ad loads when revenue gains justify additional sponsored slots. Further, we see that LAAL shifts the empirical revenue--conversion frontier outward relative to the deployed static ad-load benchmarks. In the static arms, the familiar trade-off persists: moving from the control condition to the highest ad-load condition increases revenue per user by more than 40\%, but reduces conversions per user by roughly 5\%. LAAL preserves the monetization benefits of high ad loads while substantially mitigating their conversion losses: it delivers revenue essentially identical to the always-5-ads policy ($T_5$), while generating about 3\% more conversions and keeping conversions much closer to the low-ad-load benchmarks. The system-level e-LAAL mixture remains close to the LAAL arm, indicating that the 10\% static exploration layer preserves ongoing learning and counterfactual support with limited performance cost.

We further unpack the mechanism behind the frontier improvement. Consistent with the heterogeneity analysis, LAAL recommends fewer ads for low-CVR queries and more ads for high-CVR queries, without hard-coding these segments into the policy. For non-stationary brand queries, LAAL changes its recommendations within the same query: when the brand advertiser is absent, the policy shifts toward lower ad loads; when the brand advertiser is present, it shifts toward higher ad loads.

Finally, we benchmark LAAL against both uniform and query-dependent static policies. Among uniform policies, the best historical rule selected from the 66-day experiment is $T_3$, whereas the best uniform rule during deployment is $T_5$, showing that even the best platform-wide static ad load changes over time. LAAL outperforms the historical uniform benchmark and weakly dominates the deployment-period uniform oracle by matching its revenue while delivering significantly higher conversions. A historical query-dependent benchmark improves over uniform static rules, confirming the value of query sensitivity, but remains below LAAL during deployment. Relative to an ex post deployment query-dependent empirical oracle, LAAL's empirical regret is close to zero for the highest-volume queries, where recent local estimates are most precise. 

%Robustness checks further show that LAAL's performance is stable over the deployment window and that its gains persist on non-brand queries, indicating that the improvement is not driven solely by brand-query mechanics.

%Overall, the findings show that ad-load design in sponsored search is an adaptive allocation problem. Higher ad loads increase revenue on average but reduce conversions and engagement; the strength and even the sign of the revenue effect vary across queries; and the same query's optimal ad load can change over time. e-LAAL translates these empirical insights into a scalable policy architecture: it uses local, recent, model-free evidence to adapt ad load across queries and over time, while maintaining static exploration arms that provide persistent support and credible fixed-policy benchmarks. The result is a deployable system that improves the revenue--conversion frontier at the scale of tens of millions of users and searches.

\subsection{Contributions and Implications}

The paper contributes to the literature on monetization in digital platforms, ad-load effects, and adaptive policy design. 

First, from a substantive perspective, we provide large-scale causal evidence on the effects of ad load in sponsored search and show why the optimal ad load cannot be reduced to a single platform-wide or time-invariant query-level rule. Our 66-day experiment exposes over five million users to an unusually wide range of ad loads, from the status quo single-ad layout to a six-ad layout that covers most of the first page with sponsored results. This design allows us to quantify the revenue--consumer-experience trade-off induced by ad load, as reflected in conversions and engagement, over a scale and time horizon not previously studied in sponsored search. We further show that this trade-off is highly heterogeneous: query-level advertiser quality determines whether additional sponsored slots generate incremental revenue, while market-state shifts, such as brand-advertiser presence, can move the same query to a different revenue--conversion frontier. To our knowledge, this is the first study to provide causal evidence on sponsored-search ad-load effects over this range of ad loads while also documenting that the optimal ad load varies both across queries and over time.

Second, from a methodological perspective, we provide a constructive and theoretically grounded solution to the adaptive ad-load allocation problem. LAAL is a model-free, query-level decision rule that uses sliding-window estimates of recent revenue and conversion outcomes to adapt ad load across queries and over time. e-LAAL is the deployable architecture that augments LAAL with persistent static exploration arms; these arms preserve support for each ad-load level and provide contemporaneous fixed-policy counterfactuals while the adaptive arm exploits recent evidence. In other words, e-LAAL is the practical implementation framework a platform can adopt directly: it implements LAAL in production, optimizes ad load at scale, and continues to collect the experimental data needed for ongoing learning and evaluation. We also provide a finite-time high-probability dynamic-regret bound for e-LAAL that explicitly accounts for the gap between designed traffic shares and realized query-level support. Empirically, in a platform-level production deployment serving over 22 million users and 77 million searches, LAAL improves the revenue--conversion frontier, delivers revenue comparable to high-ad-load static policies while preserving conversions closer to low-ad-load policies, and adapts in economically sensible ways by recommending fewer ads when advertiser quality is low and more ads when sponsored demand is valuable.

Finally, from a managerial perspective, the results show that ad load decisions in sponsored search markets should neither be set platform-wide nor governed by a time-invariant query-level rule. Effective ad-load design must be both query-sensitive and time-adaptive: the query determines the baseline value of additional sponsored slots, while current market conditions determine whether that value has shifted. The 22-day e-LAAL deployment demonstrates that this principle is operationally feasible at a production scale. The approach is also practical and portable because it does not require a platform-specific parametric reward model, complete observability of advertiser states, or extensive engineering. Instead, it relies on recent local outcomes and a small set of fixed exploration cohorts, making it adaptable to sponsored search and e-commerce platforms that track query-level revenue and conversion outcomes. More broadly, the paper shows that platforms can ease the revenue--consumer-experience trade-off not only through auction design, pricing, ranking, or targeting, but also by dynamically controlling the supply of sponsored attention.

\section{Related Literature}

Our paper contributes to four complementary streams at the intersection of marketing, economics, and computer science: (i) sponsored-search market design, (ii) empirical work on advertising and ad-load trade-offs, (iii) ad-load allocation systems, and (iv) adaptive decision-making with contextual bandits. 

First, our paper connects to the large literature on auction design in sponsored search and its implications for platform revenue. Early work in this area examines how auction formats and allocation rules affect outcomes and revenue \citep{edelman2007internet, Varian2007, LahaiePennock2007, AtheyEllison2011}. Building on \citet{Myerson1981OptimalAuction}, a central line of research has studied how to set optimal reserve prices in advertising auctions to improve platform revenues \citep{YuanWangChenMasonSeljan2014, PaesLemeEtAl2016FieldGuideReserves, OstrovskySchwarz2023}. This literature primarily focuses on optimizing pricing and ranking mechanisms for a given ad load. Apart from auction design, more recent work has examined how real-time bidding, behavioral targeting, and privacy regulation reshape allocation, pricing, and platform revenue in digital advertising markets \citep{Sayedi2018, YaoMela2011, RafieianYoganarasimhan2021, GoldfarbTucker2011}. We complement this literature by focusing on ad load as a supply-side design choice. We empirically quantify how ad load affects platform revenues and user-side outcomes, propose an adaptive framework that the platform can use to decide how many sponsored slots to display for any given search query at a given point in time, and evaluate that choice using platform-level outcomes rather than advertising metrics alone. Our work thus extends this research by endogenizing the amount of sponsored inventory shown to users for any given query, and by characterizing how this operational lever can improve platform performance.

Second, our work contributes to the literature that measures the effects of advertising on consumer behavior. Early work on TV ads finds that reducing advertising time increases audience size \citep{Wilbur2008TwoSidedTV,  wilbur2013correcting}. More recent work has focused on digital platforms. \citet{goli2025measuring} document consumer sensitivity to ad load in a music streaming platform and find that heavier ad loads reduce user engagement and increase churn in the long run, with heterogeneous effects across age groups. In an e-commerce setting, \citet{Moshary2025SponsoredSearch} compares two conditions---no ads and a default ad load---and finds that the number of transactions and commission revenue are higher in the no-ads condition, but that this increase in transactions does not offset the loss of advertising revenue. Finally, in a search engine sponsored-search setting, \citet{SahniZhang2024SponsoredMessages} find that reducing average ad load in the mainline above organic results from 1 to 0.83 reduces ad revenue but does not significantly affect clicks. 

In our setting, we vary ad load from one to six and find a clear revenue--conversion trade-off: higher ad loads increase ad revenue but reduce overall app installs and engagement. Beyond the main effects, we document that these effects are heterogeneous across queries and over time, implying that the revenue--conversion trade-off can shift across queries and market states. Moreover, while the earlier papers mainly document ad-load effects, we design and deploy an adaptive ad-load allocation policy that balances revenue and conversions, and present results from this platform-level deployment. Thus, our main goal is not just to show that these trade-offs exist or to quantify them; rather, we propose a methodological framework that decision-makers can use to adaptively allocate ad load in sponsored search settings.\footnote{Indeed, we show that these trade-offs can be brittle and change over time, and that choosing a fixed ad load based on a historical experiment can lead to suboptimal outcomes; see $\S$\ref{ssec:benchmark_static}.}

Our work also relates to the literature that examines how to optimally allocate ad load in advertising markets. In sponsored search, \citet{Broder2008Swing} study when to show no ads based on predicted click-through rates, while later supervised and whole-page optimization approaches use historical click patterns or predicted page-level click and revenue outcomes to determine how many ads to display \citep{Wang2011HowManyAds, Zhang2018WholePage}. Relatedly, \citet{KimBalachanderKannan2012} provide a theoretical analysis of how many sponsored slots a platform should offer under GSP when the objective is advertising revenue. These papers are important precursors, but they primarily rely on predicted user-response models based on observational data, historical logs, or analytical models. In contrast, we first run a large-scale randomized field experiment that exogenously varies ad load from one to six ads, allowing us to causally quantify the revenue--conversion trade-off and show that it varies across queries and over time. We then use these findings to design and deploy a model-free adaptive policy with persistent static exploration arms, contemporaneous fixed-policy counterfactuals, and a finite-time dynamic-regret guarantee.

More recent work treats ad load as a platform-level decision, incorporating user-engagement proxies through multi-objective rewards, constrained optimization, off-policy learning, or doubly robust estimation. In feed allocation and social-media ad supply, related work optimizes ad exposure within an ongoing stream of organic content, using signals such as feed position, session context, user ad tolerance, ad fatigue, and long-run engagement \citep{Liao2021CrossDQN, Sagtani2023AdLoadOPE, Yan2020AdsAllocation, Shi2024AdsSupplyDR}. In audio streaming, \citet{Goli2025} studies ad load as a personalized exposure problem, using long-run experiments to balance advertising revenue, subscription revenue, and user consumption. These studies establish ad load as an important operational lever, but their settings differ from sponsored search in both the decision unit and the main source of heterogeneity. In feeds and streaming, the platform manages ad exposure over a consumption path, and heterogeneity is primarily user- or session-level. In sponsored search, by contrast, the query is the central state variable: it reveals user intent, determines which organic results are displaced, and shapes the set and quality of eligible advertisers. Search platforms also observe immediate search-level outcomes, such as conversions and revenue, that can be linked directly to the ad-load decision. Thus, ad-load methods designed for feed or streaming environments do not directly translate to sponsored search, where the key problem is to adapt the number of sponsored slots across queries and over time as advertiser composition changes.

Finally, our paper relates to and contributes to the literature on contextual bandits and adaptive experimentation \citep{lattimore2020bandit}. In a standard contextual-bandit problem, the decision-maker observes a context, chooses an action, and observes only the reward associated with the chosen action. A large theoretical literature studies algorithms such as LinUCB and Thompson sampling, typically under stable reward mappings and parametric or linear structures in the relationships among contexts, actions, and rewards \citep{li2010contextual, chu2011contextual, abbasi2011improved, agrawal2013thompson}. Related work extends bandit methods to non-stationary environments using sliding windows, discounting, or variation-budget regret criteria \citep{GarivierMoulines2011, CheungSimchiLeviZhu2019NonStationarity, russac2019weighted}. Closer in application, \citet{YeEtAl2023ColdStart} develop and field-test a deployable contextual-bandit policy for new-ad cold start that leverages an incumbent CTR/CVR prediction and auction system. Our setting differs in how exploration must be engineered for deployment. We maintain parallel static ad-load arms because they provide persistent, contemporaneous counterfactuals for fixed policies of the form ``always show \(i\) ads.'' This design introduces a complication absent from canonical bandit models: because ad load can affect users' subsequent search behavior, the realized query-level support in each arm need not match the platform's designed traffic weights. We therefore provide regret bounds for the parallel-static-cohort exploration scheme, explicitly allowing realized support to deviate from intended allocation shares. We also evaluate the approach in a real production deployment, connecting the theoretical exploration design to platform-scale implementation. 

%\rev{In its deployment philosophy, our work is closest to \citet{YeEtAl2023ColdStart}, who augment an existing advertising system with a simple exploration layer rather than rebuilding its machine-learning infrastructure; e-LAAL similarly operates as a lightweight ad-load decision and exploration layer on top of the platform's existing prediction, auction, and ranking systems.}

\section{Setting and Field Experiment}

\subsection{Setting: Android App Store}
Our data come from a leading Android app store in an Asian country that commands over 70\% of the market share. On the demand side, the app store attracts around 40 million monthly users and 5 million daily active users, while on the supply side, it features nearly half a million applications developed by over eighty thousand developers. The platform's primary source of income is advertising within the store, with the most prominent form being \textit{sponsored search ads}, which is the focus of this study. On average, users install about 2.5 million apps daily from search results pages, with about 50,000 of those installs attributed to sponsored ads. 

When a user submits a search query in the app store, the platform runs a real-time Generalized Second-Price (GSP) auction \citep{edelman2007internet} among eligible ads, employing a cost-per-install payment model. Under the platform's standard setting -- referred to as the \textit{single-ad layout} -- one ad slot is allocated to the highest-ranked bidder. This sponsored ad is displayed at the top of the search results page, followed by a rank-ordered list of organic results. Users can scroll down to view additional organic listings. 

An example of the single-ad layout is illustrated in Figure \ref{fig:single_ad}. On a typical smartphone with a 6.3–6.5 inch display (e.g., a Pixel~9), users can see eight listings (ads or organic) without scrolling if the search query is non-brand. A \textit{brand query} is a search query explicitly referring to a particular application or brand. An example of a brand query is \textit{WhatsApp}, while an example of a non-brand query is \textit{Flights}. In the case of a brand query, the brand result is displayed with additional details (e.g., application screenshots), occupying two additional rows and thus reducing the number of visible items to approximately six (see Web Appendix $\S$\ref{apdx:brand_query_layouts} for examples of brand query layouts).
% \hy{In case of a brand-query, the length of the results seems different.. Something to keep in mind for later..}

\subsection{Field Experiment on Number of Ad Slots}
\label{ssec:expt_adslots}
Our analysis leverages data from a large-scale randomized field experiment conducted on the platform’s search results page. The experiment aims to examine how varying the number of sponsored ads displayed impacts user-, advertiser-, and platform-level outcomes. The experiment ran for 66 days starting from September 2023. The experiment consisted of approximately 10\% of the platform's user-base. Users in the experiment were randomly assigned to one of six conditions, as follows: 
\squishlist

\item \textbf{Control Group ($C$):} Users in the control group experience the \textit{single-ad layout}, displaying one sponsored ad per search query. This group comprises 5\% of the user base.

\item \textbf{Treatment Groups ($T_i$):} We have five distinct treatment groups, $T_2, T_3, T_4, T_5, T_6$. Users assigned to treatment group $T_i$ are exposed to an \textit{i-ad layout}, where the top $i$ slots are assigned to ads for each search (based on the same GSP auction mechanism). An illustrative example of the three-ad layout, corresponding to users in the $T_3$ treatment group, is shown in Figure \ref{fig:three_ads}. Each treatment group comprised approximately 1\% of the platform's user base. 

\squishend

Users are assigned to experimental conditions by hashing each device's unique Android identifier with a deterministic, salt‑aware function to prevent collisions with other experiments. This method yields a random but reproducible treatment assignment. Throughout the experiment, user assignment to the experiment condition remains fixed, ensuring consistent exposure to the treatments/control.\footnote{Although each treatment condition targets a specific number of sponsored ad positions, the actual number of ads shown per search is contingent upon the availability of eligible bidders, meaning fewer ads could be displayed if bidder supply is insufficient (i.e., realized ad load could be lower than treatment condition). However, such discrepancies are rare across all treatments in our data because the ad market is sufficiently thick.}

\begin{figure}[htp!]
  \centering
  \begin{subfigure}[t]{0.45\textwidth}
    \centering
    \includegraphics[width=0.5\linewidth]{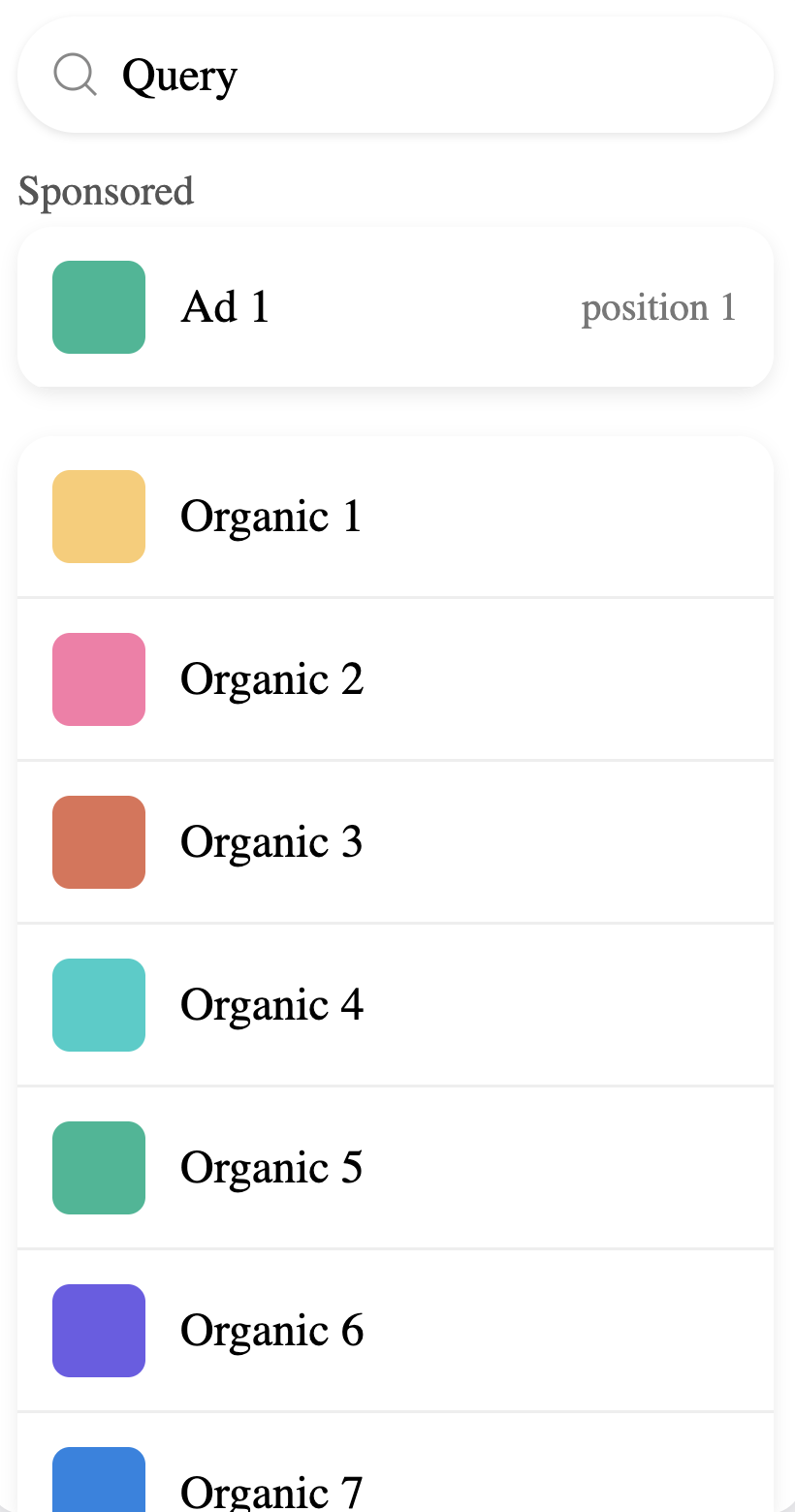}
    \caption{Single-Ad Layout }
    \label{fig:single_ad}
  \end{subfigure}
  \hfill
  \begin{subfigure}[t]{0.45\textwidth}
    \centering
    \includegraphics[width=0.5\linewidth]{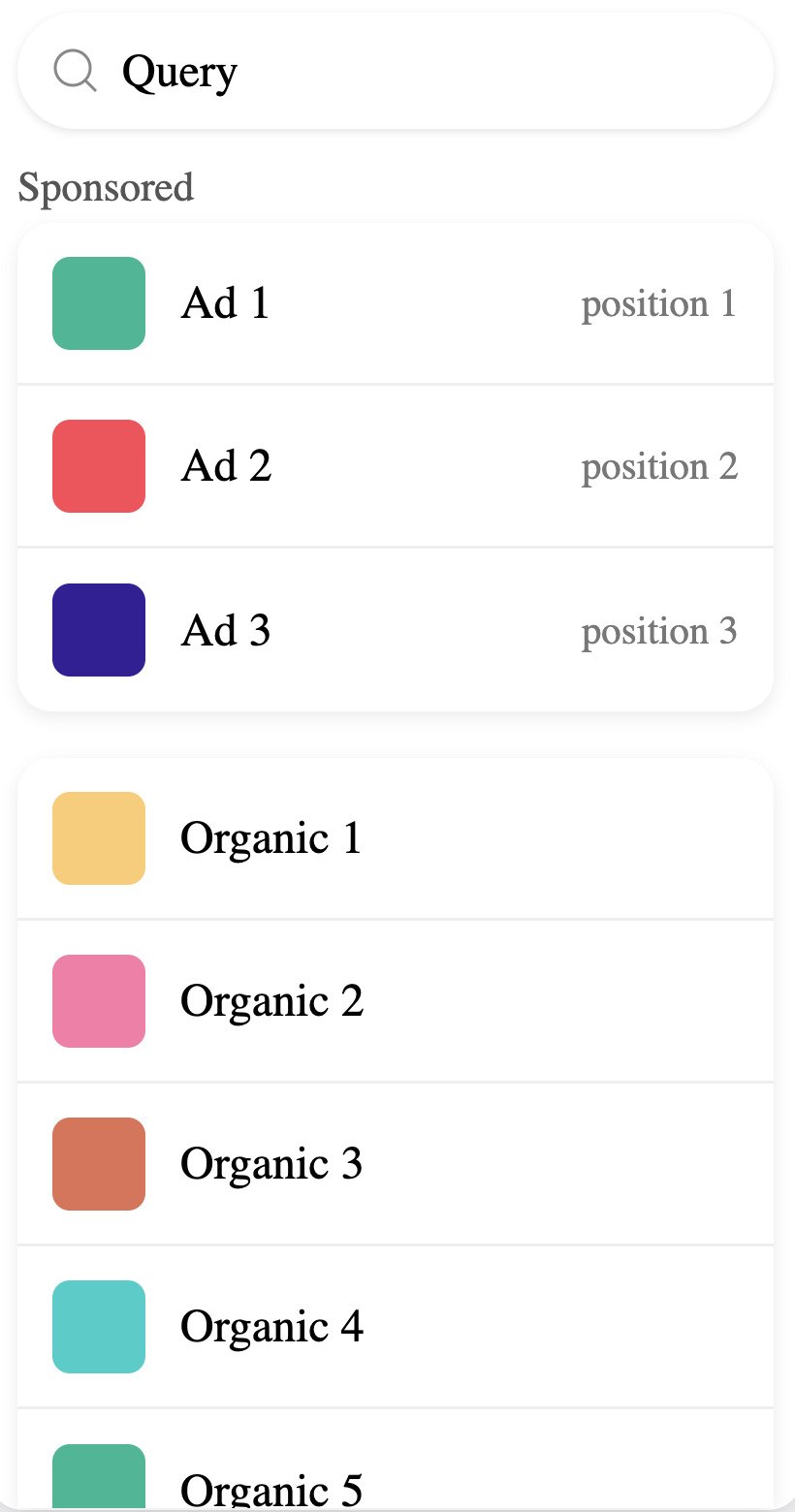}
    \caption{Three-Ad Layout}
    \label{fig:three_ads}
  \end{subfigure}
  \caption{Comparison of single-ad versus three-ad sponsored search layouts.}
  \label{fig:compare_ads}
\end{figure}

Overall, the experiment contains observations from 5,057,952 unique users who performed 26,326,936 searches, installed 377,425 ads, and 15,514,923 organic apps, over the 66-day experiment period. 

\subsection{Data}

Our data are at the impression-level. An impression is any search result displayed to a user in response to their query. Thus, each organic link or ad shown on the first page counts as a separate impression, and if the user scrolls to subsequent pages, every additional result shown is likewise recorded as its own impression.

For each impression in the data, we observe the following variables:
\squishlist
    \item Search ID ($s$): The unique identifier of a search.
    
    \item Query ($q$): The search term used by a user. We denote the query associated with Search ID $s$ by $q_s$. Our dataset contains 4,045,348 unique queries, with the top 2,000 accounting for 60\% of all searches. This concentration is consistent with prior work on search platforms, where a small share of keywords generates most searches (e.g., \citet{yoganarasimhan2020search}). Web Appendix $\S$\ref{apdx:search_CDF} reports the full cumulative distribution of query search counts. Most queries belong to the Social Media, Games, and Tools categories.
    
    \item Device ID ($u$): The unique identifier of the user, associated with the Android device. We denote the user associated with search ID $s$ by $u_s$.
    
    \item Experiment Condition ($e$): The experiment condition assigned to a user, drawn from the set \\ $\mathcal{E} = \{C,T_2,T_3,T_4,T_5,T_6\}$.  We denote the experimental condition associated with search $s$ by $e_s = e(u_s)$. 
    
    \item Timestamp ($t$): Exact time when the result is shown to the user. We denote the time of a search request $s$ by $t_s$.\footnote{Following a search, all the results displayed on the first page of the results (both ads and organic listings) share the same timestamp $t_s$. If the user navigates to the next page in a few minutes, the time at which s/he sees the second page of results is denoted as the timestamp for the impressions/listings in the second page, and so on.}
    
    \item App ID: The unique identifier of the apps shown on the platform. We observe 141,716 distinct apps in our data. The top 20\% of ads receive 90\% of ad impressions, while the top 20\% of apps receive 99\% of total impressions of the search page. This reflects how relevance‐based organic ranking algorithms amplify popular items more strongly than auction‐driven ad allocations do. The cumulative distribution of app impressions is shown in Web Appendix $\S$\ref{apdx:apps_ads_CDF}.
    % Figure~\ref{fig:apps_ads_cdfs} shows the cumulative distribution of app impressions at advertising and search(ads and organic combined) levels.
    %\textcolor{red}{Cite? Ads are more diverse than Organics?}
    \item Position ($p$): Rank position of the app in search results. 
    
    \item Is Sponsored: A binary flag indicating whether the impression/app appears as an ad. We have ad impressions for 664 unique advertisers; the most frequent categories among sponsored apps are Finance, Shopping, and Entertainment.
    
    \item Conversion Indicator ($\text{Conv}_{s,p}$): A binary indicator, defined at impression level, showing whether the user installed the app in position $p$ for search ID $s$. 
   
    We denote the sum of conversion outcomes (sum over all positions) of search ID $s$ by $\text{Conversion}_s$:
    \begin{equation}
    \text{Conversion}_s = \sum_p \text{Conv}_{s,p}
    \end{equation}
    % This extra concentration of conversion outcomes in the head is well documented in digital marketing settings. 
    Please see Web Appendix $\S$\ref{apdx:apps_ads_CDF} for the cumulative distribution of app conversions.
    
    \item Revenue ($\text{Rev}_{s,p}$): Revenue from the advertiser upon sponsored app installation, in position $p$ for search ID $s$. By definition, organic listings generate zero revenue; for ads, revenue is determined by the GSP mechanism discussed earlier. To preserve the platform's privacy, we multiply all revenue numbers by an undisclosed multiplier. We denote the sum of revenue for search ID $s$ by $\text{Revenue}_s$,
    \begin{equation}
        \text{Revenue}_s = \sum_p \text{Rev}_{s,p}
    \end{equation}
\squishend

\section{Effect of Number of Ad Slots on Conversion and Revenue}
We first describe the main effect of the number of ad slots on various outcomes in $\S$\ref{ssec:main_effects} and then quantify the heterogeneity in treatment effects across queries and time in $\S$\ref{ssec:heterogeneity}.

\subsection{Main Effects}
\label{ssec:main_effects}

Table~\ref{tab:summary_user_experiment} reports summary statistics of the various outcomes by experimental condition. The top panel shows the realized ad load and the number of users in each condition. We see that the realized ad exposure closely tracks the intended design (top row). Next, we focus on the two main outcomes of interest, shown in the middle panel of Table \ref{tab:summary_user_experiment}.  \emph{Revenue per User} is total advertising revenue from all the users in the experiment bucket divided by \emph{User Count} in that bucket. \emph{Conversion per User} is the total number of conversions from ads and organic listing divided by \emph{User Count}. We see that \emph{Revenue per User} increases monotonically with ad load: from 35.72 under $C$ to 51.18 under $T_6$ -- a 43\% improvement relative to control. In contrast, \emph{Conversions per User} declines monotonically from 3.19 ($C$) to 3.03 ($T_6$), a 5\% reduction. 

\begin{table}[htp!]
\centering
\small
\setlength{\tabcolsep}{3.5pt} % Tightens column spacing
\caption{Summary Statistics of the Experiment}
\label{tab:summary_user_experiment}
\begin{tabular}{lcccccc}
\toprule
 & \textbf{$C$} & \textbf{$T_{2}$} & \textbf{$T_{3}$} & \textbf{$T_{4}$} & \textbf{$T_{5}$} & \textbf{$T_{6}$} \\
\midrule
Avg.\ Ad Load & 1.00 & 2.00 & 3.00 & 3.99 & 4.94 & 5.88 \\
User Count & 2,528,428 & 505,193 & 505,016 & 506,335 & 506,713 & 506,267 \\
\midrule
Revenue/User\textsuperscript{a} & 35.72 & 41.61 & 44.54 & 46.39 & 49.47 & 51.18 \\
Conversion/User & 3.19 & 3.15 & 3.13 & 3.10 & 3.06 & 3.03 \\
\midrule
Daily Eng.\ (\%) & 4.95 & 4.92 & 4.89 & 4.87 & 4.85 & 4.84 \\
Searches/User & 5.14 & 5.16 & 5.30 & 5.31 & 5.28 & 5.31 \\
\bottomrule
\end{tabular}
\par\medskip
\footnotesize{\textit{Note:} \textsuperscript{a}Scaled by an undisclosed multiplier.}
\end{table}

To assess statistical significance, we estimate separate user-level regressions for outcomes \\
$Y\in\{\text{Revenue, Conversion}\}$:
\begin{equation}
\label{eq:user_outcome_ols}
Y_u \;=\; \alpha_C^Y \;+\; \sum_{i=2}^6 \alpha_{T_i}^Y\,\mathbf{1}\{e(u) = {T}_i\} \;+\; \varepsilon_u^Y,
\end{equation}
where:
\squishlist
\item $Y_u$ is the outcome of user $u$ such that $Y_u = \sum_{s:u_s=u}Y_s$
\item $\alpha_C^Y=\mathbb{E}[Y_u\mid e(u) = {C}]$ is the control mean for outcome $Y$
\item For treatment conditions $T_i$, $\alpha_{T_i}^Y$ measures the deviation from control as follows:
\begin{equation}
\label{eq: main_treatment_effect}
\alpha_{T_i}^Y \;=\; \mathbb{E}[Y_u\mid e(u) = {T}_i] - \mathbb{E}[Y_u\mid e(u) = {C}] \;=\; \text{ATE}^Y(T_i).
\end{equation}
\squishend
We present the estimation results from Equation  \eqref{eq:user_outcome_ols} in Table~\ref{tab:user_ols_condition}. It confirms the previously observed pattern: for revenue, treatment effects are positive, increasing in ad load, and highly significant; for conversion, effects are negative, more pronounced at higher ad loads, and likewise highly significant. These findings suggest an important trade-off for the platform: more ad slots raise revenue while reducing total conversions. 

%These aggregate patterns are consistent with other findings in the literature: While the ad load can increase revenue, it can depress total consumer surplus or efficiency \citep{Moshary2025SponsoredSearch}.

We now examine how ad load changes consumers' behavior and whether consumers in our setting are averse to advertising. To answer these questions, we consider two platform usage metrics, shown in the last two rows of Table \ref{tab:summary_user_experiment}. For each day and experimental condition, we calculate the proportion of assigned users who performed at least one search. \textit{Daily Engagement}, defined as the mean of these daily proportions over the duration of the experiment, declines from 4.95\% in $C$ to 4.84\% in $T_6$ (a 2.2\% relative decrease), with significant differences versus control across all treatment conditions. Within high-ad-load groups ($T_4$–$T_6$), we also detect a modest but significant downward trend over time (See Web Appendix $\S$\ref{apdx:engagement_regressions} for a detailed analysis). This result shows that ad load harms platform usage, and consumers with higher ad loads are less willing to use the search feature.  This finding is similar to that of \citet{goli2025measuring}, who report a near-linear decline in total hours spent on an audio streaming platform as ad load increases. At the same time, \textit{Searches per User} rises from 5.14 ($C$) to 5.31 ($T_6$), a 3.3\% increase despite the fact that the unique number of active users decreases with higher ad load. This is likely because high ad exposure degrades the quality of search results; users cannot find the desired app at the top, prompting additional searches. A similar pattern of higher search volume with a degraded search experience has also been observed in Bing search experiments \citep{kohavi2012trustworthy}.  We present a detailed analysis of this effect of ad load on search behavior (including regressions) in Web Appendix $\S$\ref{apdx:search_per_user_regression}.

\begin{table}[htp!]
\centering
\small
\caption{Effect of Ad Load on Revenue and Conversion}
\label{tab:user_ols_condition}
\begin{tabular}{lcc}
\toprule
 & \textbf{Revenue} & \textbf{Conversion} \\
\midrule
$C$ & 35.722$^{***}$ (0.255) & 3.192$^{***}$ (0.005) \\
$T_2$ & 5.890$^{***}$ (0.626) & -0.046$^{***}$ (0.013) \\
$T_3$ & 8.818$^{***}$ (0.593) & -0.063$^{***}$ (0.015) \\
$T_4$ & 10.669$^{***}$ (0.507) & -0.096$^{***}$ (0.016) \\
$T_5$ & 13.746$^{***}$ (0.524) & -0.133$^{***}$ (0.013) \\
$T_6$ & 15.455$^{***}$ (0.545) & -0.165$^{***}$ (0.013) \\
\midrule
Obs. & 5,057,952 & 5,057,952 \\
$R^2$ & $<$0.001 & $<$0.001 \\
\bottomrule
\end{tabular}
\par\footnotesize{\textit{Note:} The \(C\) row reports the mean outcome in the control condition. Rows \(T_2,\dots,T_6\) report differences relative to \(C\). Standard errors in parentheses. ${}^{***}p<0.001$. }
\end{table}

%\rev{The near-zero $R^2$ is expected in user-level randomized data: idiosyncratic cross-user variation in outcomes dwarfs the share explained by treatment indicators. This has no bearing on the precision of the estimated treatment effects, which is reflected in the standard errors.}

In sum, increasing the number of ad slots produces a monotonic increase in revenue and a monotonic decrease in total conversions, which frames a tradeoff in the platform’s design problem. Meanwhile, ad load hurts consumers' experience and platform usage.  In the remainder of the paper, we focus on conversion as the main metric for consumer satisfaction for two reasons. First, app installations are the platform’s primary matching outcome: conversions directly capture whether users find and adopt relevant apps. Second, conversions are observed at the search level and therefore provide immediate feedback that can be directly linked to search-level decisions. 

%\footnote{One may contrast this finding with the view that advertising signals quality and thus reduces consumers' search costs \citep{nelson1974advertising, AtheyEllison2011}. In our setting, the organic ranking already surfaces highly relevant results (15.5 million of the roughly 15.9 million search installs in the experiment are organic), so additional sponsored slots largely displace strong organic matches rather than reveal hidden high-quality options. However, the effect depends on advertiser quality: as we show in \S\ref{subsec:het-queries}, additional slots create value when advertiser quality is high, whereas on low-CVR queries the sponsored slots are filled by options weaker than the organic results they displace.} 

%\rev{We note that these user-side effects are likely conservative relative to more competitive settings: the platform we study is a near-monopolist in its market, so users have limited ability to switch to alternative distribution channels. On platforms facing stronger competition, heavy ad loads would likely induce larger disengagement and switching responses.}

%

% \item \emph{Notes:} Each column estimates 
% $\;Y_{q,e} = \gamma_q + \sum_{j \neq C} \beta_j \mathbf{1}\{e=j\} + \varepsilon_{q,e}$,
% with $Y$ equal to Revenue (col.~1) or Installs (col.~2). Query fixed effects are absorbed by demeaning. 
% $\gamma$ reports the expected outcome under the control condition $C$ (averaged across queries). 
% Standard errors in parentheses are clustered at the query level.  
% ${}^{*}p{<}0.10$, ${}^{**}p{<}0.05$, ${}^{***}p{<}0.01$.

\subsection{Heterogeneity in the Revenue–Conversion Trade-off}
\label{ssec:heterogeneity}

So far, we have shown that more ad slots generally increase revenue and decrease conversions. We now examine how the revenue-conversion trade-off varies across queries and over time.  

%Specifically, we show that: (i) the average advertising conversion rate on a query under the control condition is an important factor in determining relative revenue gain from a higher number of ads. Specifically, we show that on queries with low ad-conversion rate, adding more sponsored slots does not increase revenue, and (ii) Revenue and conversion outcomes could significantly shift over time. We use the example of brand advertising to show how the presence/absence of a brand advertiser can shift the main outcomes. Specifically, we find that brand advertisers' presence improves both revenue and conversion across conditions, with the greatest improvements in both outcomes occurring under high ad load conditions.

\subsubsection{Heterogeneity Across Queries}
\label{subsec:het-queries}

We now examine how treatment effects vary across queries. We first segment queries into three buckets based on the quality of their ads. We select the top 10{,}000 queries based on advertising conversion count and rank them by their control group ad conversion rate. Formally, for query $q$ we define ads' conversion rate based on their performance in the control condition as follows:
$$\text{AdsCVR}_q = \mathbb{E}[\text{Conv}_{s,1} \mid e_s=C, q_s=q],$$
where $\text{Conv}_{s,1}$ is the conversion indicator of the (single) sponsored position in the control group. We only use the control condition for labeling ads because it keeps the position of the ad fixed to the first position.

We then partition queries based on AdsCVR into three equal-sized segments as follows:$K_1=\text{lowest 3{,}333}$, $K_2=\text{middle 3{,}334}$, and $K_3=\text{highest 3{,}333}$. $K_1$ consists of the set of queries with the lowest advertising conversion rate under the control condition, while $K_3$ consists of the highest ad conversion rate queries. We use three segments to balance interpretability and statistical precision: each segment has sufficient search volume for precise estimation while still capturing meaningful variation in ad load effects. See Web Appendix $\S$\ref{apdx:query_bucket_summary} for the summary statistics of these three segments of queries. 

To estimate how ad load affects outcomes within each segment, we run the following query–fixed-effects regression for $Y \in \{\text{Revenue},\text{Conversion}\}$:
\begin{equation}
\label{eq:bucket}
Y_{s} \;=\; \alpha_{q_s}^Y \;+\; \sum_{k=1}^3\sum_{i=2}^6 \beta_{k,i}^Y\,
\mathbf{1}\{q_s\in K_k\}\, \mathbf{1}\{e_s = T_i\} \;+\; \varepsilon_{s}^Y
\end{equation}
$Y_s$ is the observed outcome for search $s$. The query fixed effect $\alpha_{q_s}^Y$ absorbs all baseline differences across queries. Thus, $\beta_{k,i}^Y$ isolates the incremental effect of treatment $T_i$ relative to control within segment $k$. 

We report the estimation results from Equation \eqref{eq:bucket} for all three segments of queries in Table~\ref{tab:bucket-fe-estimates} in Web Appendix $\S$\ref{apdx:query_bucket_summary}. To aid interpretation, we plot the estimated effects in Figure~\ref{fig:het-bucket-plots} after normalizing each segment’s treatment effects by its own control mean (so a value of 1.10 corresponds to a 10\% increase relative to the segment’s control mean). \footnote{Points in Figure~\ref{fig:het-bucket-plots} show the treatment-effect estimates from Equation~\eqref{eq:bucket}. Confidence intervals are reported separately for each treatment effect. Although confidence intervals overlap across some ad-load conditions, pairwise tests confirm that revenue effects increase significantly with ad load in the High-CVR segment.} We observe two patterns. First, we find that conversion decreases significantly as ad load increases in all three segments. Further, the magnitude and slope of this decline are broadly similar across all three segments. However, revenue displays strong heterogeneity across segments. In the low segment ($K_1$), additional ad slots do not raise revenue (on average), and in fact lead to significantly negative effects for $T_3$ and $T_4$. In the medium segment ($K_2$), revenue becomes positive and significant at higher ad loads. In the high segment ($K_3$), revenue rises sharply and monotonically with the number of ads. These patterns reveal a sharp asymmetry in how the top advertiser’s conversion rate mediates the effect of additional ad slots. For queries with low advertiser quality, additional ads fail to generate incremental paid installs. For such queries, more ads displace relevant organic results, depressing both revenue and conversion.\footnote{One reason revenue can decline rather than merely remain flat is that additional sponsored slots may dilute attention to the top sponsored result: when low-quality ads occupy more of the visible page, users appear less likely to convert on the first sponsored listing, and the additional lower-ranked ads do not generate enough paid installs to offset this loss.} In contrast, for queries with high-quality advertisers (who have high conversion rates), increasing ad load unlocks substantial revenue gains, albeit with expected conversion losses.

This contrast implies an important takeaway for platform design. A “one-size-fits-all” ad load policy is inefficient: showing many ads universally can backfire for queries with low-quality ads. As such, platforms should adopt query-sensitive allocation rules that limit exposure when advertiser quality is poor and expand ad supply when strong advertisers are present. Such query-specific or adaptive ad load allocation can improve monetization efficiency while reducing user-side friction caused by irrelevant ad clutter.

\begin{figure}[htp!]
  \centering
  \begin{subfigure}[t]{0.4\textwidth}
    \centering
    \includegraphics[width=\linewidth]{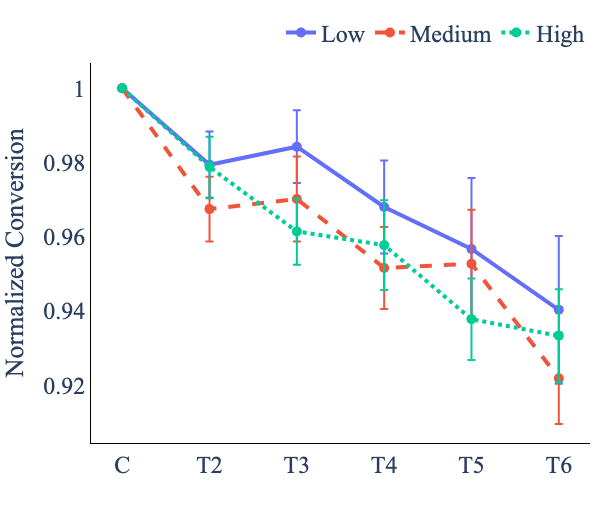}
    \caption{Normalized Conversion (Control=1)}
    \label{fig:het-bucket-conv}
  \end{subfigure}
  \hfill
  \begin{subfigure}[t]{0.4\textwidth}
    \centering
    \includegraphics[width=\linewidth]{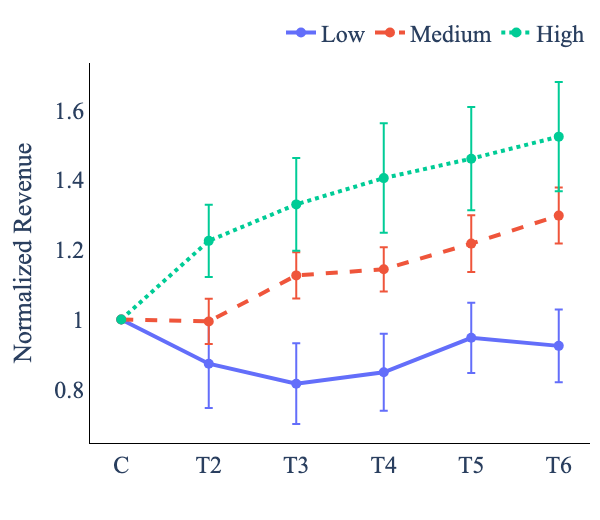}
    \caption{Normalized Revenue (Control=1)}
    \label{fig:het-bucket-rev}
  \end{subfigure}
  \caption{Treatment effects by Ad CVR segment. }
  \label{fig:het-bucket-plots}
  % \par\medskip
  % \footnotesize
  % \textit{Notes:} Points show treatment effects from Equation~\eqref{eq:bucket} normalized by the segment-specific control mean. Error bars denote 95\% confidence intervals reported for each treatment effect separately. 
\end{figure}

\subsubsection{Heterogeneity Over Time}
\label{sssec:het_time}
We now examine how the revenue-conversion trade-off varies over time using the presence or absence of brand advertising as a salient time-varying market-state. In this analysis, we focus on {\it brand queries} -- queries which refer to a specific brand/app, e.g., WhatsApp. We first split the 8{,}561 brand queries in our data into two groups:
\squishlist
\item  \emph{Brand-Active Queries}, where the brand app advertises on at least one day during the experiment (709 queries). 
\item  \emph{Brand-Inactive Queries}, where the brand app never advertises (7{,}852 queries).
\squishend

Figure~\ref{fig:on_off_queries} compares the revenue–conversion frontier across these two groups using the same axis scaling in both panels. {\it Brand-Active Queries} exhibit larger revenue gains and smaller conversion losses relative to {\it Brand-Inactive Queries}, resulting in a noticeably steeper frontier. 

However, query-specific differences could still confound the comparison above, i.e., {\it Brand-Active Queries} can be systematically different from {\it Brand-Inactive Queries}. We therefore focus on the 709 {\it Brand-Active Queries} and compare search-level outcomes when the brand ad is present ({\it Brand-On}) vs. absent ({\it Brand-Off}). Let $B_{s}\in\{0,1\}$ indicate whether the brand ad at search ID $s$ is shown. Figure~\ref{fig:brand_presence_portion} shows the distribution of the share of {\it Brand-On} searches by query among the {\it Brand-Active} set. We see significant {\it within-query} variation in the presence/absence of brand-ads for almost all queries. 
We use this variation to measure how outcomes differ between Brand-On and Brand-Off states under each ad load condition. Specifically, we run the following regressions for $Y\in \{\text{Revenue}, \text{Conversion}\}$ controlling for query-condition fixed effect:
\begin{align}
Y_{s}
&= \alpha_{q_s,\,e_s}^Y \;+\; \theta_{\,e_s}^Y \mathbf{1}\{B_{s}=1\}
\;+\; \varepsilon_{s}^Y \label{eq:brand_ad_effect}
\end{align}
where $\alpha_{q_s,\,e_s}^Y$ denote query-condition fixed effects. Thus, $\theta_{e_s}^Y$ measures the within-query change in $Y$ when the brand ad is displayed under condition $e_s$. Table~\ref{tab:brand_on_rev_conv_effect} reports $\hat{\theta}_e^Y$ estimates from this analysis. We see that Brand-On searches have higher revenue and conversion than Brand-Off searches across all ad-load conditions, and the revenue gap is larger at higher ad loads. In other words, brand-advertiser presence is associated with a more favorable revenue--conversion frontier, particularly under higher ad loads.\footnote{Previous studies have attributed these improvements to the brand ad cannibalizing the organic brand result \citep{Simonov2017CompetitionAC}} Figure~\ref{fig:on_off_both} visualizes the revenue and conversion gaps associated with brand presence. Revenue rises sharply under {\it Brand-On} while conversion declines only slightly, illustrating that brand activity meaningfully softens the tradeoff. This finding parallels the cross-query heterogeneity documented in \S\ref{subsec:het-queries}: brand participation raises the quality of the sponsored slate for a query, effectively moving the query toward the high-advertiser-quality regime in which additional sponsored slots are more valuable.

\begin{figure}[htp!]
    \centering
    \begin{subfigure}[t]{0.4\textwidth}
        \centering
        \includegraphics[width=\linewidth]{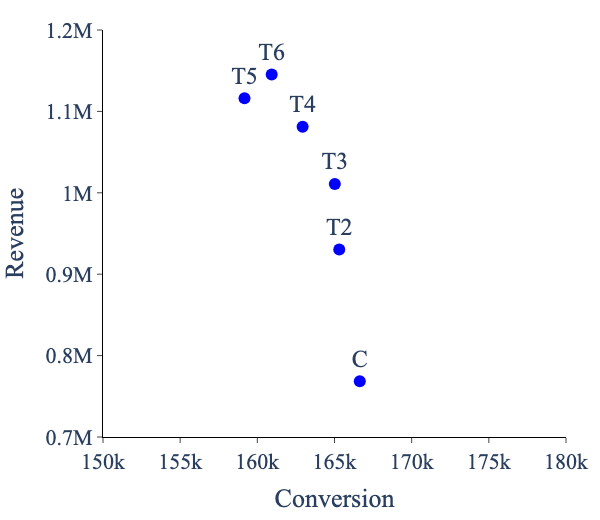}
        \caption{Brand-Active Queries (709 Queries)}
        \label{fig:brand_on_queries}
    \end{subfigure}
    \hfill
    \begin{subfigure}[t]{0.4\textwidth}
        \centering
        \includegraphics[width=\linewidth]{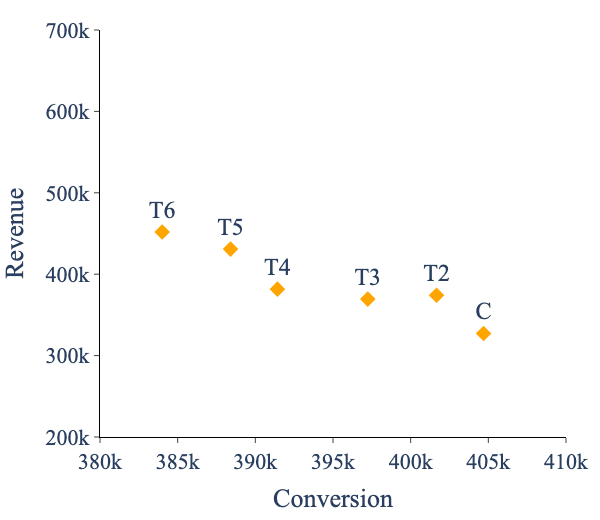}
        \caption{Brand-Inactive Queries (7,852 Queries)}
        \label{fig:brand_off_queries}
    \end{subfigure}
    \caption{Revenue-Conversion tradeoff based on brand activity. Each point corresponds to an experimental condition. The Conversion axis measures total search conversions, i.e., app installs originating from both sponsored ads and organic results; the Revenue axis measures total ad revenue.}
    \label{fig:on_off_queries}
\end{figure}

\begin{figure}[htp!]
    \centering
    \includegraphics[width=0.4\linewidth]{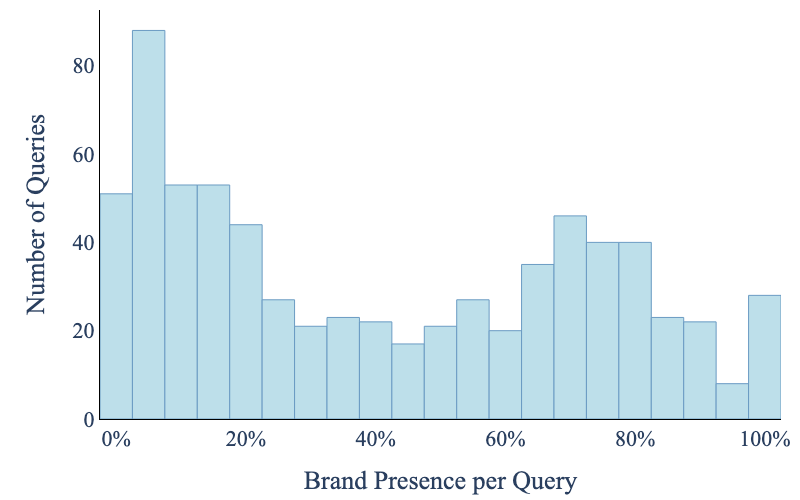}
    \caption{Distribution of {\it Brand-On} searches within {\it Brand-Active Queries}}
    \label{fig:brand_presence_portion}
\end{figure}

\begin{table}[htp!]
\centering
\small
\caption{Brand effect by condition with Query$\times$Condition fixed-effects}
\label{tab:brand_on_rev_conv_effect}
\begin{tabular}{lcc}
\toprule
 & \textbf{Revenue} & \textbf{Conversion} \\
\midrule
C  & 93.05$^{***}$ (4.75)  & 0.104$^{***}$ (0.015) \\
T2 & 116.36$^{***}$ (8.11) & 0.100$^{***}$ (0.018) \\
T3 & 128.70$^{***}$ (9.73) & 0.111$^{***}$ (0.017) \\
T4 & 132.31$^{***}$ (8.82) & 0.115$^{***}$ (0.013) \\
T5 & 139.49$^{***}$ (9.17) & 0.127$^{***}$ (0.014) \\
T6 & 140.03$^{***}$ (7.95) & 0.121$^{***}$ (0.013) \\
\midrule
Query-Condition FE & \checkmark & \checkmark \\
Obs. & 2,253,506 & 2,253,506 \\
$R^2$ & 0.103 & 0.016 \\
\bottomrule
\end{tabular}
\par\vspace{0.25em}
\footnotesize
\textit{Notes:} Each column reports $\hat{\theta}_e$ from specification \ref{eq:brand_ad_effect}. Standard errors in parentheses are clustered at the query level. ${}^{***}p < 0.001$, ${}^{**}p < 0.01$, ${}^{*}p < 0.05$.
\end{table}

\begin{figure}[htp!]
  \centering
  \begin{subfigure}[t]{0.45\textwidth}
    \centering
    \includegraphics[width=\linewidth]{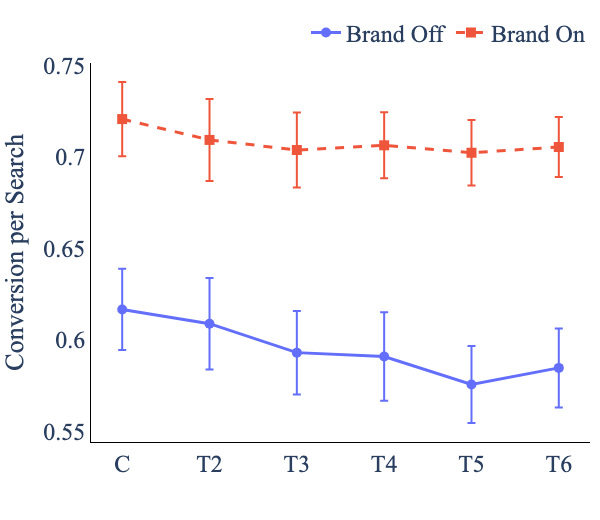}
    \caption{Conversion per search: {\it Brand Off} vs. {\it Brand On}.}
    \label{fig:conversion_on_off}
  \end{subfigure}
  \hfill
  \begin{subfigure}[t]{0.45\textwidth}
    \centering
    \includegraphics[width=\linewidth]{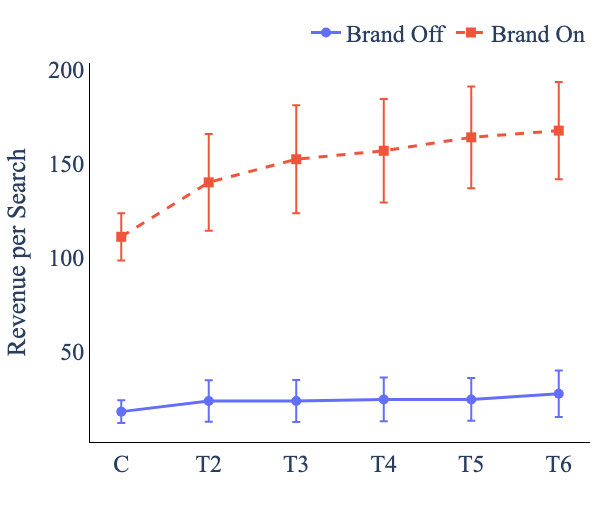}
    \caption{Revenue per search: {\it Brand Off} vs. {\it Brand On}.}
    \label{fig:revenue_on_off}
  \end{subfigure}
  \caption{The effect of brand advertising on left: conversion and right: revenue. The gaps equal to the brand effect estimates from Table \ref{tab:brand_on_rev_conv_effect}.}
  \label{fig:on_off_both}
\end{figure}

Taken together, these results show that brand activity systematically affects the revenue–conversion frontier, even within the same query. When the brand advertises, the marginal cost of showing additional ads is substantially lower: revenue gains grow while conversion losses shrink. For platforms, this highlights the importance of incorporating time-varying market conditions into ad-load decisions. A static, time-invariant allocation rule will systematically misallocate ad load (e.g., showing too few ads when brands are actively advertising and too many when they are not). Effective ad-load design, therefore, requires adaptive mechanisms that respond directly to shifts in advertiser presence and market composition.

\paragraph{Summary and implications for policy design.} 
The heterogeneity results show that the revenue--conversion frontier is neither uniform across queries nor stable over time. Across query segments, the marginal conversion loss of adding sponsored slots is broadly similar, but the marginal revenue effect varies sharply with advertiser quality. We find that additional slots generate little incremental sponsored installs---and can even reduce revenue---for queries with low ad-conversion rates, whereas they deliver substantial incremental revenue for queries with high ad-conversion rates. Furthermore, within a given query, the frontier shifts with market composition: Brand-On states are associated with higher revenue and total conversions than Brand-Off states, and these differences are largest at higher ad loads. In our reduced-form analysis above, we operationalized non-stationarity along one salient dimension---the presence or absence of the brand advertiser on a given query---because it is observable and generates substantial within-query variation. However, non-stationarity can arise from many other forces, e.g., entry/exit of other advertisers, budget exhaustion, shifts in bids and predicted relevance, seasonality, or changes in user intent. So a policy that conditions only on brand presence would be incomplete; instead, the platform needs a general adaptive rule that updates from recent outcomes rather than hard-coding a single market-state indicator. Taken together, these patterns imply that any \emph{static} ad-load rule---platform-wide or fixed per query---will systematically misallocate inventory, showing too many ads when advertiser quality is low and too few when high-quality or brand demand is present. This motivates an \emph{adaptive}, query-level policy that updates ad load using recent outcomes to track changes in advertiser quality and advertiser participation.

%\paragraph{Implications for policy design.}
%Taken together, these results show that the revenue--conversion frontier varies both across queries and within the same query over time. Query-level advertiser quality determines whether additional sponsored slots generate revenue, while changes in advertiser participation, illustrated by brand entry and exit, shift the frontier over time. These facts make static ad-load rules insufficient: platform-wide rules ignore cross-query heterogeneity, and query-fixed rules can become stale as market conditions change. This motivates the sequential decision framework developed next.

% Ads Conversion per Search
%   & 1.22\% & 1.42\% & 1.54\% & 1.64\% & 1.75\% & 1.83\% \\[0.5ex]
% Organic Conversion per Search
%   & 61.04\% & 59.71\% & 57.66\% & 56.87\% & 56.34\% & 55.33\% \\[0.5ex]
% Total Conversion per Search
%   & 62.26\% & 61.13\% & 59.20\% & 58.51\% & 58.09\% & 57.16\% \\[0.5ex]

\section{Problem Setup}
\label{sec:prob_setup}

\subsection{Background}
\label{ssec:background}

%So far, we have established that ad load design is not a one-dimensional monetization choice. Increasing the number of sponsored slots raises short-run revenue but reduces total conversions (and, on average, user engagement). Moreover, the revenue--conversion trade-off varies sharply across queries and shifts over time: for some queries the marginal revenue from additional slots is near zero or negative, while for others additional slots generate substantial revenue gains (e.g., those that attract high-quality advertisers with high conversion rates). The same query can also move to a different frontier as advertiser participation changes. These patterns imply that ``always show \(i\) ads'' rules are inefficient, and even query-fixed rules can become stale when market conditions evolve.

We now model ad load design as a sequential decision problem. For each incoming search request, the platform observes a context and chooses how many sponsored slots to display, then observes noisy outcomes. A useful policy should (i) customize ad load across queries, (ii) adapt to non-stationarity, and (iii) make the revenue--conversion trade-off explicit. The setup below defines the context, actions, and outcomes; introduces a scalar reward representing the platform’s objective; and provides a deployable adaptive policy.

\paragraph{Context, actions, outcomes.}
Consider a sequence of search requests indexed by \(s=1,2,\dots\). Each request arrives with an observable context $x_s \;=\; (t_s, q_s)\in\mathcal{X}$, where \(t_s\) is the time and \(q_s\) is the query.\footnote{The context can be expanded to include user/device attributes or advertiser-set features. We restrict attention to non-personalized ad-load rules for two reasons: (i) the platform avoids user-level ad-load personalization due to privacy and fairness concerns; and (ii) the main sources of heterogeneity in our setting (as documented above) operate at the query and time levels.} Given \(x_s\), the platform chooses an action $a_s \in \mathcal{A} \;=\; \{1,2,\dots,6\}$, 
where \(a_s\) is the number of sponsored slots displayed above organic results. Conditional on \(a_s\), the platform runs a GSP auction and shows the top \(a_s\) sponsored listings.

After the search results are shown, the platform observes a vector outcome $y_s \;=\; \big(\mathrm{Conversion}_s, \mathrm{Revenue}_s\big) \in \mathcal{Y}\subseteq\mathbb{R}_+^2$, where \(\mathrm{Conversion}_s\) is the total number of installs from the search (sum of sponsored and organic installs) and \(\mathrm{Revenue}_s\) is ad revenue generated by the search.\footnote{By definition, \(\mathrm{Revenue}_s=0\) if no sponsored listing is installed. If a sponsored listing at position \(i\) is installed, \(\mathrm{Revenue}_s\) equals the minimum bid required for the advertiser to retain position \(i\) under the GSP rules.}

\paragraph{History, policies, and objective.}
Let the information available at (calendar) time \(t\) be the history of past contexts, actions, and outcomes defined as $H_t \;=\; \{(x_{s'},a_{s'},y_{s'}) : t_{s'} < t\}\in\mathcal{H}$.

A policy maps the current context and available history to an action:
\begin{defn}[Policy]
A \emph{policy} \(\pi\) is a mapping $
\pi:\mathcal{H}\times\mathcal{X}\to\Delta(\mathcal{A})$,
where \(\Delta(\mathcal{A})\) is the probability simplex over \(\mathcal{A}\). Under \(\pi\), the action for search \(s\) is drawn as \(a_s\sim \pi(H_{t_s},x_s)\).
\end{defn}
Because ad load can affect user behavior (e.g., the number of subsequent searches), the realized sequence of contexts \(\{x_s\}\) need not be policy-invariant.

Next, to make the revenue--conversion trade-off explicit, we use a scalarized reward:
\begin{defn}[Reward Function]
For \(\lambda\ge 0\), define
\begin{equation}
\label{eq:reward_function}
r(y_s,\lambda)\;=\;\lambda\,\mathrm{Revenue}_s \;+\; \mathrm{Conversion}_s.
\end{equation}
\end{defn}
The scalarization parameter \(\lambda\) governs the relative weight on revenue versus conversions. Scalarization is standard in multi-objective optimization \citep{deb2016multi}, and varying \(\lambda\) traces out a Pareto frontier. Recent marketing studies also use this approach \citep{rafieian2025multiobjective, wang2024recommending, cheng2025balancing}. In practice, firms calibrate \(\lambda\) using the estimated monetary value of an incremental conversion, such as expected downstream retention, future purchases, or user lifetime value. This maps both terms in \(r(y_s,\lambda)\) to a common value scale, so policy comparisons reflect a transparent economic trade-off. Firms may also tune \(\lambda\) to satisfy business constraints, for example, meeting a minimum revenue target while limiting conversion losses. We discuss the choice of $\lambda$ in our setting in $\S$\ref{ssec:implementation_details}.

\subsection{Problem Definition}
Given \(\lambda\), the firm’s objective is to maximize cumulative expected reward over a fixed time horizon $T$:
\begin{equation}
\label{eq:policy_objective}
\pi^\star \in \arg\max_{\pi}\; \mathbb{E}\!\left[\sum_{s: t_s \leq T} r\!\big(y_s^\pi,\lambda\big)\right],
\end{equation}
where the expectation is taken over randomness in actions and outcomes (and, potentially, in the arrival process of contexts induced by the policy).

The objective in Equation \eqref{eq:policy_objective} is to choose, for each incoming search \(s\), an ad load that balances revenue and conversions according to the scalarization parameter \(\lambda\). For a given query \(q\) and candidate ad load \(i\), the relevant objects are the conditional mean outcomes
\[
\mathbb{E}[\mathrm{Revenue}_s\mid q_s=q,a_s=i]
\quad\text{and}\quad
\mathbb{E}[\mathrm{Conversion}_s\mid q_s=q,a_s=i].
\]
If these quantities were known and stable, the platform could select the ad load that maximizes expected reward. In practice, however, both these components are unknown; as such the platform must trade off learning these components through sufficient {\it exploration} with {\it exploiting} the information it has already gained.

%\paragraph{\rev{Connection to contextual bandits:}} 

Note that each search arrives with a context (in our case $x_s$, as defined earlier). Thus, a natural starting point for the adaptive ad-load design problem is contextual bandits, e.g., LinUCB \citep{li2010contextual, chu2011contextual, abbasi2011improved} or parametric Thompson sampling \citep{agrawal2013thompson}. These methods are theoretically appealing when rewards are generated by a stable/stationary parametric model over well-defined contexts and context arrivals are i.i.d. Unfortunately, our environment departs from these assumptions in four important ways. First, ad-load decisions can affect subsequent user behavior, so context arrivals in our setting are not strictly policy-invariant, i.e., the arrival of contexts can be non-i.i.d. and policy-dependent. This is often referred to as adversarial contexts in bandit settings \citep{lattimore2020bandit}. Second, the platform may not be able to pre-define and track all the relevant contextual variables that can affect the reward function, e.g., advertiser entry/exit, set of apps available, budget depletion, bid changes, user-demand, and seasonality. 
Third, even if the firm could identify all the relevant contextual variables, it is unclear what parametric form the reward function would take. Fourth, reward mappings are non-stationary over time; as shown in \S\ref{sssec:het_time}, within-query shifts are empirically important.

%Recall that $x_s$ is simply defined as the set of query ($q_s$) and time-stamp ($t_s$) associated with the search $s$. However, the time-stamp is unique to each search and cannot be modeled as a contextual variable (since the platform cannot learn a reward model as a function of the time-stamp)\mr{This is not true in general: Timestamp is too granular but for example hour can be used as a context}. Thus, the only contextual variable available to the platform is the query $q_s$. However, as
%Fourth, even if we could specify or learn an approximate reward model over a set of pre-defined contextual variables, these reward mappings are unlikely to be valid sine the true reward functions are non-stationary: advertiser entry/exit, budget depletion, bid changes, and relevance shifts move the revenue--conversion frontier over time. Recall that, our heterogeneity results in \S\ref{sssec:het_time} show that within-query shifts are empirically important. 

Recent work relaxes some of these assumptions but not all. For example, \citet{CheungSimchiLeviZhu2019NonStationarity, russac2019weighted} model non-stationarity via discounting/sliding-window, but still assume that the platform knows all the relevant contextual variables that affect the reward model and a linear (parametric) reward model.\footnote{They also introduce additional tuning and exploration-safety trade-offs that can be difficult to deploy in production.}  \citet{jain2024effective} allow for adversarial context arrivals and non-parametric reward models in contextual bandits setting. However, their model still requires the platform to know and define all the relevant contextual variables that affect the reward model, and assumes that rewards are stationary and stable. 

Taken together, these gaps imply three design requirements for our setting: (i) model-free reward estimation under partial observability of relevant contextual features, (ii) explicit adaptation to drifting rewards, and (iii) persistent exploration that remains operationally stable. We address these requirements with our proposed adaptive ad load algorithm introduced next.

%\rev{Notably, the platform's existing prediction, auction, and ranking systems are already context-dependent and periodically updated. Our goal is therefore not to replace this infrastructure with a new learned reward model, but to add a lightweight ad-load decision and exploration layer on top of it. This production-friendly design philosophy is similar in spirit to \citet{YeEtAl2023ColdStart}, who add a simple uniform exploration layer on top of an existing advertising system rather than rebuilding the underlying machine-learning infrastructure.} 

\section{Adaptive Ad Load Design}
\label{sec:laal_policy_design}

We propose and implement \emph{exploration-augmented Locally Adaptive Ad Load} (\textbf{e-LAAL}). The notation is intentional: \textbf{LAAL} denotes the adaptive query-level decision rule, while the prefix \(\mathbf{e}\) denotes the exploration layer in the deployed system. Throughout, ``locally'' refers to estimation that is local in \emph{query} and in \emph{time}: LAAL estimates rewards separately for each query using only recent observations from a sliding window---i.e., it is a query-level, sliding-window greedy rule. 

%We use the term to denote this query-recency structure rather than any notion of metric-space locality.

\subsection{e-LAAL Architecture and Traffic Allocation}
\label{ssec:architecture}

e-LAAL combines two components: (i) an adaptive arm that runs \textsc{LAAL}, and (ii) parallel static arms \(\{C,T_2,\dots,T_6\}\) inherited from the first experiment. Let \(w_e\) denote arm-level traffic shares of the static arms and define
\begin{equation}
\epsilon \equiv \sum_{e\in\{C,T_2,\dots,T_6\}} w_e.
\label{eq:exploration_budget}
\end{equation}
Thus, \(\epsilon\) is the fixed exploration budget, while \(1-\epsilon\) is served by the adaptive LAAL arm. In our deployment, \(\epsilon=0.10\), so 10\% of users are permanently allocated to static exploration and 90\% to LAAL (Table~\ref{tab:exp-design}).

This architecture separates exploration from exploitation in a controlled way. Static arms ensure persistent support for every ad-load level and LAAL exploits these continually refreshed observations to adapt ad load across queries and over time. The static cohorts are also used to estimate contemporaneous fixed ad load (always show \(i\) ads) counterfactuals; see \S\ref{ssec:benchmark_static} for details.  

\begin{table}[htp!]
\centering
\caption{Traffic Allocation under e-LAAL}
\small
\label{tab:exp-design}
\begin{tabular}{l r}
\toprule
Condition & Share of Users \((w_e)\) \\
\midrule
\(C\) (1 ad)              & 5\% \\
\(T_{2}\) (2 ads)         & 1\% \\
\(T_{3}\) (3 ads)         & 1\% \\
\(T_{4}\) (4 ads)         & 1\% \\
\(T_{5}\) (5 ads)         & 1\% \\
\(T_{6}\) (6 ads)         & 1\% \\
LAAL (adaptive)            & 90\% \\
\midrule
\(\text{e-LAAL}\) (all conditions combined) & 100\% \\
\bottomrule
\end{tabular}
\end{table}

\subsection{\textsc{LAAL} Decision Rule and Algorithm}
\label{ssec:decision_rule}

We now describe the adaptive \textsc{LAAL} policy, which runs on 90\% of the traffic. \textsc{LAAL} is a deployable policy for pursuing the cumulative-reward objective in Equation~\eqref{eq:policy_objective}: it uses recent query-level outcomes to track the contemporaneous reward-maximizing ad load as the reward frontier changes over time. It is adaptive along two dimensions. First, it is \emph{query-adaptive}: different queries can receive different ad loads because \(\{r_i(q,t)\}_{i=1}^{6}\) differs across queries. Second, it is \emph{time-adaptive}: for a fixed query, the selected ad load can change as recent outcomes update, allowing the policy to respond to shifts in market conditions (e.g., brand entry/exit, bidder-quality changes). Operationally, for each incoming search, LAAL combines local recency-based estimation with greedy action selection:
\squishlist
\item \textbf{Sliding-window reward estimation.} For each query--ad-load pair \((q,i)\), LAAL estimates recent mean conversion and revenue per search using a fixed-length sliding window.\footnote{Using fixed-length windows (or related forgetting schemes) to handle non-stationarity is common in online learning and bandits \citep{TrovoEtAl2020SWTS,CavenaghiEtAl2021NonStationaryMAB,CheungSimchiLeviZhu2019NonStationarity}.}
\item \textbf{Greedy ad-load selection.} LAAL maps these estimates into a scalarized reward and selects the ad load with the highest estimated reward.
\squishend

\begin{figure}[htbp]
\centering
\begin{tikzpicture}[x=1.15cm,y=1cm,>=latex,font=\footnotesize]
  % Parameters
  \def\h{3}      % window length (days)
  \def\D{2}      % gap between searches (days)
  \pgfmathsetmacro{\tm}{\D-\h}

  % Axis
  \draw[->,thick] (-4.1,0) -- (3.3,0) node[right]{calendar time};

  % Key ticks (uses \D, no hard-coded "2")
  \foreach \x/\lab in {-\h/{\(t_1-h\)},0/{\(t_1\)},\tm/{\(t_2-h\)},\D/{\(t_2=t_1+\Delta\)}}{
    \draw (\x,0.06) -- (\x,-0.06);
    \node[below] at (\x,-0.08) {\lab};
  }

  % Search times
  \draw[dashed] (0,-1.05) -- (0,1.2);
  \draw[dashed] (\D,-1.05) -- (\D,1.2);
  \node[above] at (0,1.2) {search \(s_1\)};
  \node[above] at (\D,1.2) {search \(s_2\)};

  % Gap annotation
  \draw[<->] (0,0.95) -- (\D,0.95);
  \node[above] at (\D/2,0.95) {\(\Delta\)};

  % Sliding windows with interval notation
  \draw[rounded corners,very thick,blue] (-\h,0.35) rectangle (0,0.75);
  \node[blue] at (-1.5,0.95) {\(W_{t_1}(q)=[t_1-h,t_1)\)};

  \draw[rounded corners,very thick,orange!85!black] (\tm,-1.05) rectangle (\D,-0.65);
  \node[orange!85!black] at (0.5,-1.25) {\(W_{t_2}(q)=[t_2-h,t_2)\)};

  % Overlap retained
  \fill[green!25] (\tm,-0.03) rectangle (0,0.03);
\end{tikzpicture}
\caption{Sliding-window update for two searches of the same query \(q\), separated by \(\Delta\). LAAL uses \(W_t(q)=[t-h,t)\); as the window shifts from \(t_1\) to \(t_2\), old data are dropped and new data are added.}
\label{fig:sliding_window_two_searches}
\end{figure}
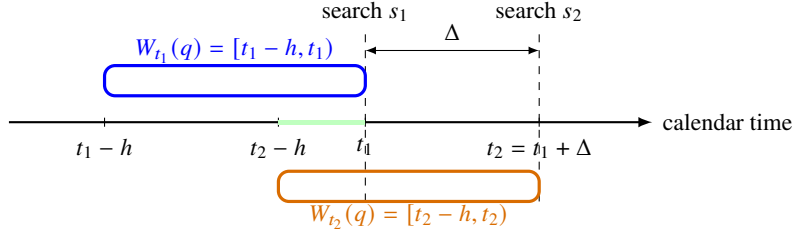

In deployment, LAAL maximizes the per-search reward in Equation~\eqref{eq:reward_function}. For each query \(q\), it forms sliding-window estimates for each ad load \(i\in\{1,\dots,6\}\). Let \(h\) denote window length, and define the {\it sliding window} $W_t(q)$ as:
\begin{equation}
W_t(q)=\{s' : q_{s'}=q,\ t-h \le t_{s'} < t\}.
\end{equation}
Figure~\ref{fig:sliding_window_two_searches} depicts how the sliding window $W_t(q)$ is used for two searches of the same query $q$, $s_1$ and $s_2$, separated by a time difference of $\Delta$. Notice that for the same query in the newer search ($s_2$), the older data is dropped and only data from the most recent time-window of length $h$ is considered.

Next, for each ad-load \(i\), let \(n_i(q,t)\) denote the {\it number of observations in the recent window} for query \(q\) when ad load \(i\) was shown:
\begin{equation}
n_i(q,t)=\sum_{s'\in W_t(q)}\mathbf{1}\{a_{s'}=i\}.
\label{eq:local_obs}
\end{equation}
Using these observations, LAAL computes two local means (the two \(\mu\)-terms):  
\(\widehat{\mu}_{C,i}(q,t)\), the \emph{local mean conversion per search}, and  
\(\widehat{\mu}_{R,i}(q,t)\), the \emph{local mean revenue per search}, both for query \(q\) under ad load \(i\):
\begin{equation}
\widehat{\mu}_{C,i}(q,t)=
\frac{\sum_{s'\in W_t(q)} \mathbf{1}\{a_{s'}=i\}\,\mathrm{Conversion}_{s'}}{n_i(q,t)+\delta},
\quad
\widehat{\mu}_{R,i}(q,t)=
\frac{\sum_{s'\in W_t(q)} \mathbf{1}\{a_{s'}=i\}\,\mathrm{Revenue}_{s'}}{n_i(q,t)+\delta},
\label{eq:local_con_rev}
\end{equation}
where \(\delta>0\) is a small stability constant. LAAL then computes  \(\widehat{r}_i(q,t)\), the \emph{estimated local reward} for choosing ad load \(i\) on query \(q\) at time \(t\):
\begin{equation}
\widehat{r}_i(q,t)=\lambda\,\widehat{\mu}_{R,i}(q,t)+\widehat{\mu}_{C,i}(q,t)-\eta i,
\label{eq:reward_est_LAAL}
\end{equation}
selects ad load:
\[
a_t(q)\in\arg\max_{i\in\{1,\dots,6\}}\widehat{r}_i(q,t).
\]
We set $\eta$ to a very small number; thus \(-\eta \times i\) is a small tie-break regularizer that breaks near ties in favor of fewer ads. To stabilize sparse queries, LAAL applies greedy selection only when the query has at least \(m\) observations in the window; otherwise, it returns the default ad load of \(i_{\mathrm{def}}\). We discuss the choice of the hyperparameters ($h, \delta, \eta, m, i_{\mathrm{def}}$) in $\S$\ref{ssec:implementation_details}. 

Algorithm~\ref{alg:slot_selection} presents the details of the deployed procedure.

%\textbf{Conditional Cohort-Invariance Assumption} LAAL pools recent observations from both the LAAL and static cohorts to estimate query--ad-load rewards (i.e., to estimate Equations \eqref{eq:local_obs} and \eqref{eq:local_con_rev}). Pooling data from the static and LAAL arms improves support for each query--ad-load pair and, as such, improves precision. However, it requires immediate per-search outcomes to be conditionally comparable across cohorts. Specifically, after conditioning on query, time, and displayed ad load, revenue and conversion responses are assumed to be independent of whether the observation came from a static or adaptive cohort. This assumption is weaker than requiring search behavior to be policy-invariant, i.e., it does not require identical future search behavior across cohorts; the static cohorts are maintained precisely because ad load may affect subsequent usage. But it rules out unobserved regime-specific response effects.  We formally test and establish the validity of this assumption in Web Appendix $\S$\ref{appsec:CCIA}.

\textbf{Conditional Cohort-Invariance Assumption.} LAAL pools recent observations from both the LAAL and static cohorts to estimate query--ad-load rewards (i.e., to estimate Equations \eqref{eq:local_obs} and \eqref{eq:local_con_rev}). Pooling improves support for each query--ad-load pair and, as such, improves precision. However, it requires immediate per-search outcomes to be conditionally comparable across cohorts: after conditioning on query, time, and displayed ad load, revenue and conversion responses are assumed to be independent of whether the observation came from a static or adaptive cohort. This assumption is weaker than requiring search behavior to be policy-invariant; it does not require identical future search behavior across cohorts, and the static cohorts are maintained precisely because ad load may affect subsequent usage. Instead, it rules out unobserved regime-specific differences in immediate search-level responses after conditioning on query, time, and realized ad load. We formally test and establish the validity of this assumption in Web Appendix $\S$\ref{appsec:CCIA}.

Further, note that while this assumption is relevant for the quality of LAAL's learned reward estimates, it does not affect the direct evaluation of realized deployment outcomes in $\S$\ref{sec:emp_performance}. If the assumption fails, the pooled estimates used by LAAL may be noisier or biased for the LAAL cohort, potentially reducing the quality of the policy's recommendations. However, outcomes for the LAAL arm and the static arms are observed directly in their respective randomized cohorts during deployment. Thus, the deployment results remain informative about how the implemented policy performed in production.

\paragraph{Illustrative example.}
We now present a simple example using a single query and two adjacent windows separated by one day to illustrate how LAAL works. Suppose LAAL uses only the last \(h=3\) days of data,  $\lambda = 0.01$, and the tie-breaking regularizer $\eta = 0.001$. Then, for query $q$ at time $t$, the window is \(W_t(q)\), and LAAL computes the estimated reward over this window as:
\[
\widehat r_i(q,t)=0.01\,\widehat\mu_{R,i}(q,t)+\widehat\mu_{C,i}(q,t)-0.001\,i.
\]
For illustration, we only consider a setting with three ad loads \(i\in\{2,4,6\}\) and two consecutive windows, as shown in Table \ref{tab:laal_two_windows}. LAAL chooses \(a_t(q)=6\) in Table~\ref{tab:laal_window_a}, when the brand is active, and the expected revenue from a higher ad load outweighs the lower conversion rate. In contrast,  \(a_t(q)=2\) in Table~\ref{tab:laal_window_b}. Here, the brand has become inactive, and the expected revenue from a higher ad load does not compensate for the lower conversion. See Figure~\ref{fig:laal_sliding_window_example} for a pictorial depiction.

% In preamble:
% \usepackage{subcaption}

\begin{table}[htp!]
\centering
\caption{LAAL estimates in two consecutive windows.}
\label{tab:laal_two_windows}

\begin{subtable}[t]{0.49\linewidth}
\centering
\caption{Window A (brand-active period)}
\label{tab:laal_window_a}
\[
\begin{array}{c|c|c|c}
i & \widehat\mu_{C,i} & \widehat\mu_{R,i} & \widehat r_i \\
\hline
2 & 0.50 & 36 & 0.50 + 0.36 - 0.002 = 0.858 \\
4 & 0.38 & 60 & 0.38 + 0.60 - 0.004 = 0.976 \\
6 & 0.27 & 84 & 0.27 + 0.84 - 0.006 = 1.104
\end{array}
\]
\end{subtable}
\hfill
\begin{subtable}[t]{0.49\linewidth}
\centering
\caption{Window B (one day later; window shifted, brand exits)}
\label{tab:laal_window_b}
\[
\begin{array}{c|c|c|c}
i & \widehat\mu_{C,i} & \widehat\mu_{R,i} & \widehat r_i \\
\hline
2 & 0.48 & 32 & 0.48 + 0.32 - 0.002 = 0.798 \\
4 & 0.33 & 40 & 0.33 + 0.40 - 0.004 = 0.726 \\
6 & 0.18 & 36 & 0.18 + 0.36 - 0.006 = 0.534
\end{array}
\]
\end{subtable}
\end{table}

This illustrates the sliding-window mechanism: as old observations leave and new ones enter the 3-day window, the estimated reward ranking changes, and LAAL adapts the ad load for the same query over time. In sum, LAAL maximizes estimated contemporaneous reward from local evidence, while parallel static cohorts provide continual exploration by maintaining exposure to every ad-load level.

\begin{algorithm}[h]
\caption{e-LAAL Algorithm: Locally Adaptive Ad Load (LAAL) with Parallel Static Arms}
\label{alg:slot_selection}
\begin{algorithmic}[1]
\Require each search \(s\) arrives with context \((q_s,t_s)\), and arm assignment $e_s$ inherited from user assignment; \(e_s \in  \{C,T_2,\dots,T_6, L\}\), where \(L\) denotes the LAAL arm 
\Require window length \(h\) (days), scalarization \(\lambda \geq 0\), smoothing \(\delta>0\), minimum window volume \(m\), regularization \(\eta\ge 0\), default ad load \(i_{\mathrm{def}}\)

\Require action set \(\mathcal{A}=\{1,2,3,4,5,6\}\)
\For{each search request \(s\)}
  \If{\(e_s \in \{C,T_2,\dots,T_6\}\)} \Comment{Static cohorts: fixed ad load}
    \State \(a_s \gets 1\cdot\mathbf{1}\{e_s=C\} + \sum_{j=2}^6 j\,\mathbf{1}\{e_s=T_j\}\)
  \Else \Comment{\(e_s=L\): LAAL cohort}
    \State \(W \gets \{\,s': q_{s'}=q_s,\ t_s-h \le t_{s'} < t_s\,\}\) \Comment{Local window for query \(q_s\), using all cohorts}
    \For{\(i\in\mathcal{A}\)}
      \State \(n_i \gets \sum_{s'\in W}\mathbf{1}\{a_{s'}=i\}\)
      \State \(C_i \gets \sum_{s'\in W}\mathbf{1}\{a_{s'}=i\}\,\mathrm{Conversion}_{s'}\)
      \State \(R_i \gets \sum_{s'\in W}\mathbf{1}\{a_{s'}=i\}\,\mathrm{Revenue}_{s'}\)
      \State \(\widehat{\mu}_{C,i} \gets \dfrac{C_i}{n_i+\delta}\)
      \State \(\widehat{\mu}_{R,i} \gets \dfrac{R_i}{n_i+\delta}\)
      \State \(\widehat{r}_i \gets \lambda\,\widehat{\mu}_{R,i} + \widehat{\mu}_{C,i} - \eta i\)
    \EndFor
    \If{\(\sum_{i\in\mathcal{A}} n_i < m\)}
      \State \(a_s \gets i_{\mathrm{def}}\) \Comment{Fallback for sparse windows}
    \Else
      \State \(a_s \gets \arg\max_{i\in\mathcal{A}} \widehat{r}_i\)
    \EndIf
  \EndIf
  \State Deploy \(a_s\) sponsored slots; observe \((\mathrm{Conversion}_s,\mathrm{Revenue}_s)\)
\EndFor
\end{algorithmic}
\end{algorithm}

% Old Version of Box-Windows
% \begin{figure}[htp!]
% \centering
% \begin{tikzpicture}[font=\footnotesize]
% % Timeline
% \draw[thick] (-0.4,0) -- (8.2,0);
% \foreach \x/\lab in {0/{t-4},1.3/{t-3},2.6/{t-2},3.9/{t-1},5.2/{t},6.5/{t+1},7.8/{t+2}}{
%   \fill (\x,0) circle (1.6pt);
%   \node[below] at (\x,-0.02) {\lab};
% }

% % Window A and B
% \draw[rounded corners, very thick, blue] (1.0,-0.32) rectangle (5.5,0.32);
% \node[blue] at (3.25,0.7) {Window A};

% \draw[rounded corners, very thick, orange!85!black] (2.3,-0.52) rectangle (6.8,0.52);
% \node[orange!85!black] at (4.55,-0.9) {Window B (shifted)};

% % Score boxes
% \node[draw,rounded corners,fill=blue!6,align=left,anchor=west] (A) at (8.8,0.8) {
% \textbf{Window A scores}\\
% $r_2=0.858,\ r_4=0.976,\ r_6=1.104$\\
% $\Rightarrow\ a_t(q)=6$
% };

% \node[draw,rounded corners,fill=orange!8,align=left,anchor=west] (B) at (8.8,-1.0) {
% \textbf{Window B scores}\\
% $r_2=0.798,\ r_4=0.726,\ r_6=0.534$\\
% $\Rightarrow\ a_t(q)=2$
% };

% \draw[->,thick] (A.south west) .. controls (7.9,0.2) and (7.9,-0.4) .. (B.north west);
% \node at (7.8,-0.15) {\scriptsize window slides};
% \end{tikzpicture}
% \caption{Illustration of LAAL's sliding-window adaptation. As the 3-day window shifts, estimated rewards change and the selected ad load for the same query can switch.}
% \label{fig:laal_sliding_window_example}
% \end{figure}

\begin{figure}[htp!]
\centering
\begin{tikzpicture}[font=\footnotesize]
% Timeline
\draw[thick] (-0.4,0) -- (8.2,0);
\foreach \x/\lab in {0/{t-4},1.3/{t-3},2.6/{t-2},3.9/{t-1},5.2/{t},6.5/{t+1},7.8/{t+2}}{
  \fill (\x,0) circle (1.6pt);
  \node[below] at (\x,-0.02) {\lab};
}

% Windows (axis y=0 bisects both boxes)
\draw[rounded corners, very thick, blue] (0.85,-0.42) rectangle (5.65,0.42);
\node[blue] at (3.25,0.86) {Window A};

\draw[rounded corners, very thick, orange!85!black] (2.15,-0.62) rectangle (6.95,0.62);
\node[orange!85!black] at (4.55,-0.94) {Window B (shifted)};

% Score boxes (shifted slightly left to stay comfortably within the page frame)
\node[draw,rounded corners,fill=blue!6,align=left,anchor=west] (A) at (8.45,0.85) {
\textbf{Window A scores}\\
$\widehat r_2=0.858,\ \widehat r_4=0.976,\ \widehat r_6=1.104$\\
$\Rightarrow\ a_t(q)=6$
};

\node[draw,rounded corners,fill=orange!8,align=left,anchor=west] (B) at (8.45,-0.95) {
\textbf{Window B scores}\\
$\widehat r_2=0.798,\ \widehat r_4=0.726,\ \widehat r_6=0.534$\\
$\Rightarrow\ a_t(q)=2$
};

% arrow between boxes
\draw[->,thick] (A.south) -- (B.north)
  node[midway,right=2pt] {\scriptsize window slides};

\end{tikzpicture}
\caption{Illustration of LAAL's sliding-window adaptation. As the 3-day window shifts, estimated rewards change, and the selected ad load for the same query can switch.}
\label{fig:laal_sliding_window_example}
\end{figure}
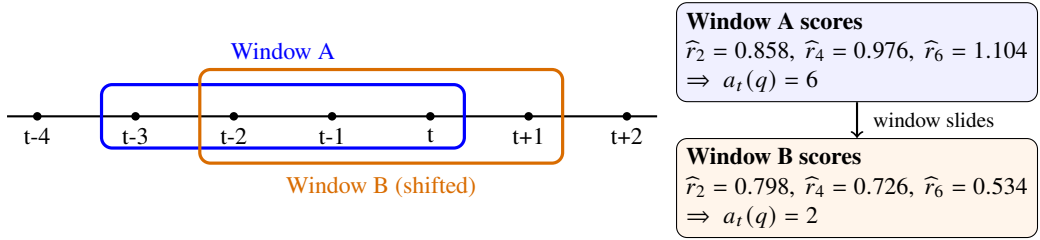

\subsection{Why e-LAAL Fits This Setting}

e-LAAL addresses the key deployment challenge in our setting: how to combine adaptive exploitation with exploration that remains operationally stable and provides meaningful fixed-policy counterfactuals. Unlike standard per-search exploration schemes in contextual bandits, e-LAAL separates exploration and exploitation through persistent user-level cohorts. This distinction is important because ad load may affect not only immediate revenue and conversion, but also subsequent search behavior and engagement. Evaluating fixed ad-load policies therefore requires users to remain under persistent regimes such as ``always show \(i\) ads,'' which the static cohorts preserve.

This architecture delivers four advantages for our deployment environment.

\squishlist

\item \textbf{User-level fixed-policy counterfactuals.}
Because ad load may affect future usage, fixed ad-load benchmarks must be evaluated using persistent user-level regimes rather than isolated per-search randomization. The static arms \(\{C,T_2,\ldots,T_6\}\) maintain users under fixed ad-load policies, allowing us to estimate outcomes for policies of the form ``always show \(i\) ads,'' including their induced effects on search behavior and engagement. A standard \(\epsilon\)-greedy policy mixes ad loads within a single adaptive arm and therefore does not provide these fixed-policy counterfactuals.

\item \textbf{Stable exploration support under changing context distributions.}
Exploration is supplied ex ante through static arms with fixed traffic shares \(w_e\). Thus, every ad-load level continues to receive exposure even if LAAL concentrates adaptive traffic on a subset of actions. This is particularly important when ad-load decisions may influence future context arrivals: realized query-level support can differ from the platform's intended traffic allocation, but the static cohorts provide a persistent exploration floor.

\item \textbf{Model-free adaptation to a drifting reward frontier.}
LAAL estimates query-level rewards directly from recent outcomes rather than relying on a prespecified reward model. This reduces dependence on observing and encoding all drivers of rewards, such as advertiser entry and exit, budget depletion, bid changes, relevance shifts, seasonality, and demand changes. Through sliding-window updates, the policy adapts as these factors shift the revenue--conversion frontier over time.

\item \textbf{Production-friendly implementation.}
e-LAAL operates as a lightweight decision layer on top of the platform's existing prediction, auction, and ranking systems. It does not require rebuilding the underlying advertising infrastructure or deploying a complex end-to-end learning system. Instead, it uses recent query-level outcomes and a small set of persistent exploration cohorts to adapt ad load at scale.

\squishend

Together, these features make e-LAAL suitable for sponsored-search environments where reward mappings evolve over time and credible fixed-policy counterfactuals are essential for deployment evaluation.

\subsection{Theoretical Guarantee for e-LAAL}

In this section, we formalize the performance of e-LAAL in a non-stationary environment. The benchmark is a dynamic query-level oracle that can choose the best ad load for a fixed query as its reward frontier changes over time. The key complication, relative to standard non-stationary bandit analyses, is that exploration in e-LAAL is supplied through persistent user-level static cohorts rather than direct per-search randomization. As a result, the realized query-level support for each ad-load arm may differ from the platform's designed traffic shares. The theorem below reports the optimized-rate implication under regular query arrivals; Web Appendix $\S$\ref{appsec:elaal_regret_proof} provides the corresponding finite-sample bound for arbitrary exploration mass and window length. For simplicity, regret is stated in normalized reward units, with raw-unit regret obtained by rescaling by the reward bound.

%In this section, we state the main dynamic regret guarantee for e-LAAL. The proof accounts for the fact that exploration is assigned through static user cohorts, while the realized query-level traffic to different ad-load arms may differ from the assigned cohort shares.

\begin{thm}[Dynamic regret of e-LAAL]
\label{thm:elaal_regret_main}
Fix query \(q\). Assume the number of realized searches for query \(q\) during \([0, T]\) satisfies \(N=\Theta(T)\), and the query arrivals are regular enough that each \(h\)-time
window contains on the order of \(Nh/T\) searches. Under the assumptions stated in Web Appendix~\(\S\)\ref{appsec:elaal_regret_proof}, after normalizing rewards and optimizing over exploration mass and window length, the dynamic regret of Algorithm~\ref{alg:slot_selection} is bounded, with high probability, by
\[
R_q^{\mathrm{dyn}}(T)
=
\tilde O\!\left(
\bar c^{1/4}\underline c^{-1/4}
T^{3/4}
\big(K\,V_q(T)\big)^{1/4}
\right).
\]
Here \(V_q(T)\) is the variation budget, $K$ is the number of arms, \(\underline c\) is the minimum-support coefficient, and \(\bar c\) is the total-static-traffic upper-bound coefficient.

\end{thm}

\noindent\textit{Proof.} See Web Appendix~\(\S\)\ref{appsec:elaal_regret_proof}.

The query-level guarantee reflects the local nature of LAAL: each query evolves according to its own reward frontier and support process. Platform-level performance aggregates these query-level decisions weighted by search volume. The coefficients \(\underline c\) and \(\bar c\) account for the gap between the design exploration budget and the realized query-level exploration traffic. The coefficient \(\underline c\) lower bounds the observed traffic on each static arm by its design weight. The coefficient \(\bar c\) upper bounds the total realized traffic from static cohorts by their total exploration weight. Thus, these parameters capture deviations between assigned user-level cohort shares and realized query-level static traffic. If there is no such deviation and the realized exploration traffic exactly matches the design exploration budget, then \(\underline c=\bar c=1\).

The regret bound combines the standard non-stationary bandit tradeoff between estimation noise and temporal drift with the additional cost and support role of static exploration cohorts. The proof follows the standard logic of non-stationary bandit analysis, where sliding-window estimators balance estimation noise against temporal drift \citep{TrovoEtAl2020SWTS,CavenaghiEtAl2021NonStationaryMAB,CheungSimchiLeviZhu2019NonStationarity}. The new feature in our setting is the exploration architecture. In canonical bandit models, exploration probabilities are typically controlled directly at the decision level. In e-LAAL, exploration is supplied by persistent user-level static cohorts, but regret is evaluated on the realized search stream for a fixed query. Our analysis therefore accounts for the fact that realized query-level support can differ from the assigned cohort shares. The support coefficients in the theorem make this distinction explicit and adapt standard non-stationary regret arguments to the parallel-static-cohort implementation used in our deployment.

\subsection{Implementation Details}
\label{ssec:implementation_details}
We now discuss the three main parameters that are used in the deployment below.

\squishlist
\item \textbf{Scalarization parameter in the objective function (\(\lambda\)): }
In principle, \(\lambda\) should reflect the platform’s monetary value of an incremental conversion (e.g., downstream engagement, retention, or LTV). Because these downstream outcomes are not observed in our data, we use \(\lambda\) as a scale-normalization parameter to convert short-run ad revenue into conversion-equivalent units, so that the two components in Equation~\eqref{eq:reward_function} are comparable.  
We calibrate \(\lambda\) from the first experiment: the uniform policies \(\{C,T_2,\dots,T_6\}\) trace an empirical revenue--conversion frontier; over this feasible menu, the relationship is approximately linear, implying a near-constant local exchange rate between conversion losses and revenue gains. We therefore set
\[
\lambda \approx -\frac{\Delta \mathrm{Conversion}}{\Delta \mathrm{Revenue}},
\]
which yields \(\lambda \approx 0.01\) in our data (Table~\ref{tab:user_ols_condition}). This calibration is not a claim about long-run welfare weights; rather it anchors LAAL to the platform’s revealed short-run rate of substitution across feasible uniform ad-load policies.

\item \textbf{Exploration budget (\(\epsilon\)):}
As discussed earlier, exploration is fixed ex ante through parallel static arms as defined in Equation \eqref{eq:exploration_budget}.
In our deployment, \(\epsilon=0.10\): 10\% of traffic is permanently allocated to static arms (5\% to \(C\), 1\% to each \(T_2,\dots,T_6\)), while 90\% is assigned to the adaptive LAAL arm.

\item \textbf{Algorithmic hyperparameters (\(h,m,\eta,\delta, i_{\mathrm{def}}\)):}
We set \(h=3\) days, \(m=100\), \(\eta=10^{-3}\), \(\delta=10^{-6}\), and fallback load \(i_{\mathrm{def}}=4\). The reasons for these choices are discussed briefly below:
\begin{itemize}
\item \(h=3\) days balances responsiveness and variance. Shorter windows react faster to market-state changes (e.g., advertiser entry/exit) but are noisier; longer windows are more stable but slower to adapt. A 3-day horizon provided sufficient sample support while preserving timely adaptation.
\item \(m=100\) enforces a minimum evidence threshold before greedy selection. This guards against unstable action flips on sparse queries where window estimates are high-variance.
\item \(\eta=10^{-3}\) is a small tie-break regularizer that favors lower ad loads only when estimated rewards are nearly equal.
\item \(\delta=10^{-6}\) is a numerical smoothing constant to avoid division instability in very low-count cells; it is intentionally negligible relative to observed counts and does not affect estimates when support is adequate.
\item $i_{\mathrm{def}}=4$ is the default ad load used when the query is sparse in the recent window. Using the historical data on static-cohorts obtained from the last month of the first experiment, we set $i_{\mathrm{def}}$ to the uniform ad load that maximizes expected reward on a low-volume search query. The definition of low volume query follows the same definition as in deployment of LAAL. See Web Appendix $\S$\ref{apdx:low_volume_queries} for the detailed analysis.
\end{itemize}
%\hy{I am a bit confused; there was no mention of 1000 searches per month-- I thought it was 100 in the sliding window. Also, is 4 chosen based on historical data? Can we forward reference the appropriate section} \mr{Yes it is based on history. Queries with 100 searches on three days are approximately the same set as queries with 1000 searches per month and both give the same result on i\_def. I put the consistent result in appendix} 

%Overall, the choice of hyper-parameters balance robustness vs. accuracy in a real deployment setting.\footnote{We note that these hyperparameters are tunable and practitioners may choose different parameters in other settings.}
\squishend

\section{Empirical Performance of the Adaptive Ad Load Policy}
\label{sec:emp_performance}

The e-LAAL architecture was deployed at the platform level, and we now report the results from the deployment for a 22-day steady-state window from December 8 to 29, 2023.\footnote{The adaptive LAAL arm was launched soon after the first experiment ended (November 17, 2023), but the system required a short ramp-up period before reaching steady state.} During this period, user assignment is fixed: \(90\%\) of traffic is routed to the adaptive LAAL arm and \(10\%\) remains in the static exploration arms \(\{C,T_2,\dots,T_6\}\) (as shown in Table~\ref{tab:exp-design}).

In the rest of this section, we address three empirical questions on the performance of our adaptive ad load policy in deployment. (1) Does e-LAAL improve the revenue--conversion trade-off relative to static ad-load policies?  (2) Does LAAL adapt ad load across query segments in economically sensible ways?  
(3) Does LAAL adapt over time when market conditions shift (brand entry/exit)?

\subsection{Main Results}
\label{ssec:main_result}

\paragraph{Summary statistics from platform-level deployment.}

\begin{table}[htp!]
\centering
\small
\setlength{\tabcolsep}{3.5pt} % Tightens column spacing
\caption{Summary Statistics of the Deployment}
\label{tab:summary_user_experiment_deploy}
\begin{tabular}{lcccccccc}
\toprule
 & \textbf{$C$} & \textbf{$T_{2}$} & \textbf{$T_{3}$} & \textbf{$T_{4}$} & \textbf{$T_{5}$} & \textbf{$T_{6}$} & \textbf{LAAL} & \textbf{e-LAAL} \\
\midrule
User Count & 1,121,219 & 223,690 & 222,024 & 222,586 & 221,871 & 221,961 & 20,097,353 & 22,330,704 \\
Searches   & 3,859,143 & 776,181 & 776,749 & 779,883 & 776,194 & 781,223 & 69,851,684 & 77,601,057 \\
\midrule
Revenue/User\textsuperscript{a} & 25.80 & 29.48 & 31.17 & 33.11 & 36.11 & 36.68 & 36.13 & 35.47 \\
Conversion/User                  & 2.152 & 2.104 & 2.101 & 2.091 & 2.064 & 2.054 & 2.127 & 2.126 \\
\midrule
Daily Eng.\ (\%) & 8.45 & 8.44 & 8.38 & 8.34 & 8.37 & 8.35 & 8.43 & 8.43 \\
Searches/User & 3.44 & 3.47 & 3.50 & 3.50 & 3.50 & 3.52 & 3.48 & 3.47 \\
\bottomrule
\end{tabular}
\par\medskip
\footnotesize{\textit{Note:} \textsuperscript{a}Scaled by an undisclosed multiplier.}
\end{table}

Table~\ref{tab:summary_user_experiment_deploy} provides four immediate facts about the deployment. First, the scale is large: e-LAAL serves \(22.33\) million users and \(77.60\) million searches in 22 days, with the adaptive LAAL arm alone covering \(20.10\) million users and \(69.85\) million searches (\(\approx 90\%\) of traffic). Second, the static arms remain large in absolute size (\(\approx 0.22\)M--\(1.12\)M users each), so exploration and benchmarking are statistically meaningful throughout deployment.

Third, the table reproduces the core economics from the first experiment (Table~\ref{tab:summary_user_experiment}): as static ad load increases from \(C\) to \(T_6\), revenue per user rises monotonically (25.80 \(\rightarrow\) 36.68), while conversion per user falls (2.152 \(\rightarrow\) 2.054), confirming a stable revenue--conversion trade-off. Fourth, LAAL improves on this frontier: revenue per user is 36.13 (essentially \(T_5\)-level, 36.11), but conversion per user is 2.127 (higher than \(T_5\): 2.064, and \(T_6\): 2.054, and close to \(C\): 2.152). Note that the actual revenue and conversion numbers differ from those in Table~\ref{tab:summary_user_experiment} because the first experiment lasted 66 days, while the deployment statistics come from a 22-day period; as such, conversions and revenue are proportionally smaller here.\footnote{$\frac{Revenue/User}{Search/User}$ is a more comparable metric that controls for the length of the experiment and this metric is similar across both experiments.}

The system-level e-LAAL outcome remains close to LAAL (35.47 revenue/user; 2.126 conversion/user), indicating that the fixed 10\% exploration layer preserves learning and benchmarking with limited performance loss. Finally, daily engagement does not show notable differences across cohorts. LAAL has slightly lower engagement than $C$ and $T_2$, but these differences are not statistically significant; Web Appendix $\S$\ref{apdx:deployment_engagement} reports the corresponding tests.  Searches per user are also tightly clustered across cohorts (3.44--3.52), suggesting that outcome differences are not driven by large differences in search intensity across arms. Overall, Table~\ref{tab:summary_user_experiment_deploy} shows that at production scale, e-LAAL preserves the monetization gains of high ad-load policies while recovering a substantial share of the conversion losses they typically induce.

\paragraph{Evaluation objects and normalization.}
The summary statistics above are descriptive. To make policy comparisons valid under unequal cohort traffic shares, we use inverse-propensity normalization \citep{horvitz1952generalization, yoganarasimhan2023design}. Let \(w_e\) be the assignment probability to cohort \(e\). For \(Y\in\{\mathrm{Revenue},\mathrm{Conversion}\}\), define
\begin{equation}
\label{eq:cohort_ipw_total}
\widehat{Y}_e \;=\; \frac{1}{w_e}\sum_{s:\,e_s=e} Y_s.
\end{equation}
For static cohorts \(e\in\{C,T_2,\ldots,T_6\}\), \(\widehat{Y}_e\) is an unbiased estimate of the total outcome if the full population were assigned to the corresponding fixed policy (``always show \(i\) ads'').

Because LAAL is trained using data from all cohorts, the policy-relevant system estimand is the deployed mixture:
\begin{equation}
\label{eq:elaal_mixture}
\widehat{Y}_{\text{e-LAAL}}
\;=\;
\sum_{e\in\{C,T_2,\ldots,T_6,L\}} w_e\,\widehat{Y}_e
\;=\;
\sum_{s} Y_s.
\end{equation}
Thus, e-LAAL is the correct ``policy + exploration'' benchmark for deployment comparisons.

\begin{figure}[htp!]
  \centering
  \begin{subfigure}[t]{0.45\textwidth}
    \centering
    \includegraphics[width=\linewidth]{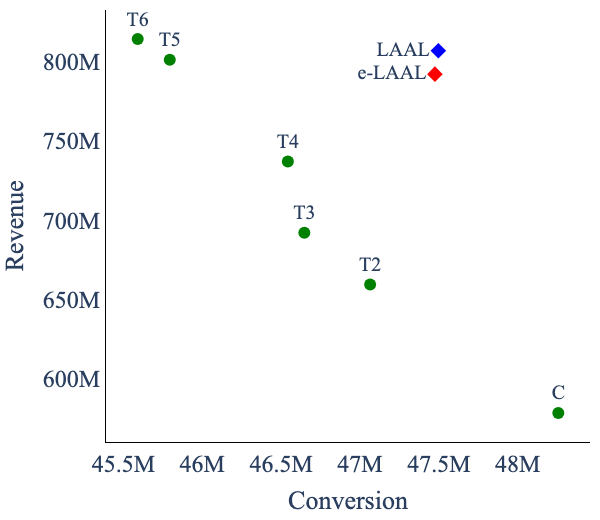}
    \caption{Revenue--conversion across fixed ad-load cohorts, LAAL, and e-LAAL.}
    \label{fig:revenue_conversion_tradeoff_greedy}
  \end{subfigure}
  \hfill
  \begin{subfigure}[t]{0.45\textwidth}
    \centering
    \includegraphics[width=\linewidth]{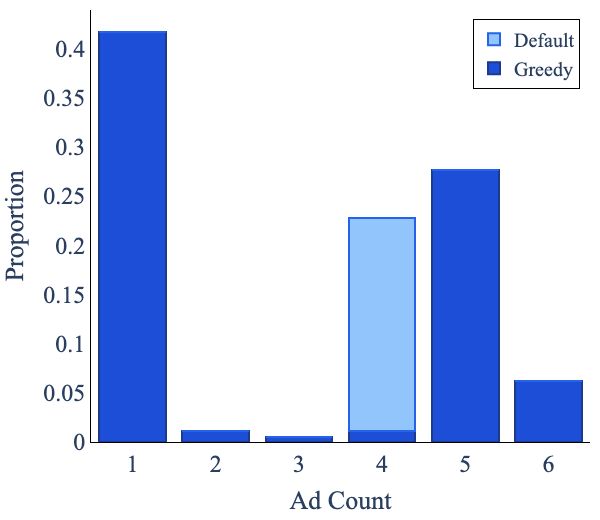}
    \caption{Distribution of the number of ads recommended by LAAL. }
    \label{fig:greedy_ad_count_distribution}
  \end{subfigure}
  \caption{Performance of LAAL (and system-level e-LAAL).}
  \label{fig:greedy_main_result}
\end{figure}

\paragraph{Frontier shift under deployment.}
Figure~\ref{fig:revenue_conversion_tradeoff_greedy} compares static cohorts, the LAAL cohort, and e-LAAL in the empirical revenue--conversion space. Two patterns emerge. First, LAAL/e-LAAL improves the empirical revenue--conversion trade-off relative to uniform static policies: it reaches revenue levels close to high-ad-load static rules (e.g., always show five ads, i.e., \(T_5\)) while maintaining conversions closer to low-ad-load rules (e.g., always show two ads). This is an efficiency gain, not movement along a fixed frontier. Second, LAAL and e-LAAL are numerically close, as expected given the \(90/10\) LAAL/exploration split, so the exploration cost is modest in practice.

\paragraph{How the policy behaves in practice.}
Figure~\ref{fig:greedy_ad_count_distribution} shows the distribution of ad counts selected by LAAL. On average, LAAL recommends 3.1 ads per search. It uses sliding-window estimates to choose greedily in 79\% of searches (those with at least \(m=100\) observations in the recent window) and falls back to \(i_{\mathrm{def}}=4\) in the remaining 21\% (sparser queries). The recommendation distribution is bimodal: LAAL frequently chooses either low ad loads (conversion-oriented) or high ad loads (revenue-oriented), with fewer intermediate choices. This is consistent with the discrete nature of the decision problem: interior ad loads are optimal only under a narrow knife-edge in incremental revenue versus conversion loss.

\subsection{Heterogeneity in Policy Performance}
\label{ssec:het_performance}

\paragraph{Performance across query segments.}
We next evaluate LAAL across the query segments defined in $\S$\ref{ssec:heterogeneity} (bucketed by control-group advertising CVR). Figure~\ref{fig:rev_conv_tradeoff_query_buckets} plots the empirical revenue--conversion trade-off within each segment. LAAL improves upon fixed ad-load (or static) policies in all segments, with the most pronounced gains in the high-CVR segment, where incremental ad slots are most likely to translate into additional paid installs and revenue.

\begin{figure}[htp!]
    \centering
    \begin{subfigure}[t]{0.3\textwidth}
        \centering
        \includegraphics[width=\linewidth]{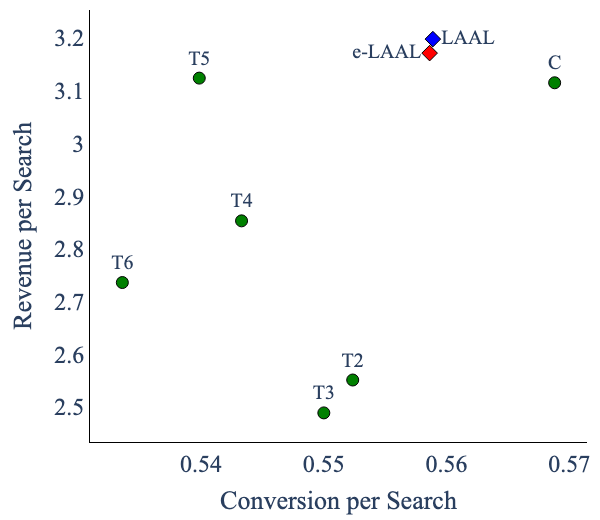}
        \caption{Low CVR segment}
        \label{fig:rev_conv_low}
    \end{subfigure}
    \hfill
    \begin{subfigure}[t]{0.3\textwidth}
        \centering
        \includegraphics[width=\linewidth]{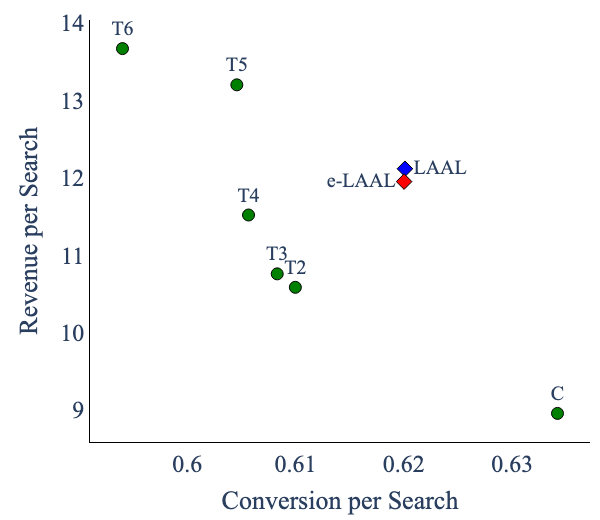}
        \caption{Mid CVR segment}
        \label{fig:rev_conv_mid}
    \end{subfigure}
    \hfill
    \begin{subfigure}[t]{0.3\textwidth}
        \centering
        \includegraphics[width=\linewidth]{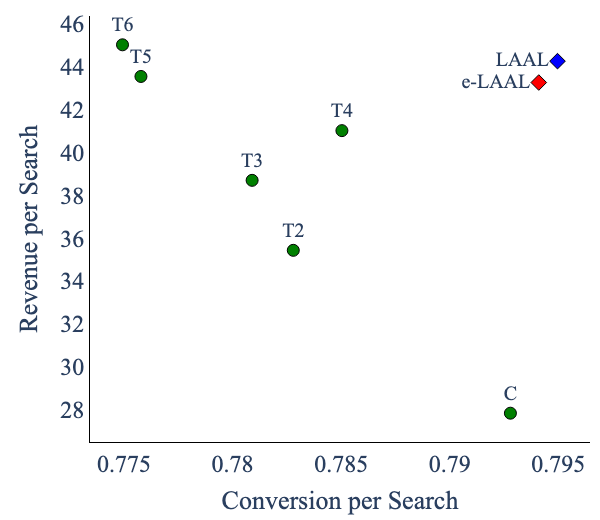}
        \caption{High CVR segment}
        \label{fig:rev_conv_high}
    \end{subfigure}
    \caption{Revenue--conversion trade-offs across query segments.}
    \label{fig:rev_conv_tradeoff_query_buckets}
\end{figure}

\begin{figure}[htp!]
    \centering
    \begin{subfigure}[t]{0.3\textwidth}
        \centering
        \includegraphics[width=\linewidth]{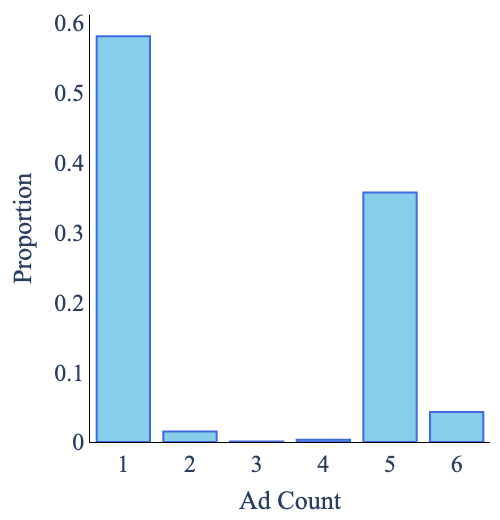}
        \caption{Low CVR segment}
        \label{fig:greedy_hist_low}
    \end{subfigure}
    \hfill
    \begin{subfigure}[t]{0.3\textwidth}
        \centering
        \includegraphics[width=\linewidth]{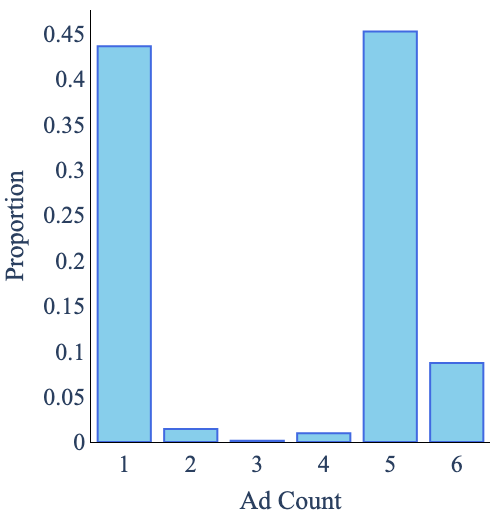}
        \caption{Mid CVR segment}
        \label{fig:greedy_hist_mid}
    \end{subfigure}
    \hfill
    \begin{subfigure}[t]{0.3\textwidth}
        \centering
        \includegraphics[width=\linewidth]{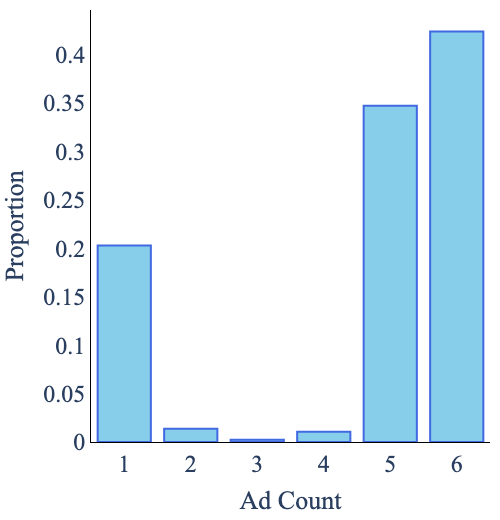}
        \caption{High CVR segment}
        \label{fig:greedy_hist_high}
    \end{subfigure}
    \caption{Distribution of ad loads recommended by LAAL across query segments.}
    \label{fig:histogram_query_buckets}
\end{figure}
Figure~\ref{fig:histogram_query_buckets} reports the distribution of LAAL’s recommended ad loads by segment. The policy adapts in the expected direction: it recommends fewer ads on low-CVR queries (where additional slots generate little revenue and primarily displace relevant organic content) and more ads on high-CVR queries (where additional slots generate revenue). This behavior mirrors the heterogeneity patterns in $\S$\ref{ssec:heterogeneity} and illustrates that LAAL operationalizes those insights using only recent outcome data, without hard-coding segment membership.

\paragraph{Within-query heterogeneity over time: brand entry and exit}

\begin{figure}[htp!]
    \centering
    \includegraphics[width=0.5\linewidth]{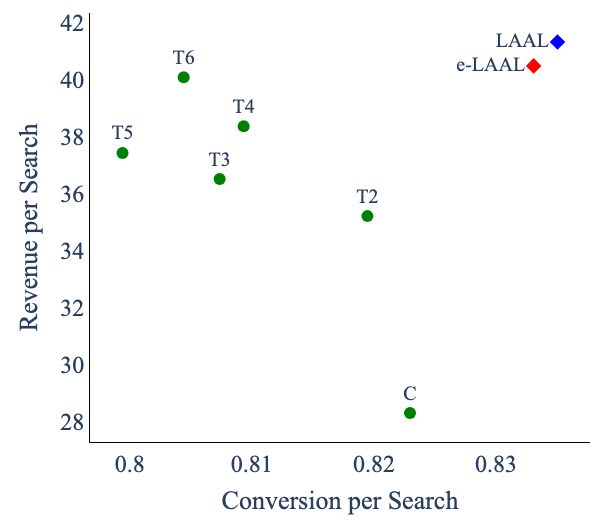}
    \caption{Performance of LAAL and e-LAAL on non-stationary queries (brand entry/exit)}
    \label{fig:LAAL_performance_non_stationary}
\end{figure}

Finally, we examine the performance of LAAL in non-stationary settings, i.e., queries with explicit time variation in advertiser composition. We focus on brand queries whose focal brand app advertises on some days but not others during the 22-day deployment window. Restricting attention to queries that experience both {\it Brand-On} and {\it Brand-Off} periods during deployment yields 315 such queries, which together account for 13\% of revenue, 13\% of conversions, and 11\% of searches in the control condition. We refer to these as \emph{non-stationary queries}. 
% \hy{Confirm with Mohammad that the definiton based on deployment period is correct} \mr{It is correct}

% $\mathcal{Q}_{\mathrm{NS}}$.

% For an outcome $Y\in\{\mathrm{Revenue},\mathrm{Conversion}\}$ and arm $e\in\{C,T_2,\ldots,T_6, \text{LAAL}, \text{e-LAAL}\}$
% , define the non-stationary-query mean outcome per search as
% \begin{equation}
% \label{eq:mean_ns_estimand}
% \mu^{\mathrm{NS}}_{Y,e}
% =
% \mathbb{E}\!\left[Y_s \,\middle|\, e_s=e,\ q_s\in\mathcal{Q}_{\mathrm{NS}}\right].
% \end{equation}

Figure~\ref{fig:LAAL_performance_non_stationary} reports the mean observed outcomes for non-stationary queries. We see that LAAL dominates static cohorts on both revenue and conversions, consistent with a key mechanism of LAAL: time-adaptive reallocation when the reward frontier shifts.

A direct test of adaptivity is whether LAAL recommends different ad loads for the same query depending on brand presence. For the non-stationary queries, we compare LAAL's recommended ad-load distributions during \emph{Brand-On} and \emph{Brand-Off} searches, using within-query reweighting so that queries with different Brand-On shares are comparable; details of the reweighting procedure are provided in Web Appendix \(\S\)\ref{appsec:brand_on_off_reweighting}.

Figure~\ref{fig:brand_on_off_hist} plots the resulting recommendation histograms. When the brand is absent, LAAL concentrates on lower ad loads; when the brand is present, the distribution shifts toward higher ad loads. This within-query reallocation is precisely the behavior a time-adaptive ad-load policy should exhibit in response to changing advertiser composition.

% \begin{figure}[htp!]
%     \centering
%     \begin{subfigure}[t]{0.35\linewidth}
%         \centering
%         \includegraphics[width=\linewidth]{Pics/greedy_conversion_line_brand.png}
%         \caption{Conversion per search across cohorts and LAAL.}
%     \end{subfigure}
%     \hfill
%     \begin{subfigure}[t]{0.35\linewidth}
%         \centering
%         \includegraphics[width=\linewidth]{Pics/greedy_revenue_line_brand.png}
%         \caption{Revenue per search across cohorts and LAAL.}
%     \end{subfigure}
%     \caption{Performance on non-stationary queries (brand entry/exit). LAAL dominates on both outcomes.}
%     \label{fig:brand_entrance_revenue_conversion}
% \end{figure}

\begin{figure}[htp!]
    \centering
    \begin{subfigure}[t]{0.45\textwidth}
        \centering
        \includegraphics[width=\linewidth]{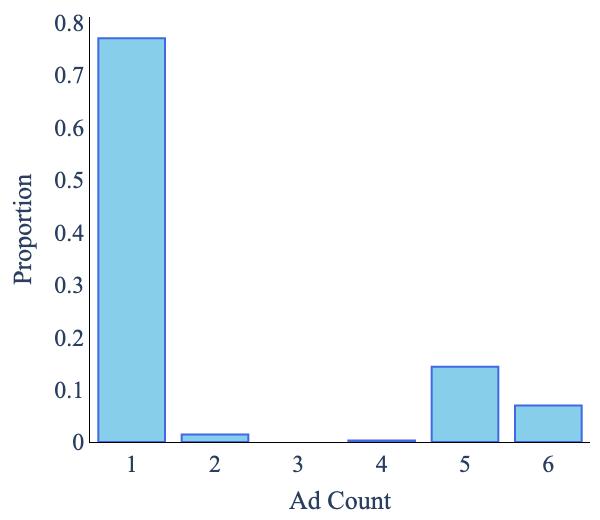}
        \caption{\emph{Brand-Off} state}
        \label{fig:brand_off_recs}
    \end{subfigure}
    \hfill
    \begin{subfigure}[t]{0.45\textwidth}
        \centering
        \includegraphics[width=\linewidth]{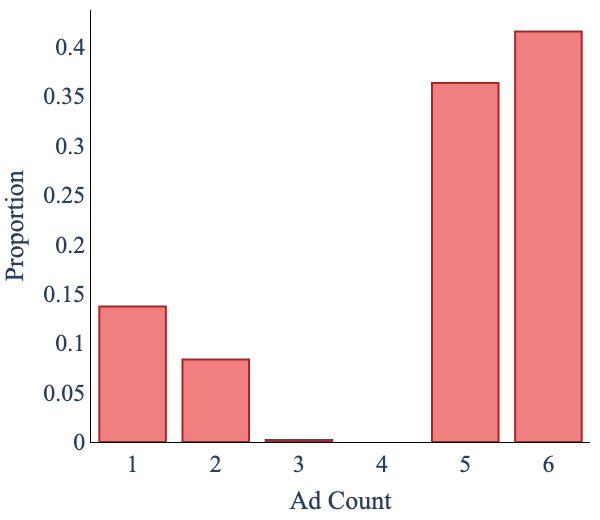}
        \caption{\emph{Brand-On} state}
        \label{fig:brand_on_recs}
    \end{subfigure}
    \caption{Distribution of LAAL ad-load recommendations under \emph{Brand-Off} and \emph{Brand-On} states for non-stationary queries. Histograms are reweighted within query to make Brand-On and Brand-Off states comparable across queries with different Brand-On shares.}
    \label{fig:brand_on_off_hist}
\end{figure}

\subsection{Managerial Implications}

Three implications follow from the deployment evidence.  
First, \emph{fixed exploration with adaptive exploitation} is implementable at scale: the platform can pre-commit a small exploration budget (\(10\%\) here) and still obtain substantial frontier gains through local adaptive exploitation.  
Second, \emph{query-level ad-load personalization matters}: the platform can improve monetization without proportionate conversion losses by reallocating ad load at the query level rather than committing to a single global ad count; this matters because monetization-conversion trade-offs vary sharply across query types.  
Third, \emph{time adaptivity is economically material}: when advertiser composition shifts (e.g., brand entry/exit), adaptive ad-load reallocation can dominate static rules on both objectives.

Taken together, the results indicate that e-LAAL is not only statistically and operationally feasible, but also managerially actionable: it provides a practical mechanism to improve monetization while preserving user outcomes in real large-scale non-stationary sponsored search platforms.

\section{Comparisons with Benchmark Policies}
\label{sec:comp_benchmark}

The previous section compared LAAL and the deployed e-LAAL system to the contemporaneous static cohorts \(\{C,T_2,\ldots,T_6\}\) using IPW-normalized outcomes from the deployment window. Those comparisons evaluate LAAL against the fixed ad-load regimes that were actually deployed alongside the adaptive arm. However, they benchmark LAAL only against the six individual ``always show \(i\) ads'' regimes; they do not ask whether LAAL outperforms static rules that a platform might have selected using performance data, such as a historically chosen uniform rule or a query-specific static mapping. In this section, we therefore provide a broader benchmark analysis against static policies that are still time-invariant once chosen, but are selected using either historical or deployment-period data.

We study two classes of static policies. The first class consists of \emph{uniform static policies}, which show the same number of ads for every query. The second class consists of \emph{query-dependent static policies}, which allow the ad load to vary by query but keep each query's assignment fixed over time. These benchmarks allow us to separate the value of query sensitivity from the value of time adaptivity.

For each policy class, we distinguish two timing regimes. A \emph{historical benchmark} is selected using the 66-day experiment window and is therefore feasible ex ante before the 22-day deployment. It answers: what static policy would the platform have chosen if it had used the earlier experiment to set a rule for deployment? A \emph{deployment oracle} is selected using the 22-day deployment window itself. It is not feasible ex ante; rather, it is an ex post benchmark that asks how LAAL performs relative to the best static rule in the realized deployment environment. Thus, historical benchmarks evaluate implementable static alternatives, while deployment oracles provide stronger diagnostic comparisons.

\paragraph{Notation:}
For compactness, let \(H\) denote the historical 66-day experiment window and \(D\) denote the 22-day deployment window. Let \(\mathcal S_W\) and \(\mathcal V_W\) denote the sets of searches and users observed during window \(W\in\{H,D\}\). For each ad load \(i\in\mathcal A\), let \(\mathcal S_W(i)\) and \(\mathcal V_W(i)\) denote the sets of searches and users in window \(W\) assigned to the static cohort that displays ad load \(i\). Let \(w_i\) denote the traffic share of the corresponding static cohort. For query-dependent benchmarks, let \(\mathcal S_W(q,i)\) denote the sets of searches observed during window $W$ for query $q$ and ad load $i$.

\subsection{Uniform Static Benchmarks}
\label{ssec:benchmark_static}

A uniform static policy chooses the same number of sponsored slots for all queries. These policies correspond directly to the fixed experimental cohorts: policy \(i\) implements ``always show \(i\) ads.''

\begin{defn}[Uniform Static Policy Class]
For each \(i\in\mathcal A\), define the uniform static policy \(\pi_i^U:\mathcal Q\to\mathcal A\) by $
\pi_i^U(q)=i$ $\forall \; q \; \in \mathcal Q$. The class of uniform static policies is $
\Pi^U=\{\pi_i^U:i\in\mathcal A\}$.
\end{defn}
For \(W\in\{H,D\}\) and \(i\in\mathcal A\), define the sample mean reward per user under uniform ad load \(i\) as
\[
\widehat R_W^U(i;\lambda)
:=
\frac{1}{|\mathcal V_W(i)|}
\sum_{s\in\mathcal S_W(i)} r(y_s,\lambda).
\]
This user-level normalization preserves the interpretation of the static cohorts as persistent user-level regimes.

\begin{defn}[Historical Uniform Benchmark and Deployment Uniform Oracle]
For \(W\in\{H,D\}\), define
\[
i_W^{U,\star}\in
\arg\max_{i\in\mathcal A}\widehat R_W^U(i;\lambda),
\qquad
\pi_W^{U,\star}:=\pi_{i_W^{U,\star}}^U.
\]
We call \(\pi_H^{U,\star}\) the \emph{historical uniform benchmark} and \(\pi_D^{U,\star}\) the \emph{deployment uniform oracle}.
\end{defn}

The two benchmarks are selected by the same reward-maximization rule, but from different samples. The historical uniform benchmark \(\pi_H^{U,\star}\) uses only the 66-day experiment and is therefore feasible before deployment. The deployment uniform oracle \(\pi_D^{U,\star}\) uses the deployment window itself and is therefore an ex post benchmark. In our data, \(i_H^{U,\star}=3\), corresponding to the 3-ad policy \(T_3\), whereas \(i_D^{U,\star}=5\), corresponding to the 5-ad policy \(T_5\). The shift from \(T_3\) historically to \(T_5\) in deployment shows that even the best platform-wide static ad load changes over time. This instability is precisely the type of environment in which an adaptive policy such as LAAL can improve on fixed rules.

We compare LAAL to the uniform benchmarks using user-level outcomes from the deployment window. For each outcome $
Y\in\{\mathrm{Revenue},\mathrm{Conversion},\mathrm{Reward}\}$,
where \(\mathrm{Reward}=\mathrm{Conversion}+0.01 \times \mathrm{Revenue}\), we estimate the following equation:
\begin{equation}
\label{eq:laal_vs_uniform}
Y_u
=
\alpha_L^Y
+
\sum_{j\in\{C,T_2,\ldots,T_6\}}
\beta_j^Y\mathbf 1\{e(u)=j\}
+
\varepsilon_u^Y.
\end{equation}
Here \(Y_u\) is the total outcome for user \(u\) over the deployment window, \(\alpha_L^Y\) is the mean outcome in the LAAL arm, and \(\beta_j^Y\) measures the difference between static cohort \(j\) and LAAL.

Table~\ref{tab:laal_regression_rewards} reports the estimates. LAAL strictly dominates the historical uniform benchmark \(T_3\): relative to \(T_3\), LAAL delivers significantly higher revenue and conversions, and therefore significantly higher scalarized reward. LAAL also weakly dominates the deployment uniform oracle \(T_5\): revenue is statistically indistinguishable between LAAL and \(T_5\), while LAAL delivers significantly higher conversions. Thus, LAAL outperforms both the feasible ex ante historical uniform benchmark as well as the ex post deployment uniform oracle.

\begin{table}[htp!]
\centering
\small
\caption{User-level deployment regressions comparing LAAL to uniform static benchmarks.}
\label{tab:laal_regression_rewards}
\begin{tabular}{lccc}
\toprule
 & \textbf{Revenue} & \textbf{Conversion} & \textbf{Reward} \\
\midrule
LAAL & 36.13$^{***}$ (0.08) & 2.127$^{***}$ (0.000) & 2.488$^{***}$ (0.001) \\
$C$   & -10.33$^{***}$ (0.33) & 0.025$^{***}$ (0.002) & -0.078$^{***}$ (0.004) \\
$T_2$ & -6.65$^{***}$ (0.72)  & -0.023$^{***}$ (0.004) & -0.090$^{***}$ (0.009) \\
$T_3$ & -4.96$^{***}$ (0.72)  & -0.026$^{***}$ (0.004) & -0.075$^{***}$ (0.009) \\
$T_4$ & -3.02$^{***}$ (0.72)  & -0.036$^{***}$ (0.004) & -0.066$^{***}$ (0.009) \\
$T_5$ & -0.02 (0.72)          & -0.063$^{***}$ (0.004) & -0.064$^{***}$ (0.009) \\
$T_6$ & 0.55 (0.72)           & -0.073$^{***}$ (0.004) & -0.067$^{***}$ (0.009) \\
\midrule
Obs. & 22,330,704 & 22,330,704 & 22,330,704 \\
$R^2$ & $<0.001$ & $<0.001$ & $<0.001$ \\
\bottomrule
\end{tabular}
\par\medskip
\footnotesize{\textit{Note:} The LAAL row reports the mean outcome in the LAAL arm. Rows \(C,T_2,\dots,T_6\) report differences relative to LAAL. Reward is defined as \(\mathrm{Conversion}+0.01\cdot\mathrm{Revenue}\). Standard errors in parentheses. Significance levels: ${}^{***}p<0.001$, ${}^{**}p<0.01$, ${}^{*}p<0.05$.}
\end{table}

Next, we use an iso-reward line to visualize whether LAAL improves the scalarized objective relative to the best uniform static rule in deployment. Recall that the platform's scalarized reward is
\[
\mathrm{Reward}=\mathrm{Conversion}+\lambda\cdot \mathrm{Revenue}.
\]
The deployment uniform oracle \(\pi_D^{U,\star}\) is the uniform static policy with the highest deployment-period reward; in our data this corresponds to \(T_5\). Let its total deployment-period reward be $ B_D^U := |\mathcal V_D|\cdot
\widehat R_D^U(i_D^{U,\star};\lambda)$, where $\mathcal V_D$ is the set of all users observed during the 22-day deployment period. Then, the associated iso-reward line is
\begin{equation}
\mathcal L_D^U
:=
\{(\mathcal C,\mathcal M)\in\mathbb R_+^2:
\mathcal C+\lambda\mathcal M=B_D^U\},
\label{eq:isoreward}
\end{equation}
where \(\mathcal C\) denotes total conversions and \(\mathcal M\) denotes total revenue in the deployment window. This line contains all revenue--conversion bundles that deliver exactly the same scalarized reward as the best uniform static policy in deployment. Points above the line have a higher reward than the deployment uniform oracle, because they deliver either more conversions for a given revenue level or more revenue for a given conversion level. Points below the line have a lower reward. Thus, the iso-reward line provides a visual test of whether LAAL improves the platform's scalarized objective, not merely one of the two underlying outcomes. We plot this line together with LAAL, and the static cohorts in Figure~\ref{fig:benchmark_comparison} (along with the historical query-dependent benchmark discussed next). We see that LAAL clearly improves the reward compared to the best uniform fixed-ad load policy.

\begin{figure}[htp!]
    \centering
    \includegraphics[width=0.5\linewidth]{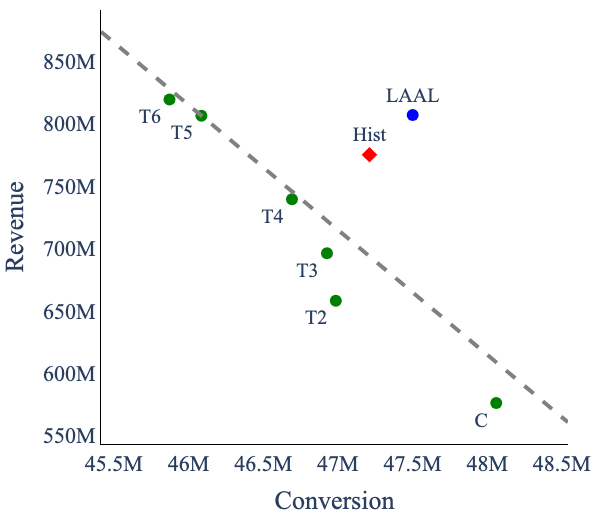}
    \caption{Revenue--conversion comparison of LAAL with static benchmarks. The figure shows the static cohorts, the iso-reward line induced by the deployment uniform oracle, and the historical query-dependent benchmark.}
    \label{fig:benchmark_comparison}
\end{figure}

\subsection{Query-Dependent Benchmarks}
\label{ssec:benchmark_query}

Uniform static rules ignore the cross-query heterogeneity documented in 
\S~\ref{ssec:heterogeneity}. A stronger static benchmark therefore allows the ad load to vary across queries, while still requiring the rule to remain fixed over time. These benchmarks help separate two sources of improvement: the value of \emph{query sensitivity} and the value of \emph{time adaptivity}. If a fixed query-level rule performs close to LAAL, then most of LAAL's gain comes from choosing different ad loads for different queries. If LAAL outperforms such a rule, then its additional gain comes from adapting within the same query as market conditions change.

\begin{defn}[Query-Dependent Static Policy Class]
Let \(\Pi^Q\) denote the class of deterministic query-dependent static policies $\Pi^Q=\{\pi:\mathcal Q\to\mathcal A\}$, where \(\pi(q)\) is the fixed ad load assigned to query \(q\).
\end{defn}
Since each uniform rule applies the same ad load to every query, uniform static policies are a special case of query-dependent static policies $\Pi^U\subseteq \Pi^Q$. 

For query-dependent benchmarks, whenever \(|\mathcal S_W(q,i)|>0\), define the sample mean reward per search for query \(q\) under ad load \(i\) as
\begin{equation}
\widehat r_W(q,i;\lambda)
:=
\frac{1}{|\mathcal S_W(q,i)|}
\sum_{s\in \mathcal S_W(q,i)} r(y_s,\lambda).
\end{equation}

\begin{defn}[Historical Query-Dependent Benchmark and Deployment Query-Dependent Empirical Oracle]
For \(W\in\{H,D\}\) and any query \(q\) with at least one observed query--ad-load cell in window \(W\), define
\[
\pi_W^{Q,\star}(q)
\in
\arg\max_{\{i\in\mathcal A:\,|\mathcal S_W(q,i)|>0\}}
\widehat r_W(q,i;\lambda).
\]
For queries that do not appear in the historical window, set \(\pi_H^{Q,\star}(q)=i_{\mathrm{def}}\). We refer to \(\pi_H^{Q,\star}\) as the historical query-dependent benchmark and to \(\pi_D^{Q,\star}\) as the deployment query-dependent empirical oracle.
\end{defn}

The historical query-dependent benchmark \(\pi_{H}^{Q,\star}\) is selected using the 66-day experiment window and is therefore feasible ex ante before deployment. It represents the best fixed query-level rule the platform can construct from historical experimental evidence. By contrast, the deployment query-dependent empirical oracle \(\pi_{D}^{Q,\star}\) is selected using the 22-day deployment window itself. It is not a deployable policy; rather, it is an ex post benchmark that asks how close LAAL comes to the best time-invariant query-level rule in the realized deployment environment.

\paragraph{Historical query-dependent benchmark in deployment.}

We first evaluate the feasible ex ante benchmark \(\pi_{H}^{Q,\star}\) using deployment-window outcomes. This comparison asks whether a query-sensitive static rule selected from the earlier experiment would still perform well in the later deployment environment.

For each ad load \(i\in\mathcal A\), define the set of queries assigned to \(i\) by the historical query-dependent benchmark, $Q_i^{H} :=
\{q\in\mathcal Q:\pi_{H}^{Q,\star}(q)=i\}$. 
Then, for any outcome \(Y\in\{\mathrm{Revenue},\mathrm{Conversion}\}\), define the IPW estimate of the deployment-window total outcome under \(\pi_{H}^{Q,\star}\) as
\begin{equation}
\label{eq:hist_q_ipw_total}
\widehat{\mathcal Y}_{H \rightarrow D}^{Q}(Y)
=
\sum_{i\in\mathcal A}
\sum_{s\in\mathcal S_{D}(i)}
\frac{\mathbf 1\{q_s\in Q_i^{H}\}}{w_i}
Y_s.
\end{equation}
The estimator uses deployment searches from static cohort \(i\) whenever the historical query-dependent rule would have assigned the corresponding query to ad load \(i\), and rescales by the static cohort traffic share \(w_i\). Thus, it evaluates a historically selected query-level rule in the deployment environment using contemporaneous randomized support from the static arms.

Using \(Y=\mathrm{Revenue}\) and \(Y=\mathrm{Conversion}\) in Equation~\eqref{eq:hist_q_ipw_total}, we obtain the platform-level revenue and conversion totals generated by the historical query-dependent benchmark during deployment and show it in Figure~\ref{fig:benchmark_comparison}. We see that the historical query-dependent benchmark improves the empirical revenue--conversion trade-off relative to uniform static rules, confirming the value of query sensitivity. However, it remains materially below LAAL. This gap shows that query sensitivity alone is not sufficient: a fixed query-level rule selected from historical data can become stale when within-query market conditions shift during deployment.

\paragraph{Empirical regret relative to the deployment query-dependent empirical oracle.}

We next compare LAAL to the deployment query-dependent empirical oracle \(\pi_D^{Q,\star}\). This oracle is selected using deployment outcomes themselves: for each query, it chooses the fixed ad load that performed best during the same 22-day deployment window. It therefore could not have been implemented ex ante at the start of deployment. We use it instead as an ex post diagnostic benchmark that measures how close LAAL comes to the best fixed query-level rule in the realized deployment environment. Let $\mathcal S_{D}(q, L)$ 
denote the set of deployment-window searches for query \(q\) served by the LAAL arm. Whenever \(|\mathcal S_{D}(q, L)|>0\), define LAAL's deployment-window mean reward on query \(q\) as
\begin{equation}
\widehat r_{D}^{L}(q;\lambda)
=
\frac{1}{|\mathcal S_{D}(q, L)|}
\sum_{s\in\mathcal S_{D}(q, L)}
r(y_s,\lambda).
\end{equation}
For each query $q$, we define LAAL's query-level empirical regret relative to the deployment query-dependent empirical oracle as
\begin{equation}
\label{eq:empirical_regret_query}
\widehat{\mathrm{EmpReg}}(q)
=
\widehat r_{D}
\!\left(q,\pi_{D}^{Q,\star}(q);\lambda\right)
-
\widehat r_{D}^{L}(q;\lambda),
\end{equation}
This quantity is the per-search reward gap between LAAL and the best fixed ad load for query \(q\), where the fixed ad load is chosen ex-post from deployment data. The deployment query-dependent empirical oracle, $\widehat r_{D}
\!\left(q,\pi_{D}^{Q,\star}(q);\lambda\right)$, is an upper benchmark within the class of time-invariant query-level rules. It assigns a single ad load to each query for the entire deployment window based on ex-post optimal ad load during deployment. As such, it is an infeasible benchmark. LAAL, by contrast, can change the ad load for the same query over time. Therefore, empirical regret can be negative for queries where within-query non-stationarity allows LAAL to outperform any single fixed ad load over the evaluation window.

% For each \(q\), define LAAL's query-level empirical regret relative to the deployment query-dependent empirical oracle as
% \begin{equation}
% \widehat{\mathrm{EmpReg}}(q)
% =
% \widehat r_D\!\left(q,\pi_D^{Q,\star}(q);\lambda\right)
% -
% \widehat r_D^{L}(q;\lambda).
% \end{equation}
% This quantity is the per-search reward gap between LAAL and the best fixed ad load for query \(q\), where the fixed ad load is chosen ex post from deployment data. Because the oracle assigns a single ad load to each query for the entire deployment window, the reward from this policy ($\widehat r_{D}
% \!\left(q,\pi_{D}^{Q,\star}(q);\lambda\right)$) is an upper benchmark only within the class of time-invariant query-level rules. LAAL, by contrast, can change the ad load for the same query over time. Therefore, empirical regret can be negative when LAAL's within-query adaptation outperforms every single fixed ad load for that query over the evaluation window.

For the empirical analysis, we restrict our attention to the 2,451 queries that consistently maintained sufficient support for LAAL's greedy selection, i.e., queries with at least 100 searches during every rolling 3-day window. 
We rank these queries by deployment search frequency from highest to lowest and report mean empirical regret within bins of that ranking in Figure~\ref{fig:regret_per_auction}. We construct 95\% confidence intervals using bootstrap resampling at the query level. Mean empirical regret declines with query frequency: in the highest-frequency bins, the confidence intervals include zero, indicating that LAAL performs nearly as well as the best query-dependent static rule when local support is ample. Empirical regret is larger in lower-frequency bins, where local estimation is noisier.
% and fallback behavior is more common.

\begin{figure}[htp!]
    \centering
    \includegraphics[width=0.5\linewidth]{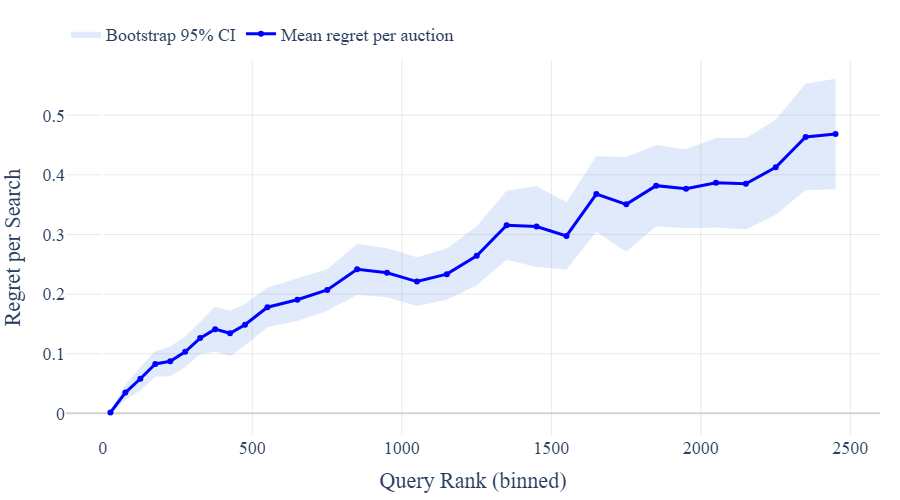}
    \caption{Average empirical regret per search of LAAL relative to the deployment query-dependent empirical oracle, by query-volume bin. The sample is restricted to the 2,451 queries.}
    \label{fig:regret_per_auction}
\end{figure}

Taken together, the query-dependent benchmark comparisons deliver two takeaways. First, allowing ad load to vary across queries improves over uniform static rules, confirming the importance of query sensitivity. Second, LAAL continues to outperform the historical query-dependent benchmark and exhibits near-zero empirical regret relative to the deployment query-dependent empirical oracle on the highest-volume queries. Thus, LAAL's gains come not only from choosing different ad loads for different queries, but also from adapting ad load within the same query as market conditions evolve.

\section{Discussion and Robustness Checks}

The benchmark results above show that LAAL improves the revenue--conversion frontier relative to both uniform and query-dependent static policies. We now examine two additional concerns about the interpretation of these gains. First, we ask whether LAAL's performance is stable throughout the deployment window, rather than being driven by short-lived launch-period effects. Second, we ask whether the gains are driven solely by brand-query mechanics, where the focal brand advertiser's participation decision may interact with the platform's ad-load policy. 

\subsection{Performance Over Time}

One potential concern with the deployment results is that they may not be stable over the evaluation window. Launch-period effects could inflate early gains, and performance may erode as advertisers adjust their bids, budgets, or participation. Such responses would appear as changes in observed query-level rewards, which LAAL is designed to accommodate through its sliding-window updates and which are captured as reward drift in Theorem~\ref{thm:elaal_regret_main}. Nevertheless, whether LAAL's relative performance remains stable during deployment is an empirical question.

We assess whether LAAL's improvement is stable over time by testing for a linear trend in daily outcomes relative to the control group. For each outcome \(Y\in\{\text{Revenue},\text{Conversion}\}\) and date \(d\), we estimate
\begin{equation}
\label{eq:time_trend_reg}
Y_{L,d}\;=\;\alpha^{Y}\,Y_{C,d}\;+\;\beta^Y\,d\;+\;\varepsilon_{Y,d},
\end{equation}
where \(Y_{L,d}\) and \(Y_{C,d}\) denote the LAAL and control IPW-scaled totals on day \(d\), respectively, and \(d\) is mean-centered within the 22-day window. In this specification, \(\alpha^Y\) captures the average multiplicative change in LAAL relative to control, controlling for common day-to-day shocks through \(Y_{C,d}\), while \(\beta^Y\) tests whether LAAL's relative performance drifts over time.

\begin{table}[!htbp]
\centering
\caption{LAAL vs.\ Control With Time Trend}
\label{tab:trend_stability}
\begin{tabular}{lcc}
\toprule
 & \textbf{Revenue} & \textbf{Conversion} \\
\midrule
Relative change $\alpha$ & 1.3685$^{***}$ (0.0421) & 0.9891$^{***}$ (0.0016) \\
Trend $\beta$            & 487.04 (884.65)         & -5.05 (4.5351)          \\
\midrule
Observations & 22 & 22 \\
\bottomrule
\end{tabular}

\par\vspace{0.25em}
\footnotesize
Standard errors in parentheses. Significance: $^{*}p<0.05$, $^{**}p<0.01$, $^{***}p<0.001$.
\end{table}
Table~\ref{tab:trend_stability} reports the estimates. Consistent with the results from \S\ref{ssec:main_result}, LAAL increases revenue by 36.8\% and decreases conversions by 1.1\% relative to the control condition. Importantly, the estimated trend coefficients \(\beta^Y\) are statistically indistinguishable from zero for both outcomes, suggesting that LAAL's relative performance did not erode over the 22-day evaluation window. Nevertheless, we caution that 22 days may be too short for advertisers to fully re-optimize. The analysis therefore establishes stability over the observed deployment window, but does not rule out slower or longer-run equilibrium responses.

\subsection{Robustness Beyond Brand Queries}
\label{ssec:brand_advertising}

For brand queries, the platform often shows more ads, and the focal brand’s decision to advertise can affect which layouts are observed and how profitable additional slots are. This creates the possibility that advertiser behavior and the platform’s ad-allocation rule interact: if the brand’s propensity to advertise is itself affected by the realized ad load, then LAAL’s performance on brand queries may partially reflect this strategic response rather than purely mechanical reallocation.

As a robustness check, we therefore repeat the main revenue--conversion comparison from $\S$\ref{ssec:main_result} on \emph{non-brand} queries only, where no single advertiser is sufficiently prominent to plausibly manipulate the overall design. The results are reported in Figure~\ref{fig:non_brand_revenue_conversion}. On non-brand queries, LAAL continues to outperform all non-control static cohorts in conversions, and it dominates all alternatives in revenue except the highest ad-load cohorts (\(T_5\) and \(T_6\)). These results indicate that the gains from LAAL are not driven solely by brand-query mechanics; substantial improvements persist in the broader set of non-brand queries.

% \begin{figure}[htp!]
%     \centering
%     \begin{subfigure}[t]{0.45\linewidth}
%         \centering
%         \includegraphics[width=\linewidth]{Pics/non_brand_conv_per_search.png}
%         \caption{Conversion per search across variants (non-brand queries).}
%     \end{subfigure}
%     \hfill
%     \begin{subfigure}[t]{0.45\linewidth}
%         \centering
%         \includegraphics[width=\linewidth]{Pics/non_brand_rev_per_search.png}
%         \caption{Revenue per search across variants (non-brand queries).}
%     \end{subfigure}
%     \caption{Performance of experimental variants on non-brand queries.}
%     \label{fig:non_brand_revenue_conversion}
% \end{figure}

\begin{figure}[htp!]
    \centering
        \centering
        \includegraphics[width=0.5\linewidth]{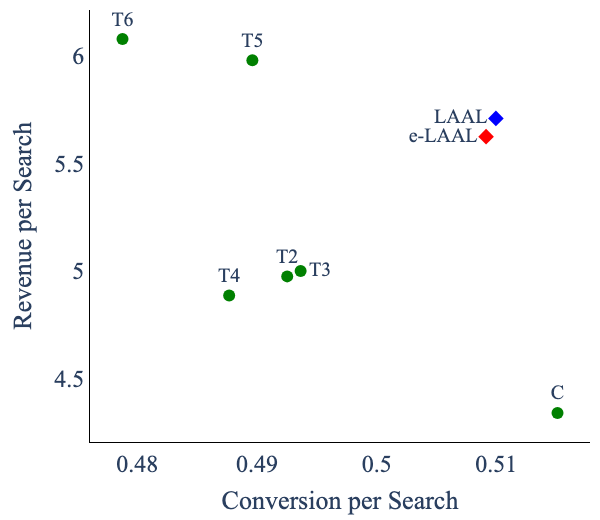}
    \caption{Revenue--conversion across fixed ad-load cohorts, LAAL, and e-LAAL on \textit{Non-Brand} Queries.}
    \label{fig:non_brand_revenue_conversion}
\end{figure}

%Finally, we note that it is possible to make LAAL explicitly \emph{brand-aware}. If brand advertisers’ participation decisions respond to ad load, the platform can estimate (i) the probability that a brand begins advertising when currently inactive and (ii) the probability it exits when active. Action selection would then compare layouts based not only on short-run realized rewards, but on the equilibrium rewards implied by these predicted participation responses (i.e., selecting the ad load that maximizes expected reward under the counterfactual equilibrium induced by the policy).

\section{Conclusion}
\label{sec:conclusion}

Ad-load design is a central supply-side decision in sponsored search: increasing sponsored inventory can raise short-run revenue, but may also crowd out relevant organic results and reduce user-side outcomes. Using a 66-day user-level randomized field experiment, we document this revenue--conversion trade-off at scale. Higher ad loads substantially increase advertising revenue, but reduce total search conversions and engagement. We further show that these effects are neither homogeneous nor stable over time. Additional slots generate large revenue gains on queries with high-conversion advertisers, but little or no incremental revenue on queries where ads convert poorly. Within-query variation in brand-advertiser presence further shows that the revenue--conversion frontier can shift as market conditions change. 

Motivated by these findings, we design and deploy exploration-augmented Locally Adaptive Ad Load (e-LAAL). The deployed architecture combines LAAL, a model-free query-level policy that uses recent outcomes to adapt the number of sponsored slots, with persistent static exploration arms that maintain support and provide fixed-policy counterfactuals. In a 22-day production deployment serving over 22 million users and 77 million searches, LAAL recommends 3.1 ads per search on average and improves the empirical revenue--conversion frontier relative to deployed static benchmarks. It achieves revenue comparable to high-ad-load static policies while preserving conversions closer to low-ad-load benchmarks. The policy's recommendations are economically interpretable: it shows fewer ads when advertiser quality is low or the brand advertiser is absent, and more ads when sponsored demand is valuable or the brand advertiser is active.

The e-LAAL architecture also provides a practical and theoretically grounded way to implement adaptive ad-load allocation. By maintaining parallel static exploration arms, the platform preserves support for every ad-load level and obtains contemporaneous fixed-policy counterfactuals, while the adaptive arm exploits recent local evidence. We provide a finite-time, high-probability dynamic-regret guarantee for e-LAAL that accounts for non-stationarity, sliding-window estimation, and deviations between designed traffic shares and realized query-level support. Empirically, LAAL outperforms both uniform static benchmarks and historical query-dependent static benchmarks, and exhibits near-zero empirical regret on high-volume queries relative to an ex post deployment oracle.

Together, the results show that ad load should not be treated as a fixed platform-wide parameter or as a time-invariant query-level rule. Effective ad-load design must be both query-sensitive and time-adaptive. More broadly, the paper shows that sponsored-search platforms can ease the revenue--consumer-experience trade-off not only through auction design, pricing, ranking, or targeting, but also by dynamically controlling the supply of sponsored attention. LAAL provides a lightweight, scalable, and portable blueprint for doing so in large-scale, non-stationary sponsored-search markets.

\section*{Funding and Competing Interests}
The first author was a paid intern at the company that provided the data when the experiments were conducted. The company had no input or constraints on what is reported in the manuscript.

%\bibliographystyle{plainnat}
%\bibliography{references}
\putbib 
\end{bibunit}

\newpage
\begin{appendices}

\setcounter{table}{0}
\setcounter{figure}{0}
\setcounter{equation}{0}
\setcounter{page}{0}
\renewcommand{\thetable}{A\arabic{table}}
\renewcommand{\thefigure}{A\arabic{figure}}
\renewcommand{\theequation}{A\arabic{equation}}
\renewcommand{\thepage}{\roman{page}}

\renewcommand{\theHtable}{A\arabic{table}}
\renewcommand{\theHfigure}{A\arabic{figure}}
\renewcommand{\theHequation}{A\arabic{equation}}

\pagenumbering{roman}
\begin{bibunit}

\section{Brand Query Layouts}
\label{apdx:brand_query_layouts}
\begin{figure}[h]
  \centering
  \begin{subfigure}[t]{0.45\textwidth}
    \centering
    \includegraphics[width=0.5\linewidth]{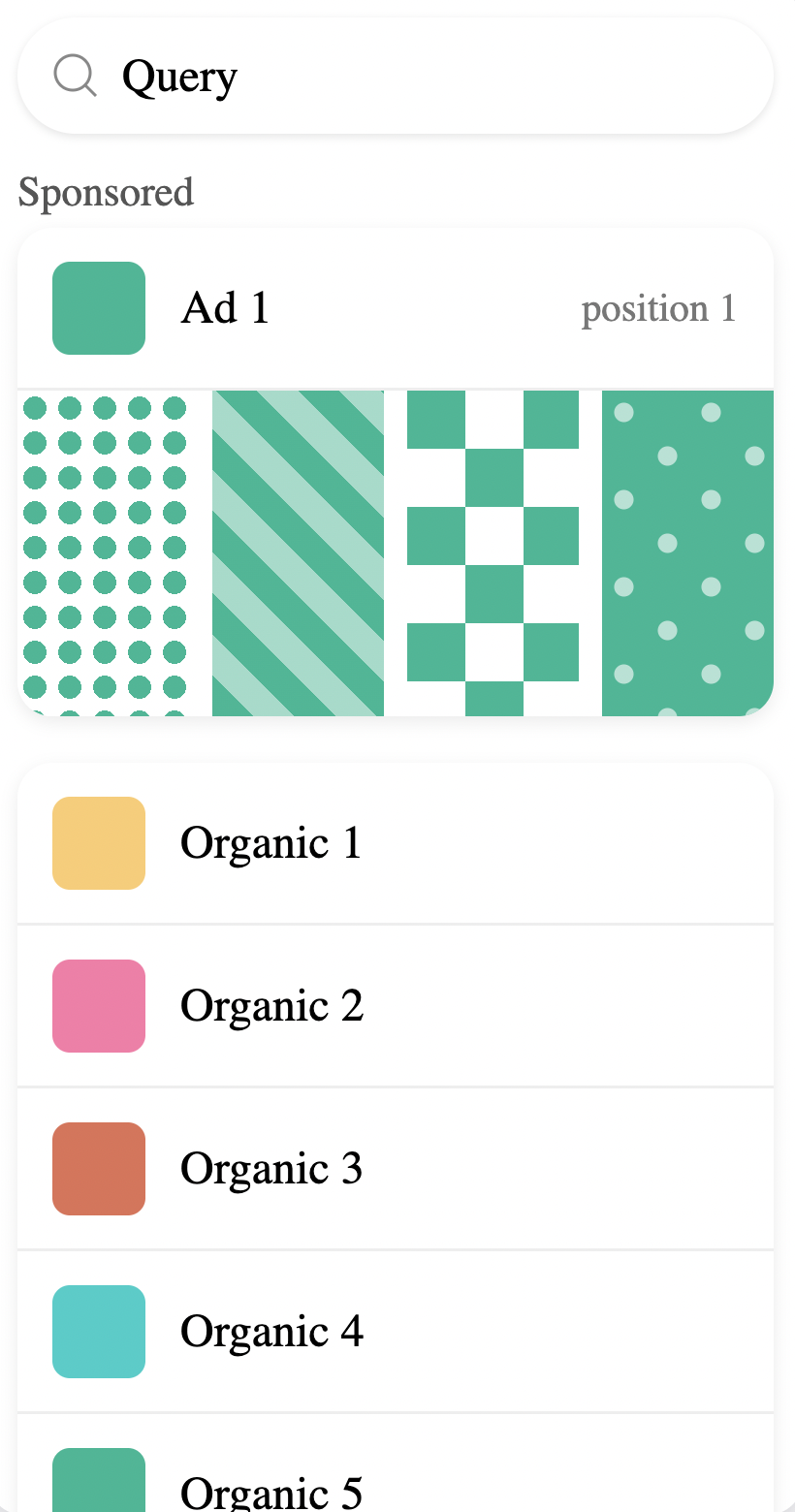}
    \caption{Brand Query Single-Ad Layout }
    \label{fig:single_ad_brand_layout}
  \end{subfigure}
  \hfill
  \begin{subfigure}[t]{0.45\textwidth}
    \centering
    \includegraphics[width=0.5\linewidth]{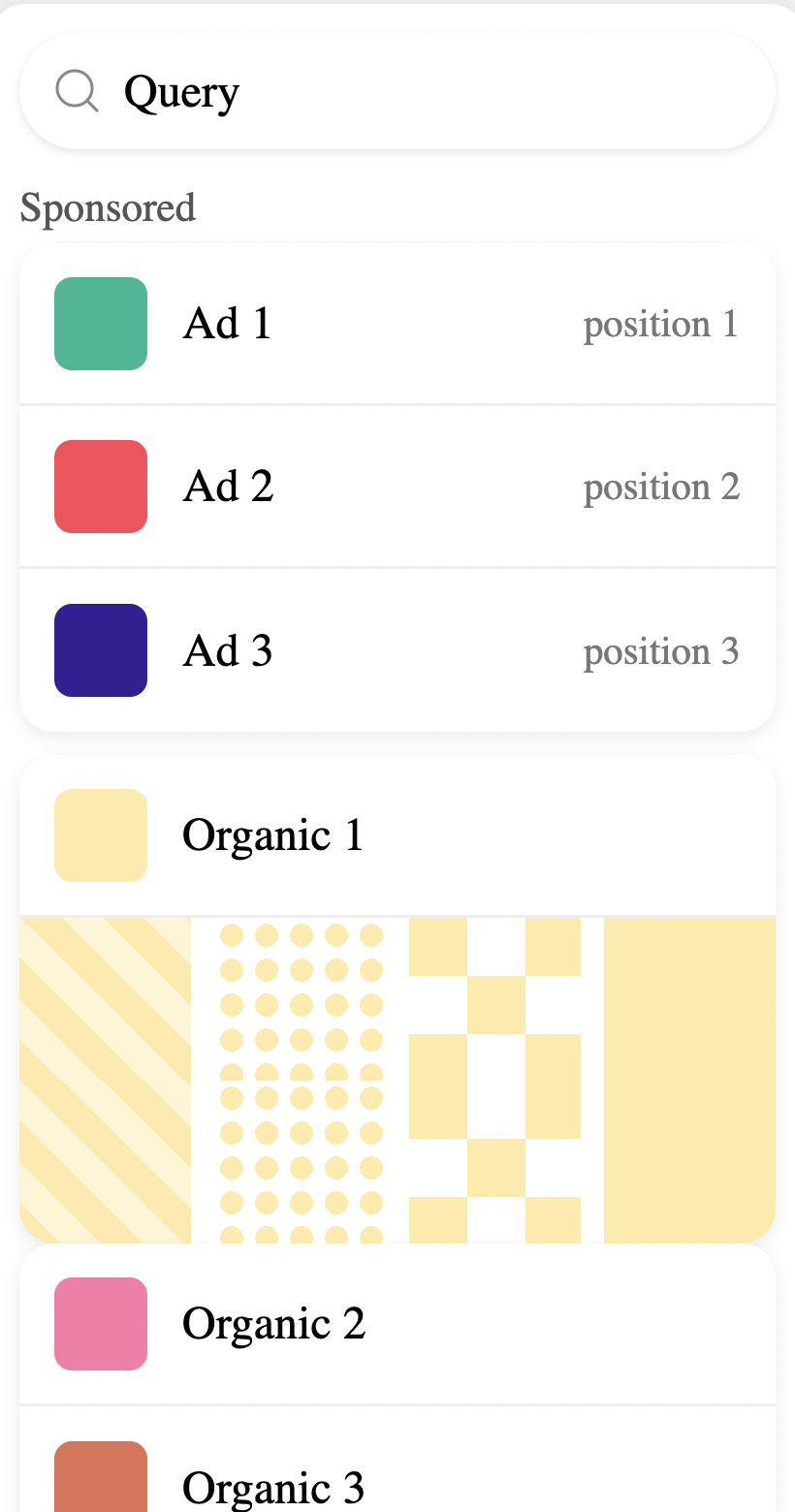}
    \caption{Brand Query Three-Ad Layout Example}
    \label{fig:three_ads_brand_layout}
  \end{subfigure}
  \caption{Examples of brand query layouts: in (a), the brand appears as the sponsored Ad~1; in (b), it appears as Organic~1.}
  \label{fig:brand_query_layouts}
\end{figure}

\section{Data Distributions}
\subsection{CDF of Search Frequency}
\label{apdx:search_CDF}
\begin{figure}[htp!]
    \centering
    \includegraphics[width=0.4\linewidth]{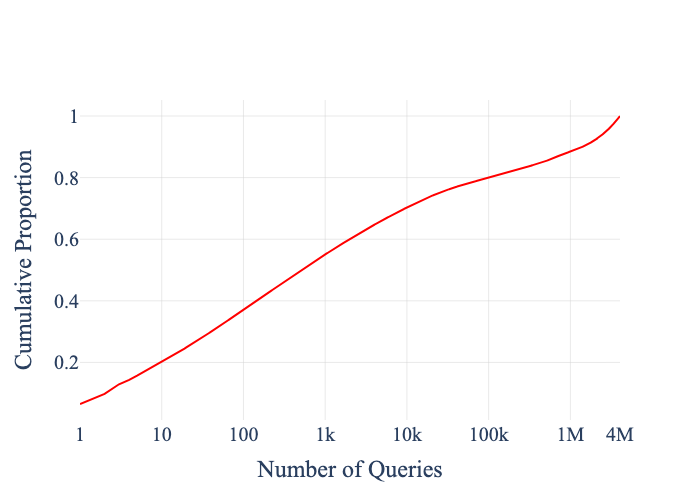}
    \caption{Cumulative distribution of search frequency of queries}
    \label{fig:query_cumulative_distribution}
\end{figure}
\subsection{CDF of Impressions and Conversions}
\label{apdx:apps_ads_CDF}
\begin{figure}[htp!]
  \centering
  % First subplot
  \begin{subfigure}[b]{0.4\textwidth}
    \centering
    \includegraphics[width=\textwidth]{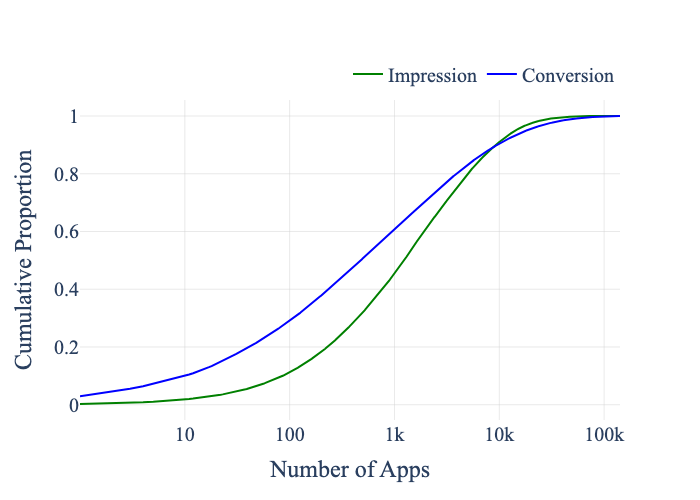}
    \caption{CDF of Apps (Organic and Ads).}
    \label{fig:apps_cdf}
  \end{subfigure}
  \hfill
  % Second subplot
  \begin{subfigure}[b]{0.4\textwidth}
    \centering
    \includegraphics[width=\textwidth]{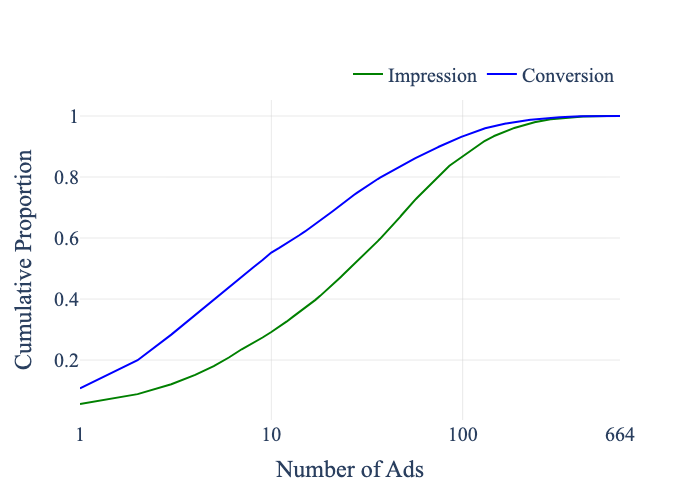}
    \caption{CDF of Ads.}
    \label{fig:ads_cdf}
  \end{subfigure}

  \caption{Cumulative distribution of Conversions and Impressions}
  \label{fig:apps_ads_cdfs}
\end{figure}
 The blue and green curves in Figure~\ref{fig:apps_ads_cdfs} show the cumulative distribution of conversions and impressions respectively. 
    Notice that “Conversion” CDFs lie above the corresponding ``Impression” CDFs. This suggests that conversions are more concentrated (across both apps and ads) than impressions.
\section{Platform Usage}
\subsection{Engagement}
\label{apdx:engagement_regressions}

We evaluate whether daily engagement changes under higher ad loads.
Let $ d$ be mean-centered time within the experiment window (in days). We estimate a condition-specific levels-and-trends model with the control group as baseline:

% Intercepts shows the mean daily fraction under $C$ and daily engagement trend under $C$. We use day level mean centered time $\bar{t}$. If the daily fraction has a trend under $C$ we must observe a significant value for $\delta_C$. $\beta_j$ is the average deviation of daily engagement from control group's engagement. $\delta_j$ represent the daily trend of engagement under condition $j$ in addition to the absorbed trend $\delta_C$. 
% The result is presented in \ref{tab:engagement_levels_trends}. We see that the engegment consstently decrease as more ads are displayed. In addition, there is no significant trend on daily engagement under control group ($\delta_C$). However, we see negative trend on daily engagement under the last three conditions . Thus, we expect that daily engagement further decreases under $T_4,T_5,T_6$ if the experiment continue for longer period of time. 
\begin{align}
\text{Engagement}_{e,d}
&= (\beta_{C} \;+\; \delta_{C}\, d)
\;+\; \sum_{j =2}^6 (\beta_{T_j}+\delta_{T_j}\, d)\,\mathbf{1}\{e=T_j\}
\;+\; \varepsilon_{e,d}
\label{eq:engagement_levels_trends_clean}
\end{align}
$\text{Engagement}_{e,d}$ measures the daily engagement under experimental condition $e$ on day $d$.
Here, $\beta_C$ and $\delta_C$ estimate the mean engagement and time trend of the control group respectively. If engagement trends under control are absent, we expect $\delta_C$ to be statistically indistinguishable from zero. $\beta_{T_j}$ and $\delta_{T_j}$ estimate the level difference for condition $T_j$ and the incremental trend difference for condition $T_j$ relative to control respectively.
Table~\ref{tab:engagement_levels_trends} reports the estimates: Engagement levels decline monotonically with ad load. The control trend $\delta_C$ is small and not statistically significant, indicating no intrinsic drift in engagement absent additional ads. In contrast, we observe significantly negative incremental trends for $T_4$–$T_6$, implying that daily engagement erodes over time under extreme ad loads. If the experiment were extended, this pattern suggests further engagement losses in $T_4$–$T_6$ relative to control.

\begin{table}[htb!]
\centering
\caption{Daily Engagement: Condition Levels and Trends per Day}
\label{tab:engagement_levels_trends}
\begin{tabular}{lcc}
\toprule
 & \textbf{Daily Engagement} &  \textbf{Trend} \\
\midrule
$C$ & 4.95e-02$^{***}$ (1.00e-03) & 5.82e-05 (3.16e-05) \\
% \multicolumn{3}{l}{\textbf{Conditions (difference vs.\ C)}} \\
$T_2$ & -3e-04$^{***}$ (4.83e-05) & 1.45e-06 (2.47e-06) \\
$T_3$ & -6e-04$^{***}$ (4.63e-05) & -1.11e-06 (3.08e-06) \\
$T_4$ & -8e-04$^{***}$ (5.79e-05) & -9.76e-06$^{**}$ (3.65e-06) \\
$T_5$ & -10e-04$^{***}$ (4.96e-05) & -6.38e-06$^{*}$ (3.01e-06) \\
$T_6$ & -11e-04$^{***}$ (4.96e-05) & -9.97e-06$^{**}$ (3.16e-06) \\
\midrule
Observations & 396 & 396 \\
\bottomrule
\end{tabular}

\par\vspace{0.25em}
\footnotesize{Standard errors in parentheses are clustered at the day level. Significance levels: ${}^{***}p<0.001$, ${}^{**}p<0.01$, ${}^{*}p<0.05$.}
\end{table}

\subsection{Search per User}
\label{apdx:search_per_user_regression}
We now test whether higher ad loads induce additional search activity. Let $\text{SearchesPerUser}_u$ denote the total number of search queries executed by user $u$ during the experiment:
$$\text{SearchPerUser}_u = \sum_{s} \mathbf{1}\{u_s=u\}$$
With control as the baseline, we estimate:
\begin{align}
\label{eq:search_per_user_vsC}
\text{SearchesPerUser}_u
\;=\; \zeta_{C}
\;+\; \sum_{j=2}^6 \zeta_{T_j}\,\mathbf{1}\{e(u) = {T}_j\}
\;+\; \varepsilon_{u}
\end{align}

where $\zeta_C$ is the average under $C$ and $\zeta_{T_j}$ estimates the deviation for condition $T_j$ relative to $C$.
Table~\ref{tab:search_per_user_vsC} presents the estimates. It shows that search intensity is higher under larger ad loads: coefficients for $T_3$–$T_6$ are positive and highly significant, with the largest increase in $T_6$ (0.171 additional searches per user). Relative to the control mean of $\zeta_C=5.1389$, this corresponds to approximately a 3.3\% increase. 

\begin{table}[htp!]
\centering
\caption{Searches per User by Condition}
\label{tab:search_per_user_vsC}
\begin{tabular}{lc}
\toprule
 & \textbf{Searches per User} \\
\midrule
$C$ & 5.1389$^{***}$ (0.0060) \\
% \multicolumn{2}{l}{\textbf{Conditions (difference vs.\ C)}} \\
$T_2$ & 0.0199 (0.0170) \\
$T_3$ & 0.1620$^{***}$ (0.0170) \\
$T_4$ & 0.1673$^{***}$ (0.0170) \\
$T_5$ & 0.1412$^{***}$ (0.0160) \\
$T_6$ & 0.1710$^{***}$ (0.0160) \\
\midrule
Observations & 5,057,952 \\
\bottomrule
\end{tabular}

\par\vspace{0.25em}
\footnotesize{Standard errors in parentheses. Significance levels: ${}^{***}p<0.001$, ${}^{**}p<0.01$, ${}^{*}p<0.05$.}
\end{table}

\subsection{Engagement During Deployment}
\label{apdx:deployment_engagement}
We also examine engagement during the deployment period. For each date and condition, engagement is defined as the fraction of assigned users who performed at least one search. We estimate average differences relative to LAAL while absorbing common date shocks:
\begin{align}
\text{Engagement}_{e,d}
&= \alpha_L
\;+\; \beta_C \mathbf{1}\{e=C\}
\;+\; \sum_{j=2}^{6} \beta_{T_j}\mathbf{1}\{e=T_j\}
\;+\; \gamma_d
\;+\; \varepsilon_{e,d}.
\label{eq:deployment_engagement_date_fe}
\end{align}
Here, LAAL is the omitted condition and $\gamma_d$ denotes date fixed effects. Thus, $\beta_C$ and $\beta_{T_j}$ measure average engagement differences relative to LAAL after controlling for day-level shocks.

Table~\ref{tab:deployment_engagement} reports the results. $C$ and $T_2$ have higher engagement than LAAL, while the other static conditions have lower engagement than LAAL. However, only the effects of $T_4$ and $T_6$ are statistically significant. This limited significance is expected because the deployment period contains only 22 days of data.

\begin{table}[htp!]
\centering
\caption{Deployment Engagement Relative to LAAL}
\label{tab:deployment_engagement}
\begin{tabular}{lc}
\toprule
 & \textbf{Daily Engagement} \\
\midrule
LAAL  & 8.43e-02$^{***}$ (6.19e-04) \\
$C$   & 1.69e-04 (3.66e-04) \\
$T_2$ & 1.09e-04 (3.86e-04) \\
$T_3$ & -4.72e-04 (3.71e-04) \\
$T_4$ & -9.55e-04$^{*}$ (3.89e-04) \\
$T_5$ & -6.44e-04 (4.13e-04) \\
$T_6$ & -7.87e-04$^{*}$ (3.75e-04) \\
\midrule
Observations & 154 \\
\bottomrule
\end{tabular}

\par\vspace{0.25em}
\footnotesize{Standard errors in parentheses. The LAAL row reports the baseline engagement level. Other rows report differences relative to LAAL. Estimates are from Equation~\eqref{eq:deployment_engagement_date_fe}. Significance levels: ${}^{***}p<0.001$, ${}^{**}p<0.01$, ${}^{*}p<0.05$.}
\end{table}

\section{Query Segmentation Results}
\label{apdx:query_bucket_summary}
\subsection{Query Segmentation Summary Statistics}
Table~\ref{tab:bucket-composition} summarizes the distribution share of outcomes across query segments introduced in \ref{subsec:het-queries}.\footnote{Queries outside this 10k set account for the remaining mass:
$40.41\%$ of searches, $10.58\%$ of revenue, $37.67\%$ of total conversions, and $7.61\%$ of ad conversions.}. Most of revenue and advertising conversion come from the segment ${K_3}$. The segment ${K_1}$ contributes the most to the total search and conversion share while it has a small share to the advertising results.

\begin{table}[htp!]\centering
\caption{Query Segment Composition. Numbers show the shares under each variable. }
\label{tab:bucket-composition}
\begin{tabular}{lrrrrr}
\toprule
Segment & \#Queries & Search & Revenue & Ads Conversion & Conversion \\
\midrule
K1   & 3{,}333 & 0.468 & 0.188 & 0.091 & 0.450  \\
K2   & 3{,}334 & 0.069 & 0.153 & 0.100 & 0.092  \\
K3   & 3{,}333 & 0.059 & 0.553 & 0.733 & 0.081 \\
\bottomrule
\end{tabular}
\vspace{0.25em}
\end{table}
% & 0.00278& 0.02068& 0.17863 ; =0.00270$
\subsection{Regression Estimates}
Table \ref{tab:bucket-fe-estimates} presents the regression estimates from Equation \eqref{eq:bucket}. 

%\hy{Standard errors are clustered at the query level to allow arbitrary correlation in outcomes within a query across searches and experimental conditions.}

\begin{table}[htp!]
\centering
\small
\setlength{\tabcolsep}{2.5pt} % Tightens column spacing
\caption{Effect of Ad Load on Revenue and Conversion Across Segments}
\label{tab:bucket-fe-estimates}
\begin{tabular}{lcccccc}
\toprule
& \multicolumn{3}{c}{\textbf{Revenue}} & \multicolumn{3}{c}{\textbf{Conversion}} \\
\cmidrule(lr){2-4}\cmidrule(lr){5-7}
& Low & Medium & High & Low & Medium & High \\
\midrule
$C$ & 3.368 & 16.212 & 62.741 & 0.547 & 0.763 & 0.781 \\
$T_2$ & -0.427 (0.218) & -0.094 (0.536) & 14.136$^{***}$ (3.327) & -0.011$^{***}$ (0.003) & -0.025$^{***}$ (0.003) & -0.016$^{***}$ (0.003) \\
$T_3$ & -0.621$^{**}$ (0.199) & 2.047$^{***}$ (0.547) & 20.701$^{***}$ (4.252) & -0.009$^{**}$ (0.003) & -0.023$^{***}$ (0.004) & -0.030$^{***}$ (0.004) \\
$T_4$ & -0.510$^{**}$ (0.189) & 2.330$^{***}$ (0.524) & 25.432$^{***}$ (5.027) & -0.017$^{***}$ (0.003) & -0.037$^{***}$ (0.004) & -0.033$^{***}$ (0.005) \\
$T_5$ & -0.177 (0.173) & 3.526$^{***}$ (0.674) & 28.927$^{***}$ (4.743) & -0.024$^{***}$ (0.005) & -0.036$^{***}$ (0.006) & -0.049$^{***}$ (0.004) \\
$T_6$ & -0.255 (0.180) & 4.828$^{***}$ (0.665) & 32.887$^{***}$ (5.018) & -0.033$^{***}$ (0.006) & -0.060$^{***}$ (0.005) & -0.052$^{***}$ (0.005) \\
\midrule
Obs. & 12,313,120 & 1,825,671 & 1,548,914 & 12,313,120 & 1,825,671 & 1,548,914 \\
$R^2$ & 0.001 & 0.010 & 0.072 & 0.176 & 0.048 & 0.029 \\
\bottomrule
\end{tabular}

\par\vspace{0.25em}
\footnotesize
\textit{Notes:} Standard errors in parentheses are clustered at the query level. Significance levels: ${}^{***}p<0.001$, ${}^{**}p<0.01$, ${}^{*}p<0.05$.
\end{table}

\section{Conditional Cohort-Invariance Assumption}
\label{appsec:CCIA}
LAAL uses observations from the adaptive cohort and the static cohorts to estimate query--ad-load outcomes. This requires that, after conditioning on the same query, ad load, and time, the cohort label does not contain additional information about immediate search-level outcomes. We evaluate this condition by estimating two regressions, one for revenue and one for conversion:
\begin{equation}
\label{eq:ccia_test}
Y_s
=
\alpha^Y_{q_s,a_s,h_s}
+
\sum_{i=1}^{6}
\beta_i^Y
\mathbf{1}\{e_s=LAAL\}
\mathbf{1}\{a_s=i\}
+
\varepsilon_s^Y,
\end{equation}
where \(Y_s \in \{\mathrm{Revenue}_s,\mathrm{Conversion}_s\}\), \(q_s\) is the query, \(a_s\) is the realized ad load, and \(h_s\) is the calendar-hour of the search. The fixed effects \(\alpha^Y_{q_s,a_s,h_s}\) absorb all query-by-ad-load-by-hour variation. Thus, each coefficient \(\beta_i^Y\) measures whether the LAAL cohort has additional explanatory power for outcome \(Y\), relative to the corresponding static cohort with the same realized ad load.

Table~\ref{tab:ccia_test} reports the results. The conversion coefficients are small and statistically insignificant across all ad loads. For revenue, only the LAAL-versus-\(C\) comparison at one ad is statistically significant. To assess its economic magnitude, note that Table~\ref{tab:summary_user_experiment_deploy} implies a control reward per search of approximately
\[
\frac{2.152+0.01\times 25.80}{3.44}=0.701.
\]
The significant revenue coefficient, \(0.199\), changes the policy reward by only \(0.01\times 0.199=0.002\), which is \(0.002/0.701 \approx 0.28\%\) of the control reward per search. Thus, although statistically detectable, the effect is economically negligible. Overall, the results do not provide economically meaningful evidence that cohort identity explains remaining revenue or conversion variation after controlling for query-by-ad-load-by-hour fixed effects.

\begin{table}[htp!]
\centering
\small
\caption{Testing Cohort Effects Conditional on Query--Ad-Load--Hour Fixed Effects}
\label{tab:ccia_test}
\begin{tabular}{lcc}
\toprule
 & \textbf{Revenue} & \textbf{Conversion} \\
\midrule
LAAL vs. \(C\), 1 ad
    & 0.199$^{*}$ (0.080)
    & 0.008 (0.004) \\
LAAL vs. \(T_2\), 2 ads
    & 0.305 (1.888)
    & 0.003 (0.071) \\
LAAL vs. \(T_3\), 3 ads
    & 6.592 (3.434)
    & -0.005 (0.123) \\
LAAL vs. \(T_4\), 4 ads
    & 15.971 (9.870)
    & -0.023 (0.110) \\
LAAL vs. \(T_5\), 5 ads
    & 0.290 (0.594)
    & 0.009 (0.012) \\
LAAL vs. \(T_6\), 6 ads
    & 4.011 (4.108)
    & 0.008 (0.032) \\
\midrule
Query \(\times\) Ad Load \(\times\) Hour FE
    & \checkmark & \checkmark \\
Obs.
    & 77,601,057 & 77,601,057 \\
\(R^2\)
    & 0.141 & 0.024 \\
\bottomrule
\end{tabular}

\par\vspace{0.25em}
\footnotesize{
\textit{Notes:} Each column estimates Equation~\eqref{eq:ccia_test} on the full deployment search sample. Each row reports the LAAL coefficient for realized ad load \(i\), comparing LAAL searches with the corresponding static cohort after absorbing the fixed effects. Standard errors in parentheses are clustered at the query level. Significance levels: ${}^{***}p<0.001$, ${}^{**}p<0.01$, ${}^{*}p<0.05$.
}
\end{table}

\section{Regret Guarantee for e-LAAL}
\label{appsec:elaal_regret_proof}

This appendix gives the finite-sample dynamic-regret guarantee for e-LAAL for a fixed query \(q\). It provides the detailed definitions, support conditions, and tuning argument behind Theorem~\ref{thm:elaal_regret_main}. 
\paragraph{Query-level sequence and rewards.}
Fix a query \(q\), and let \(\mathcal S_q(T)=\{s_1,\dots,s_N\}\) be the realized searches containing \(q\) during \([0,T]\), ordered by time. Write \(t_j:=t_{s_j}\), \(e_j:=e_{s_j}\), and \(a_j:=a_{s_j}\). We identify the static cohort \(C\) with arm \(1\), and the static cohort \(T_i\) with arm \(i\) for \(i=2,\dots,K\). Thus the action set is \(\mathcal A=\{1,\dots,K\}\), and the cohort set is \(\{L\}\cup\mathcal A\), where \(L\) denotes the LAAL cohort.

Let the raw scalarized reward be
\[
r_j^{\mathrm{raw}}:=r(y_{s_j},\lambda)
=\lambda\,\mathrm{Revenue}_{s_j}+\mathrm{Conversion}_{s_j},
\qquad 0\le r_j^{\mathrm{raw}}\le B_r<\infty.
\]
For the proof, normalize rewards by \(B_r\): \(\tilde r_j:=r_j^{\mathrm{raw}}/B_r\in[0,1]\). All mean rewards, variation budgets, regret quantities, and the tie-break parameter \(\eta\) in this appendix are expressed in normalized units. Bounds in raw reward units are recovered by multiplying by \(B_r\).

For each arm \(i\in\mathcal A\), define the possibly time-varying mean reward
\[
\mu_{i,j}:=\mathbb E[\tilde r_j\mid a_j=i,q].
\]
Let \(i_j^\star\in\arg\max_{i\in\mathcal A}\mu_{i,j}\). The dynamic regret over \([0,T]\) is
\[
R_q^{\mathrm{dyn}}(T):=\sum_{j=1}^N\big(\mu_{i_j^\star,j}-\mu_{a_j,j}\big),
\]
and the query-level variation budget, with \(V_q(T)>0\), is
\[
V_q(T):=\sum_{j=2}^N\max_{i\in\mathcal A}|\mu_{i,j}-\mu_{i,j-1}|.
\]

\paragraph{Design weights and query-level support.}
Let \(w_e\) denote the pre-assigned user-level cohort weights, with \(w_L+\sum_{i=1}^K w_i=1\). Define the total static design mass \(\epsilon:=\sum_{i=1}^K w_i=1-w_L\), and let \(w_{\min}:=\min_{i\in\mathcal A}w_i\). User assignment is fixed by design, but for a fixed query \(q\), realized query-level traffic may not match the user-level design shares. Let \(\varepsilon_{e,j}:=\Pr(e_j=e)\) denote the expected share of the \(j\)th query-\(q\) search coming from cohort \(e\). We assume that there exist constants \(\underline c\in(0,1]\) and \(\bar c\ge 1\) such that, for every search \(j\),
\begin{equation}
\label{eq:lower_static_support}
\varepsilon_{i,j}\ge \underline c\,w_i,
\qquad \forall i\in\mathcal A,
\end{equation}
and
\begin{equation}
\label{eq:upper_total_static_mass}
\sum_{i=1}^K\varepsilon_{i,j}\le \bar c\,\epsilon.
\end{equation}
The coefficient \(\underline c\) is the minimum-support coefficient: it lower bounds realized query-level traffic on each static arm by its design weight. The coefficient \(\bar c\) is the total-static-traffic upper-bound coefficient: it upper bounds total realized static-cohort traffic by the total design exploration mass. If realized query-level static traffic exactly matches the design exploration shares, then \(\underline c=\bar c=1\).

\paragraph{Sliding-window estimation and fallback.}
At each search \(s_j\), LAAL uses the most recent \(h\) units of time:
\[
W_j(h):=\{\ell<j:t_j-h\le t_\ell<t_j\}.
\]
Because the searches are ordered by time, write \(W_j(h)=\{b_j,\dots,j-1\}\), and let \(M_j:=|W_j(h)|=j-b_j\). Define the number of warm-up searches \(N_h:=|\{j\le N:t_j-t_{s_1}<h\}|\), and set \(J_h:=N_h+1\). Thus, for \(j\ge J_h\), the lookback window is not truncated by the start of the horizon.

Define the post-warm-up index sets
\[
\mathcal J_G(h,m):=\{j\in\{J_h,\dots,N\}:M_j\ge m\},
\qquad
\mathcal J_F(h,m):=\{j\in\{J_h,\dots,N\}:M_j<m\}.
\]
The realized LAAL greedy and fallback rounds are \(S_L^G:=\{j\in\mathcal J_G(h,m):e_j=L\}\) and \(S_L^F:=\{j\in\mathcal J_F(h,m):e_j=L\}\). The post-warm-up static rounds are \(S_S^+:=\{j\in\{J_h,\dots,N\}:e_j\in\mathcal A\}\), and the fallback count is \(F_q(T;h,m):=|S_L^F|\). When \(\mathcal J_G(h,m)\neq\varnothing\), define
\[
M_-^{(m)}:=\min_{j\in\mathcal J_G(h,m)}M_j,
\qquad
M_+^{(m)}:=\max_{j\in\mathcal J_G(h,m)}M_j.
\]

For each arm \(i\in\mathcal A\), let
\[
n_{i,j}:=\sum_{\ell\in W_j(h)}\mathbf 1\{a_\ell=i\}
\]
be the number of action-\(i\) observations in the window. The LAAL estimator is
\[
\widehat\mu_{i,j}
=
\frac{\sum_{\ell\in W_j(h)}\tilde r_\ell\,\mathbf 1\{a_\ell=i\}}{n_{i,j}+\delta_{\mathrm{sm}}}.
\]
Here \(\delta_{\mathrm{sm}}>0\). On post-warm-up LAAL rounds, for tie-break penalty \(\eta>0\), the algorithm plays \(a_j\in\arg\max_{i\in\mathcal A}\{\widehat\mu_{i,j}-\eta i\}\) for \(j\in S_L^G\), and plays the default \(a_j=i_{\mathrm{def}}\) for \(j\in S_L^F\). On static rounds, cohort \(i\) deterministically plays action \(i\).

\paragraph{Assumptions.}
(E1) \(\tilde r_j\in[0,1]\) for all \(j\).\par
(E2) Conditional on the realized query-search sequence \(\mathcal S_q(T)\), the cohort of each search is drawn independently across searches, with query-level cohort shares satisfying \eqref{eq:lower_static_support} and \eqref{eq:upper_total_static_mass}.\par
(E3) For each arm \(i\), the noise \(\tilde r_j-\mu_{i,j}\) is conditionally mean-zero and bounded in \([-1,1]\) on rounds where \(a_j=i\).
\par
(E4) Conditional cohort invariance holds for immediate search-level rewards: after conditioning on query \(q\), search time \(t_j\), and realized ad load \(i\), the conditional mean reward does not depend on whether the observation comes from the LAAL cohort or the corresponding static cohort.

\paragraph{Lemma E.1 (Per-arm support from static cohorts).}
Assume \(\mathcal J_G(h,m)\neq\varnothing\). For each \(i\in\mathcal A\) and \(j\in\mathcal J_G(h,m)\), define the number of static observations of arm \(i\) inside the window:
\[
Z_{i,j}:=\sum_{\ell\in W_j(h)}\mathbf 1\{e_\ell=i\}.
\]
Then \(n_{i,j}\ge Z_{i,j}\), because static cohort \(i\) always plays action \(i\). The conditional mean of \(Z_{i,j}\) is
\[
m_{i,j}:=\sum_{\ell\in W_j(h)}\varepsilon_{i,\ell}
\ge \underline c\,w_i\,M_j
\ge \underline c\,w_{\min}\,M_-^{(m)}.
\]
By Assumption~E2, for any \(\alpha\in(0,1)\),
\[
\Pr\!\left(n_{i,j}<(1-\alpha)\underline c\,w_{\min}\,M_-^{(m)}\right)
\le
\exp\!\left(-\frac{\alpha^2}{2}\underline c\,w_{\min}\,M_-^{(m)}\right).
\]
A union bound over \(i\in\mathcal A\) and \(j\in\mathcal J_G(h,m)\) gives
\[
\Pr\!\left(
\exists j\in\mathcal J_G(h,m),\exists i\in\mathcal A:
 n_{i,j}<(1-\alpha)\underline c\,w_{\min}\,M_-^{(m)}
\right)
\le
KN\exp\!\left(-\frac{\alpha^2}{2}\underline c\,w_{\min}\,M_-^{(m)}\right).
\]

\paragraph{Choice of \(\alpha\) and \(n_{\min}\).}
Assume \(\mathcal J_G(h,m)\neq\varnothing\). Fix \(\xi\in(0,1)\), and choose
\[
\alpha:=
\sqrt{\frac{2\log(3KN/\xi)}{\underline c\,w_{\min}\,M_-^{(m)}}},
\qquad
n_{\min}:=(1-\alpha)\underline c\,w_{\min}\,M_-^{(m)},
\]
assuming \(\alpha<1\). Then Lemma~E.1 implies that the support event
\[
\mathcal E_{\mathrm{supp}}
:=
\{\forall j\in\mathcal J_G(h,m),\forall i\in\mathcal A,
 n_{i,j}\ge n_{\min}\}
\]
satisfies \(\Pr(\mathcal E_{\mathrm{supp}})\ge 1-\xi/3\).

\paragraph{Lemma E.2 (Total static traffic concentration).}
Let \(Y_j:=\mathbf 1\{e_j\in\mathcal A\}\) for \(j=J_h,\dots,N\). Then
\[
\mathbb E[Y_j]=\sum_{i=1}^K\varepsilon_{i,j}\le \bar c\,\epsilon.
\]
Therefore, \(\mathbb E|S_S^+|\le \bar c\,\epsilon N\). By Assumption~E2, for any \(\zeta\in(0,1)\),
\[
\Pr\!\left(
|S_S^+|>
\bar c\,\epsilon N+
\sqrt{\frac{N}{2}\log\!\left(\frac{1}{\zeta}\right)}
\right)
\le \zeta.
\]

\paragraph{Theorem E.1 (Finite-sample dynamic regret bound).}
Under Assumptions (E1)--(E4), assume \(\mathcal J_G(h,m)\neq\varnothing\) and \(\alpha<1\). For any \(\xi\in(0,1)\), with probability at least \(1-\xi\),
\[
R_q^{\mathrm{dyn}}(T)
\le
N_h
+
\bar c\,\epsilon N
+
\sqrt{\frac{N}{2}\log\!\left(\frac{3}{\xi}\right)}
+
2|S_L^G|\sqrt{\frac{\log(6KN/\xi)}{2n_{\min}}}
+
2M_+^{(m)}V_q(T)
+
\eta(K-1)|S_L^G|
+
\frac{2\delta_{\mathrm{sm}}}{n_{\min}}|S_L^G|
+
F_q(T;h,m).
\]

\paragraph{Proof.}
The first \(N_h\) searches do not yet have a full \(h\)-time lookback window; since per-search regret is at most one, their contribution is at most \(N_h\). On any post-warm-up static round, regret is also at most one, so the total static-round regret is at most \(|S_S^+|\). Lemma~E.2 with \(\zeta=\xi/3\) gives the event
\[
\mathcal E_{\mathrm{stat}}
:=
\left\{
|S_S^+|
\le
\bar c\,\epsilon N+
\sqrt{\frac{N}{2}\log\!\left(\frac{3}{\xi}\right)}
\right\},
\]
with probability at least \(1-\xi/3\).

Now fix a greedy LAAL round \(j\in S_L^G\). Define regularized means \(g_{i,j}:=\mu_{i,j}-\eta i\) and \(\widehat g_{i,j}:=\widehat\mu_{i,j}-\eta i\). Greedy selection implies
\[
\mu_{i_j^\star,j}-\mu_{a_j,j}
\le
2\max_{i\in\mathcal A}|\widehat\mu_{i,j}-\mu_{i,j}|+
\eta(K-1).
\]
To bound the estimation error, define
\[
\widehat\mu^0_{i,j}:=
\frac{1}{n_{i,j}}\sum_{\ell\in W_j(h)}\tilde r_\ell\mathbf 1\{a_\ell=i\},
\qquad
\bar\mu_{i,j}:=
\frac{1}{n_{i,j}}\sum_{\ell\in W_j(h)}\mu_{i,\ell}\mathbf 1\{a_\ell=i\}.
\]
Then
\[
|\widehat\mu_{i,j}-\mu_{i,j}|
\le
|\widehat\mu^0_{i,j}-\bar\mu_{i,j}|
+
|\bar\mu_{i,j}-\mu_{i,j}|
+
|\widehat\mu_{i,j}-\widehat\mu^0_{i,j}|.
\]
On \(\mathcal E_{\mathrm{supp}}\), \(n_{i,j}\ge n_{\min}\). By (E3) and Hoeffding--Azuma, choosing
\[
\beta:=\sqrt{\frac{\log(6KN/\xi)}{2n_{\min}}}
\]
and applying a union bound over \(i\in\mathcal A\) and \(j\in\mathcal J_G(h,m)\) gives a noise event \(\mathcal E_{\mathrm{noise}}\), with probability at least \(1-\xi/3\), on which
\[
\max_{i,j}|\widehat\mu^0_{i,j}-\bar\mu_{i,j}|\le \beta.
\]
Since \(\tilde r_j\in[0,1]\), the smoothing term is bounded by
\[
|\widehat\mu_{i,j}-\widehat\mu^0_{i,j}|
\le
\frac{\delta_{\mathrm{sm}}}{n_{i,j}}
\le
\frac{\delta_{\mathrm{sm}}}{n_{\min}}.
\]
For the drift term, let \(v_r:=\max_{i\in\mathcal A}|\mu_{i,r}-\mu_{i,r-1}|\) and \(B_j:=\sum_{r=b_j+1}^{j}v_r\). Then \(|\bar\mu_{i,j}-\mu_{i,j}|\le B_j\), and
\[
\sum_{j\in S_L^G}B_j\le M_+^{(m)}V_q(T),
\]
because each variation increment appears in at most \(M_+^{(m)}\) greedy-eligible windows.

On \(\mathcal E_{\mathrm{supp}}\cap\mathcal E_{\mathrm{noise}}\cap\mathcal E_{\mathrm{stat}}\), which has probability at least \(1-\xi\), every greedy LAAL round satisfies
\[
\mu_{i_j^\star,j}-\mu_{a_j,j}
\le
2\beta+2B_j+\frac{2\delta_{\mathrm{sm}}}{n_{\min}}+\eta(K-1).
\]
Summing over \(j\in S_L^G\) gives
\[
\sum_{j\in S_L^G}(\mu_{i_j^\star,j}-\mu_{a_j,j})
\le
2|S_L^G|\beta
+
2M_+^{(m)}V_q(T)
+
\eta(K-1)|S_L^G|
+
\frac{2\delta_{\mathrm{sm}}}{n_{\min}}|S_L^G|.
\]
Fallback rounds contribute at most \(|S_L^F|=F_q(T;h,m)\). Adding the warm-up, static, greedy-LAAL, and fallback contributions proves the theorem. \(\square\)

\paragraph{Assumption E5 (Regular query rate).}
For the tuning results, assume there exist constants \(0<\kappa_-\le\kappa_+<\infty\) such that
\[
\kappa_-\frac{Nh}{T}\le M_j\le \kappa_+\frac{Nh}{T},
\qquad \forall j\in\{J_h,\dots,N\},
\qquad
N_h\le \kappa_+\frac{Nh}{T}.
\]
This formalizes the statement that an \(h\)-time window contains on the order of \(Nh/T\) searches of query \(q\).

\paragraph{Corollary E.1 (Fixed exploration mass).}
Under Assumption~E5, ignoring logarithmic factors and lower-order warm-up, smoothing, regularization, and fallback terms, and taking \(m=0\), Theorem~E.1 gives the dominant scaling
\[
R_q^{\mathrm{dyn}}(T)
\lesssim
\bar c\,\epsilon N
+
\sqrt{\frac{NT}{\underline c\,w_{\min}h}}
+
\frac{Nh}{T}V_q(T).
\]
For fixed \(\epsilon\), choosing
\[
h^\star\asymp
\left(
\frac{T^3}{\underline c\,w_{\min}N\,V_q(T)^2}
\right)^{1/3}
\]
yields
\[
R_q^{\mathrm{dyn}}(T)
=
O(\bar c\,\epsilon N)
+
\tilde O\!\left(
(\underline c\,w_{\min})^{-1/3}
N^{2/3}V_q(T)^{1/3}
\right).
\]
If \(N=\Theta(T)\), this becomes
\[
R_q^{\mathrm{dyn}}(T)
=
O(\bar c\,\epsilon T)
+
\tilde O\!\left(
(\underline c\,w_{\min})^{-1/3}
V_q(T)^{1/3}T^{2/3}
\right).
\]

\paragraph{Corollary E.2 (Joint tuning under equal static shares).}
Under the same simplifications as Corollary~E.1, suppose the \(K\) static cohorts receive equal design mass, \(w_i=\epsilon/K\) for \(i=1,\dots,K\). The dominant terms become
\[
R_q^{\mathrm{dyn}}(T)
\lesssim
\bar c\,\epsilon N
+
\sqrt{\frac{KNT}{\underline c\,\epsilon h}}
+
\frac{Nh}{T}V_q(T).
\]
For fixed \(\epsilon\), the optimal window length is
\[
h^\star(\epsilon)
\asymp
\left(
\frac{K T^3}{\underline c\,\epsilon N\,V_q(T)^2}
\right)^{1/3}.
\]
Substituting this back and optimizing over \(\epsilon\) gives
\[
\epsilon^\star
\asymp
\bar c^{-3/4}\underline c^{-1/4}
\left(\frac{K\,V_q(T)}{N}\right)^{1/4},
\]
and therefore
\[
R_q^{\mathrm{dyn}}(T)
=
\tilde O\!\left(
\bar c^{1/4}\underline c^{-1/4}
N^{3/4}
\big(K\,V_q(T)\big)^{1/4}
\right).
\]
If \(N=\Theta(T)\), this becomes
\[
R_q^{\mathrm{dyn}}(T)
=
\tilde O\!\left(
\bar c^{1/4}\underline c^{-1/4}
T^{3/4}
\big(K\,V_q(T)\big)^{1/4}
\right),
\]
which is the bound stated in Theorem~\ref{thm:elaal_regret_main} after writing \(V(T)\equiv V_q(T)\).

\section{Low Volume Queries}
\label{apdx:low_volume_queries}

\textbf{Default ad load for sparse queries ($i_{\mathrm{def}}$):}
LAAL applies greedy selection only when a query has at least $m$ recent observations in the sliding window; otherwise it falls back to a fixed default $i_{\mathrm{def}}$. We choose $i_{\mathrm{def}}$ using historical data from the static cohorts $\{C,T_2,\ldots,T_6\}$ during the last month of the first experiment.

Let $\mathcal{S}_H$ denote the historical sample of searches from the static cohorts (so each $s\in\mathcal{S}_H$ has a fixed realized ad load $a_s\in\{1,\ldots,6\}$ determined by its cohort). Using the same window length $h$ as in deployment, define the trailing-window volume for search $s$ as the number of searches for the same query in the preceding $h$ days:
\[
N_s \;\equiv\; \sum_{s'\in\mathcal{S}_H} \mathbf{1}\{q_{s'}=q_s\}\,\mathbf{1}\{t_s-h \le t_{s'} < t_s\}.
\]
Let $\epsilon$ denote the aggregate traffic share assigned to these static cohorts. We label a search as \emph{low-volume} if its full-traffic-equivalent trailing-window volume is at most the greedy threshold $m$:
\[
\mathcal{S}_{\text{Low}} \;\equiv\; \{\, s\in\mathcal{S}_H : N_s/\epsilon \le m \,\}.
\]
We then choose $i_{\mathrm{def}}$ as the best-performing \emph{uniform} ad-load policy on this low-volume subset, under the deployed reward $r(y_s,\lambda)=\lambda\,\mathrm{Revenue}_s+\mathrm{Conversion}_s$. Specifically, for each $i\in\{1,\ldots,6\}$ define the mean reward among low-volume searches that were served with ad load $i$:
\[
\widehat{r}^{\,\text{Low}}(i)
\;\equiv\;
\mathbb{E}\!\left[r(y_s,\lambda)\,\middle|\, s\in\mathcal{S}_{\text{Low}},\, a_s=i \right]
\;\approx\;
\frac{\sum_{s\in\mathcal{S}_{\text{Low}}}\mathbf{1}\{a_s=i\}\,r(y_s,\lambda)}
{\sum_{s\in\mathcal{S}_{\text{Low}}}\mathbf{1}\{a_s=i\}}.
\]
Finally, we set
\[
i_{\mathrm{def}} \in \arg\max_{i\in\{1,\ldots,6\}} \widehat{r}^{\,\text{Low}}(i).
\]
In our historical data, this maximizer is $i_{\mathrm{def}}=4$. Intuitively, $i_{\mathrm{def}}$ is the best uniform static policy when the query is sparse in the recent window (i.e., when LAAL cannot reliably estimate per-query arm values).

\begin{table}[htp!]
\centering
\small
\setlength{\tabcolsep}{3.5pt}
\caption{Low-Volume Queries: Outcomes by Conditions}
\label{tab:low_volume_queries_by_variant}
\begin{tabular}{lcccccc}
\toprule
 & \textbf{$C$} & \textbf{$T_{2}$} & \textbf{$T_{3}$} & \textbf{$T_{4}$} & \textbf{$T_{5}$} & \textbf{$T_{6}$} \\
\midrule
Search Count        & 1,787,515 & 354,797 & 335,833 & 338,457 & 334,132 & 336,299 \\
Conversion/Search   & 0.479 & 0.470 & 0.489 & 0.485 & 0.475 & 0.466 \\
Revenue/Search      & 4.03  & 4.37  & 4.75  & 5.15  & 5.52  & 5.99  \\
Reward/Search       & 0.519 & 0.514 & 0.536 & 0.537 & 0.531 & 0.526 \\
\bottomrule
\end{tabular}
\end{table}

\section{Reweighting Brand-On and Brand-Off Recommendation Histograms}
\label{appsec:brand_on_off_reweighting}

For each non-stationary query \(q\), let \(d_{\mathrm{on},i}^q\) and \(d_{\mathrm{off},i}^q\) denote the number of LAAL searches in which ad load \(i\) was recommended during \emph{Brand-On} and \emph{Brand-Off} states, respectively. Let \(D_{\mathrm{on}}^q=\sum_i d_{\mathrm{on},i}^q\), \(D_{\mathrm{off}}^q=\sum_i d_{\mathrm{off},i}^q\), and \(p_{\mathrm{on}}^q=D_{\mathrm{on}}^q/(D_{\mathrm{on}}^q+D_{\mathrm{off}}^q)\). To compare Brand-On and Brand-Off recommendation distributions while accounting for different Brand-On shares across queries, we reweight counts within query as \(\tilde d_{\mathrm{on},i}^q=d_{\mathrm{on},i}^q/p_{\mathrm{on}}^q\) and \(\tilde d_{\mathrm{off},i}^q=d_{\mathrm{off},i}^q/(1-p_{\mathrm{on}}^q)\). We then aggregate \(\tilde d_{\mathrm{on},i}=\sum_{q\in\mathcal Q_{\mathrm{NS}}}\tilde d_{\mathrm{on},i}^q\) and \(\tilde d_{\mathrm{off},i}=\sum_{q\in\mathcal Q_{\mathrm{NS}}}\tilde d_{\mathrm{off},i}^q\), and normalize each vector to sum to one.

\newpage
\putbib

\end{bibunit}
\end{appendices}

\end{document}